\documentclass{article}
\usepackage{amsmath}
\usepackage{amsfonts}
\usepackage{amssymb}
\usepackage{amsthm}
\usepackage{algorithm} 
\usepackage{algorithmic}
\usepackage{color}
\usepackage[parfill]{parskip}
\usepackage[margin=0.8in]{geometry}
\usepackage{graphicx} %
\usepackage{caption,subcaption}
\usepackage{bbm}
\usepackage{mdframed}
\usepackage{natbib}
\usepackage[hidelinks]{hyperref}
\usepackage{titlesec}
\usepackage{authblk}
\usepackage{placeins}

\DeclareMathOperator*{\argmin}{arg\,min}
\DeclareMathOperator*{\argmax}{arg\,max}

\usepackage{thmtools, thm-restate}
\newtheorem{example}{Example} 
\newtheorem{theorem}{Theorem}
\newtheorem{lemma}[theorem]{Lemma}

\newtheorem{corollary}[theorem]{Corollary}
\newtheorem{definition}[theorem]{Definition}

\newcommand{\env}{\mathcal{E}}
\newcommand{\ospace}{\mathcal{O}}
\newcommand{\sspace}{\mathcal{S}}
\newcommand{\aspace}{\mathcal{A}}
\newcommand{\actions}{\mathcal{A}}
\newcommand{\transition}{\rho}
\newcommand{\xspace}{\mathcal{X}}
\newcommand{\yspace}{\mathcal{Y}}
\newcommand{\Ent}{\mathbb{H}}
\newcommand{\Prob}{\mathbb{P}}
\newcommand{\E}{\mathbb{E}}
\newcommand{\KL}{{\mathbf d}_{\rm KL}}
\newcommand{\I}{\mathbb{I}}
\newcommand{\Z}{\mathbb{Z}}

\definecolor{red}{RGB}{255,0,0}
\definecolor{blue}{RGB}{0,0,255}
\definecolor{green}{RGB}{0,255,0}
\definecolor{orange}{RGB}{255,165,0}
\definecolor{purple}{RGB}{128,0,128}
\definecolor{teal}{RGB}{0,128,128}
\definecolor{x}{RGB}{255,102,102}
\definecolor{darkgreen}{rgb}{0.0,0.5,0.0}

\newcommand{\saurabh}[1]{{\color{black} #1}}
\newcommand{\henrik}[1]{{\color{black} #1}}

\definecolor{lightgray}{RGB}{230, 230, 230}
\newmdenv[linecolor=lightgray, backgroundcolor=lightgray,innerbottommargin=15pt,innertopmargin=10pt,skipabove=20pt,splittopskip=20pt,splitbottomskip=15pt]{graybox}

\newenvironment{summary}
    {\begin{graybox}
    \addcontentsline{toc}{subsection}{Summary}
    \subsection*{Summary}
    }
    {
    \end{graybox}
    }

\title{Continual Learning as \\ Computationally Constrained Reinforcement Learning}
\author[1\footnote{Equal contribution. Correspondence to \texttt{szk@stanford.edu}.}]{Saurabh Kumar}
\author[1*]{Henrik Marklund}
\author[1]{Ashish Rao}
\author[2]{Yifan Zhu}
\author[1]{\\ Hong Jun Jeon}
\author[3]{Yueyang Liu}
\author[2,3]{Benjamin Van Roy}

\affil[1]{Department of Computer Science, Stanford University}
\affil[2]{Department of Electrical Engineering, Stanford University}
\affil[3]{Department of Management Science and Engineering, Stanford University}

\date{}

\begin{document}

\maketitle

\begin{abstract}
An agent that accumulates knowledge to develop increasingly sophisticated skills over a long lifetime could advance the frontier of artificial intelligence capabilities.  The design of such agents, which remains a long-standing challenge, is addressed by the subject of continual learning.  This monograph clarifies and formalizes concepts of continual learning, introducing a framework and tools to stimulate further research. 
 We also present a range of empirical case studies to illustrate the roles of forgetting, relearning, exploration, and auxiliary learning.

Metrics presented in previous literature for evaluating continual learning agents tend to focus on particular behaviors that are deemed desirable, such as avoiding catastrophic forgetting, retaining plasticity, relearning quickly, and maintaining low memory or compute footprints. 
In order to systematically reason about design choices and compare agents, a coherent, holistic objective that encompasses all such requirements would be helpful. To provide such an objective, we cast continual learning as reinforcement learning with limited compute resources. In particular, we pose the continual learning objective to be the maximization of infinite-horizon average reward subject to a computational constraint. Continual supervised learning, for example, is a special case of our general formulation where the reward is taken to be negative log-loss or accuracy. Among the implications of maximizing average reward are that remembering all information from the past is unnecessary, forgetting non-recurring information is not ``catastrophic,” and learning about how an environment changes over time is useful.

Computational constraints give rise to informational constraints in the sense that they limit the amount of information used to make decisions.  A consequence is that, unlike \henrik{in more common }%
framings of machine learning in which per-timestep regret vanishes as an agent accumulates information, the regret experienced in continual learning typically persists.  
\henrik{Related to this is that even in stationary environments, informational constraints can incentivize perpetual adaptation. }
Informational constraints also give rise to the familiar stability-plasticity dilemma, which we formalize in information-theoretic terms.
\end{abstract}

\clearpage

\tableofcontents

\clearpage

\section{Introduction}\label{sec:introduction}

\henrik{Continual learning remains a long-standing challenge. Success requires continuously ingesting new knowledge while retaining old knowledge that remains useful. Consider an AI personal assistant which may need to adapt to a user's evolving needs and preferences as well as learn new skills as new challenges arise. For example, an assistant that manages a user's schedule would need to retain information about constant habits while adapting to changing events (e.g., changing work hours, travel arrangements, etc). As another example, consider an ``AI scientist" used to accelerate scientific discovery. This system would need to build upon past knowledge by incorporating evolving data collected over time, reformulating and updating hypotheses, and designing new types of experiments with evolving tools and methodology. More generally, an agent needs to efficiently accumulates knowledge to develop increasingly sophisticated skills over a long lifetime~\citep{HADSELL20201028,khetarpal2022towards,ring2005toward,thrun1998learning}.} 

\henrik{Existing incremental machine learning techniques fall short of these ambitions of continual learning, as a major challenge has been to develop scalable systems that judiciously control what information they ingest, retain, or forget.}
Indeed, catastrophic forgetting (ejecting useful information from memory) and implasticity (forgoing useful new information) are recognized as obstacles to effective continual learning.

\henrik{When applying machine learning techniques on stationary data distributions, it is common to view a machine learning algorithm as acquiring knowledge about a fixed latent variable.}
The aim is to develop methods that quickly learn about the latent variable, which we will refer to as the {\it learning target}, as data accumulates. For instance, in supervised learning, the learning target could be an unknown function mapping inputs to labels \henrik{\citep{russell2016artificial}}. In the reinforcement learning literature, the learning target is often taken to be the unknown transition matrix of a Markov decision process\henrik{~\citep{sutton2018reinforcement}}. \henrik{In these settings}, an agent can be viewed as driving per-timestep regret -- the performance shortfall relative to what could have been if the agent began with perfect knowledge of the learning target -- to zero.  If the agent is effective, regret vanishes as the agent accumulates knowledge.  When regret becomes negligible, the agent is viewed as ``done'' with learning.

In contrast, continual learning addresses environments in which there may be no natural fixed learning target and an agent ought to never stop acquiring new knowledge.  %
\henrik{To} perform well, an agent must constantly adapt its behavior in response to evolving patterns.

There is a large gap between the state of the art in continual learning and what may be possible, making the subject ripe for innovation.  This difference becomes evident when examining an approach in common use, which entails periodically training a new model from scratch on buffered data.  To crystalize this gap, consider as a hypothetical example an agent designed to \henrik{predict electricity prices. Because more recent data may have stronger predictive power than older data, the agent's neural network is trained continuously on new data. To prioritize recent data, at} the end of each month, this agent trains a new neural network model from scratch on data observed during the previous twelve months. This new model replaces the old model and governs predictions over the next month.  This agent serves as a simple baseline that affords opportunity for improvement. For example, training each month's model from scratch is likely wasteful since it does not benefit from computation invested over previous months.  Further, by limiting knowledge ingested by each model to that available from data acquired over the preceding twelve months, the agent forgoes the opportunity to acquire complex skills that might only be developed over a much longer duration.

While it ought to be possible to design more effective continual learning agents, how to go about that or even how to assess improvement remains unclear.  Work on deep learning suggests that agent performance improves with increasing sizes of models, datasets, and inputs.  However, computational resource requirements scale along with these and become prohibitive.  A practically useful objective must account for computation.  \henrik{There are two primary goals of this monograph: (1) propose such an objective for continual learning and (2) understand key factors to consider in designing a performant continual learning agent. Rather than offer definitive methods, we aim to stimulate research toward identifying them.}

Metrics presented in previous literature for evaluating continual learning agents tend to focus on particular behaviors that are deemed desirable, such as avoiding catastrophic forgetting, retaining plasticity, relearning quickly, and maintaining low memory or compute footprints~\citep{ashley2021does,dohare2021continual,fini2020online,kirkpatrick2017overcoming}. For instance, the most common evaluation metric measures prediction accuracy on previously seen tasks to study how well an agent retains past information~\citep{wang2023comprehensive}. However, the extent to which each of these behaviors matters is unclear. In order to systematically reason about design decisions and compare agents, a coherent, holistic objective that reflects and encompasses all such requirements would be helpful. 

In this monograph, we view continual learning under the lens of reinforcement learning \citep{agarwal2019reinforcement,bertsekas1996neuro,meyn2022control,sutton2018reinforcement,szepesvari2010algorithms} to provide a formalism for what an agent is expected to accomplish.  Specifically, we consider maximization of infinite-horizon average reward subject to a computational constraint.  Average reward emphasizes long-term performance, which is suitable for the purpose of designing long-lived agents.  The notion of maximizing average reward generalizes that of online average accuracy, as used in some literature on continual supervised learning \citep{cai2021online,ghunaim2023real,hammoud2023rapid,hu2022drinking,lin2021clear,prabhu2023online,xu2022revealing}.

As reflected by average reward, an agent should aim to perform well on an ongoing basis in the face of incoming data it receives from the environment.  Importantly, remembering all information from the past is unnecessary, and forgetting non-recurring information is not ``catastrophic.''  An agent can perform well by remembering the subset that continues to remain useful.  Although our objective relaxes the requirement of retaining all information to only retaining information useful in the future, even this remains difficult, or even impossible, in practice.  Computational resources limit an agent's capacity to retain and process information.  The computational constraint in our continual learning objective reflects this gating factor.  This is in line with recent work highlighting the need to consider computational costs in continual learning~\citep{prabhu2023computationally}.

The remainder of this monograph is organized as follows. In Section 2, we introduce our framing of continual learning as reinforcement learning with an objective of maximizing average reward subject to a computational constraint.  In Section 3, we introduce information-theoretic tools inspired by \cite{jeon2023informationtheoretic,lu2021reinforcement} to offer a lens for studying agent behavior and performance.  
 In Section 4, we formalize the concepts of stability and plasticity to enable a coherent analysis of trade-offs between these conflicting goals.  
 In Section 5, we interpret in information-theoretic terms what it means for an agent to perpetually learn rather than drive regret to zero and be ``done'' with learning.  This line of thought draws inspiration from \citet{Abel2023definition}, which defines a notion of convergence and associates continual learning with non-convergence.
 Finally, in Section 6, to highlight the implications of our continual learning objective, we study simulation results from a set of case studies.

\section{An Objective for Continual Learning} \label{sec:objective}

Continual learning affords the never-ending acquisition of skills and knowledge \citep{Ring:1994}.  An agent operating over an infinite time horizon can develop increasingly sophisticated skills, steadily building on what it learned earlier.  On the other hand, due to computational resource constraints, as such an agent observes an ever-growing volume of data, it must forgo some skills to prioritize others.  Designing a performant continual learning agent requires carefully trading off between these considerations.  A suitable mathematical formulation of the design problem must account for that.  While many metrics have been proposed in the literature, they have tended to focus on particular behaviors that are deemed desirable.  A coherent, holistic objective would help researchers to systematically reason about design decisions and compare agents.  In this section, we formulate such an objective in terms of computationally constrained reinforcement learning (RL).

The subject of RL addresses the design of agents that learn to achieve goals through interacting with an environment \citep{sutton2018reinforcement}.  As we will explain, the {\it general} RL formulation subsumes the many perspectives that appear in the continual learning literature.  In this section, we first review RL and its relation to continual learning.  We then highlight the critical role of computational constraints in capturing salient trade-offs that arise in continual learning.  Imposing a computational constraint on the general RL formulation gives rise to a coherent objective for continual learning.  Finally, we reflect on several implications of framing continual learning in this manner.

\henrik{As a running example throughout this section, we will consider the agent from Section~\ref{sec:introduction} that each day predicts the next day's average electricity price. These predictions may help in tasks like trading energy stocks or allocating energy resources. This example will help contextualize the concepts introduced, including our formalization of the environment, agent, and reward function, and the notion of computational constraint.}

\saurabh{
\subsection{Background Knowledge}\label{sec:background}
This monograph draws upon terminology and background from probability, information theory, reinforcement learning, and neural networks. While the monograph is relatively self-contained with respect to information theory and machine learning, it does assume familiarity with the basics of measure-theoretic probability. Here we provide additional references that can be helpful to the reader. 
\begin{itemize}
    \item Probability: ~\citet{kolmogorov2018foundations} and Section 3.1 of~\citet{jeon2024information}.
    \item Information Theory: Chapters 2 and 10 of~\citet{cover2006elements}.
    \item Reinforcement Learning: Chapters 1--3 and Chapter 6 of~\citet{sutton2018reinforcement}.
\end{itemize}
}

\subsection{Continual Interaction}\label{sec:continual_interaction}

We consider continual interaction across a general agent-environment interface as illustrated in Figure \ref{fig:agent-environment-interface}.  At each time step $t=0,1,2,\ldots$, an agent executes an action $A_t$ and then observes a response $O_{t+1}$ produced by the environment. Actions take values in an action set $\aspace$.  Observations take values in an observation set $\ospace$.  The agent's experience through time $t$ forms a sequence $H_t = (A_0, O_1, A_1, O_2, \ldots, A_{t-1}, O_t)$, which we refer to as its {\it history}.  We denote the set of possible histories by $\mathcal{H} =  \cup_{t=0}^\infty (\aspace \times \ospace)^t$.

\begin{figure}
\centering
\includegraphics[width=3.75in]{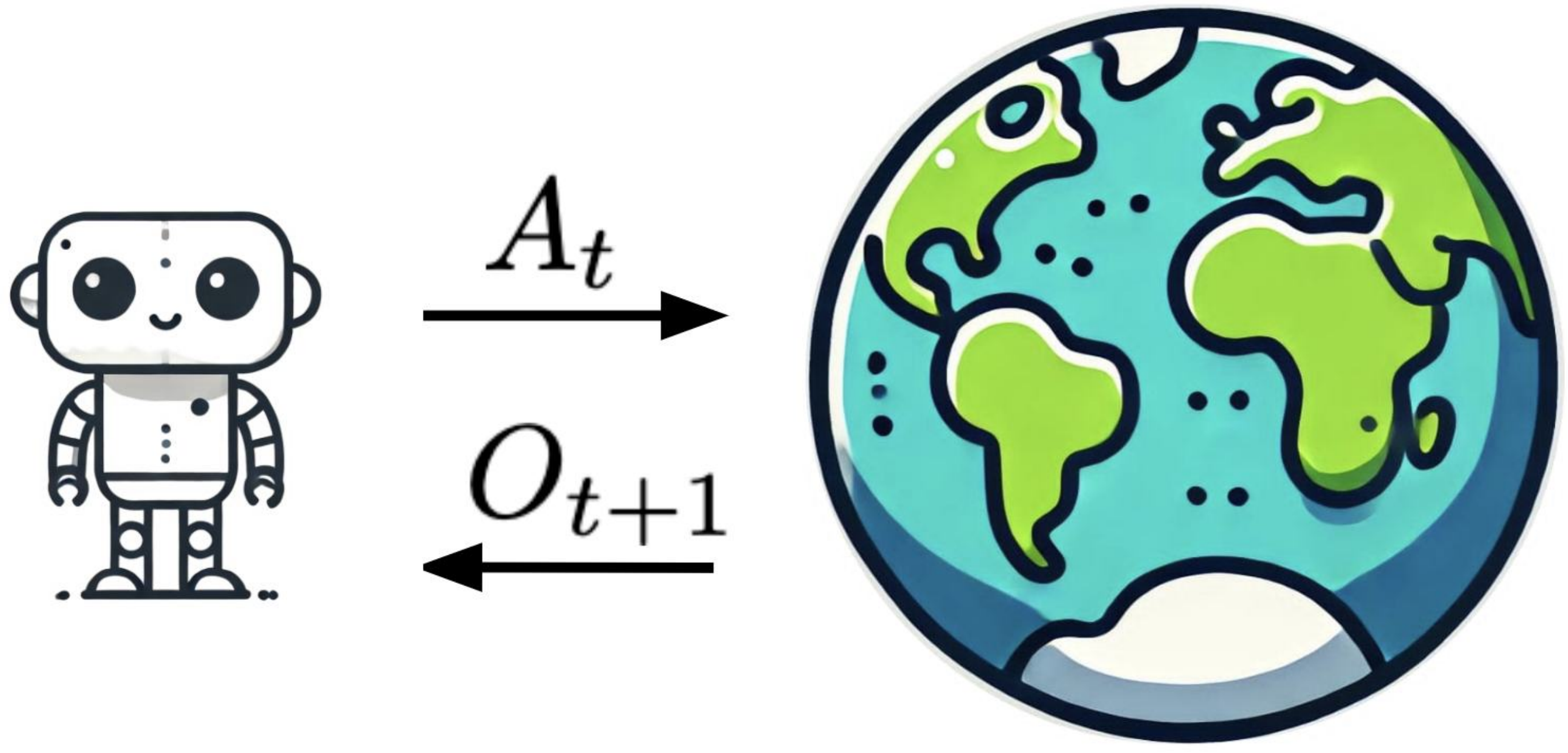}
\caption{The agent-environment interface.}
\label{fig:agent-environment-interface}
\end{figure}

\henrik{All random variables are defined with respect to a common probability space $(\Omega, \mathcal{F}, \mathbb{P})$.}

An environment is characterized by a triple $\env = (\aspace, \ospace, \transition)$, where $\rho$ is an observation probability distribution, such that $\rho(\cdot | H_t, A_t)$ is the distribution over next observations, given history $H_t$ and action $A_t$. We can think of the environment as generating the next observation $O_{t+1}$ by sampling from $\rho(\cdot | H_t, A_t)$.

The agent generates each action $A_t$ based on the previous history $H_t$.  This behavior is characterized by a policy $\pi$, for which $\pi(\cdot|H_t)$ is a probability measure over actions. We refer to this policy, which characterizes the agent's behavior, as the {\it agent policy}. While the agent may carry out sophisticated computations to determine each action, \henrik{the agent can also be described as simply drawing samples $A_t$ from $\pi(\cdot |H_t)$.} 

\henrik{
Formally, we will denote the probability measure over histories induced by $\rho$ and $\pi$ as $\mathbb{P}_{\rho, \pi}$. Thus, $\rho$ and $\pi$ satisfy $\rho(\cdot|H_t,A_t) = \Prob_{\rho, \pi}(O_{t+1} \in \cdot|H_t,A_t)$ and $\pi(\cdot|H_t) = \Prob_{\rho, \pi}(A_t \in \cdot| H_t)$, respectively. When $\rho$ or $\pi$ are clear from context, we will drop either of them, and just write $\Prob_\pi(O_{t+1} \in \cdot|H_t,A_t)$ or $\Prob(O_{t+1} \in \cdot|H_t,A_t)$.  %
}

\henrik{%
We take event probabilities to represent uncertainty from the agent designer's perspective. Let $\tilde{\ospace} \subseteq \ospace$ be a set of observations. Then, $\mathbb{P}(O_{t+1} \in \tilde{\ospace} |H_t,A_t)$ is the designer's subjective assessment of the chance that the next observation will fall in the set $\tilde{\ospace}$ conditioned on $H_t$ and $A_t$. Similarly, let $\tilde{\aspace} \subseteq \aspace$  be a set of actions. Then, $\mathbb{P}_\pi(A_{t+1} \in \tilde{\aspace} |H_t)$ is the designer's subjective assessment of the chance that the next action will fall in the set $\tilde{\aspace}$ assuming the agent implements $\pi$ and conditioned on history $H_t$. }With some abuse of notation, as shorthand, with singleton sets $\tilde{\aspace} = \{a\}$ and $\tilde{\ospace} = \{o\}$, we write $\pi(a|h) \equiv \pi(\{a\}|h)$ and $\rho(o|h,a) \equiv \rho(\{o\} | h,a)$.

\henrik{In order to illustrate how this formalism can express practical agent-environment interactions, we now apply our formalism of environment and agent to the electricity price prediction example from the start of this section.
\begin{itemize}
    \item \textbf{Observation space.} The observation space is $\mathcal{O} = \mathbb{R}^d$, where $d \in \mathbb{Z}_{++}$, the set of positive integers. For $o \in \mathcal{O}$, each component of the observation vector is a measurement of a different quantity. These quantities may include the current spot price, weather and climate data (temperature, humidity, wind speed, rainfall amount, etc), and energy supply data (power plant functionality, fuel prices, renewable energy availability, etc).
    \item \textbf{Action space.} The action space is $\mathcal{A} = \mathbb{R_+}$. The action $A_t$ corresponds to the agent's prediction about the stock price at the next timestep $t+1$. In a more sophisticated application, the action may be whether or not to buy or sell electricity at a certain price.
    \item \textbf{Observation probability function.} The observation probability function $\rho$ represents the agent designer's beliefs about how observations are generated. This function is necessary to make it a coherent decision-making problem. In any real-world application, however, this function typically cannot be written down. Thus, it is primarily when studying an algorithm either theoretically or in simulation that $\rho$ is written down explicitly.
    \item \textbf{Agent policy.} Consider an agent using a recurrent neural network (RNN) to predict the next day's price. At each timestep, the RNN takes as input the current hidden state and the most recent observation. The network outputs a prediction of the next day's price. To train this network, stochastic gradient descent (SGD) is used to continuously update the RNN. Then, for any finite history $H_t$, the RNN together with the training procedure, induces a distribution over predictions. That is, the RNN, together with the training procedure, induces the agent policy $\pi(\cdot | H_t)$.
\end{itemize}}

The following simple example offers a \henrik{more} concrete instantiation of our formulation and notation.
\begin{example}
\label{ex:coins-fixed-biases}
{\bf (coin tossing)} Consider an environment with two actions $\aspace = \{1,2\}$, each of which identifies a distinct coin.  At each timestep $t$, an action $A_t$ determines which coin to toss next.  Observations are binary, meaning $\ospace = \{0,1\}$, with $O_{t+1}$ indicating whether the selected coin lands heads.  The coin biases $(p_1, p_2)$ are independent but initially unknown, and the designer's uncertainty prescribes prior distributions.  These environment dynamics is characterized by a function $\rho$ for which \henrik{$\rho(1|h_t,a) = \mathbb{E}[p_a | H_t = h_t]$ is the probability that coin $a$ lands heads, given some history $h_t$}. %
\end{example}
As a concrete special case, suppose the prior distribution over each coin's bias is uniform over the unit interval.  Then, at each timestep $t$, each bias $p_a$ is distributed $\mathrm{beta}(\alpha_{t,a},\beta_{t, a})$, with parameters initialized at $\alpha_{0,a}=\beta_{0,a}=1$ and updated according to
$(\alpha_{t+1,A_t}, \beta_{t+1,A_t}) = (\alpha_{t,A_t} + O_{t+1}, \beta_{t,A_t} + 1 - O_{t+1})$.  Hence, $\rho(1|H_t,A_t) = \alpha_{t,A_t} / (\alpha_{t,A_t} + \beta_{t,A_t})$.

\subsubsection*{Characterizing the Environment}
The form of interaction we consider is fully general \henrik{in the sense that} each observation can exhibit {\it any} sort of dependence on history.  Notably, we do not assume that observations identify the state of the environment.  While such an assumption -- \henrik{that the sequence of observations obeys the Markov property} -- %
is common to much of the RL literature \citep{sutton2018reinforcement}, a number of researchers have advocated for the general action-observation interface, especially when treating design of generalist agents for complex environments \citep{daswani2013q,daswani2014feature,dong2022simple,hutter2007universal,lu2021reinforcement,mccallum1995instance,Ring:1994,ring2005toward}. \henrik{This is in contrast to making strong assumptions about the environment dynamics, such as assuming observations obey the Markov property, which may be useful in more specialized and narrow applications.}

\henrik{There are two equivalent ways to characterize the environment. First, in the Bayesian framing, there is an observation probability function $\rho_\theta$, typically parameterized by some unknown parameter $\theta \in \Theta$. There are two primitives: a prior over $\theta \in \Theta$, and a ``model" that prescribes an observation probability function $\rho_\theta(O_{t+1} \in \cdot | H_t, A_{t+1})$ for each realization of $\theta$. Because $\theta$ is a random variable, so is $\rho_\theta$. %
Alternatively, we can characterize the same environment through a single primitive: the so-called \textit{posterior predictive distribution}. In particular, observations generated by a function $\rho_\theta$ that is a random variable are indistinguishable from observations generated by a known function $\rho$ (not a random variable) defined as $\rho(\cdot|H_t, A_t) = \E[\rho_\theta(\cdot|H_t,A_t) | H_t, A_t]$. While there is no unique decomposition into a prior and likelihood, the posterior predictive distribution is unique. In Section \ref{se:continual-vs-convergent}, we discuss how some environments have a more ``natural" decomposition into a prior and likelihood compared to other environments. In our most general formulation above, we opt to characterize the environment through a single primitive $\rho$.      %

To make this concrete, consider the environment of Example \ref{ex:coins-fixed-biases}. This environment is characterized by a unknown parameter $\theta = (p_1, p_2)$ and observation probability function $\rho_{\theta} (1| h, a) = p_a$ for each $h \in \mathcal{H}$ and $a \in \mathcal{A} = \{1, 2\}$. The prior distribution over $\theta$ induces a prior distribution over $\rho_{\theta}$.  This environment can alternatively be characterized by the known (not a random variable) function $\rho$ for which $\rho(1|H_t, A_t) = \E[\rho_\theta(1|H_t,A_t)|H_t,A_t]$. 
If the prior distribution over coin biases is uniform then $\rho(1|h, a)=\frac{\alpha_{a,t}}{\alpha_{a,t} + \beta_{a,t}} = 1 - \rho(0|h,a)$ for each $h$ and $a$.  
}

\henrik{
Our definition of an environment ($\env = (\aspace, \ospace, \rho)$) is compatible with any modeling choice. For instance, sometimes it is helpful to model the environment as a \textit{partially observable Markov Decision Process} (POMDP)~\citep{russell2016artificial}. In the POMDP framing, we think of the environment as characterized by a latent MDP that generates observations. While we will not elaborate on it here, any POMDP can be characterized by a tuple $\env = (\aspace, \ospace, \rho)$ and is therefore compatible with our definition. Finally, we note that there will always be many equivalent ways to characterize an environment. We chose our definition of an environment because it is minimal and does not commit to any specific ``view'' of how observations are generated. In particular, our definition does not pose any specific latent structure. }

\subsection{Average Reward}\label{sec:avg_reward}

The agent designer's preferences are expressed through a reward function $r : \mathcal{H} \times \aspace \times \ospace \rightarrow \mathbb{R}$. The agent computes a reward $R_{t+1} = r(H_t,A_t,O_{t+1})$ at each timestep.  As is customary to the RL literature, this reward indicates whether the agent is achieving its purpose, or goals.

A coherent objective requires trading off between short and long term rewards because decisions expected to increase reward at one point in time may reduce reward at another.  As a continual learning agent engages in a never-ending process, a suitable objective ought to emphasize {\it the long game}.  In other words, a performant continual learning agent should attain high expected rewards over asymptotically long horizons.  This behavior is incentivized by an average reward objective function:
\begin{equation}\label{eq:avg_reward_objective}
\overline{r}_\pi = \liminf_{T\rightarrow\infty} \E_\pi\left[\frac{1}{T} \sum_{t=0}^{T-1} R_{t+1} \right].
\end{equation}
\henrik{
The subscript $\pi$ of $\E_\pi$ indicates that the expectation is calculated under $\Prob_\pi$. That is, the expectation integrates with respect to probability distributions prescribed \henrik{by $\pi$ and $\rho$.}}
We will frame the goal of agent design to be maximizing average reward subject to a computation constraint that we will later introduce and motivate.

\henrik{We note that we use the \textit{limit inferior} (denoted by $\liminf$) rather than the limit (which would be denoted by $\lim$) in the objective function. The reason is that the sequence of partial sums of rewards may oscillate and therefore not converge (or diverge). In the case when the sequence is oscillating, $\liminf$ will pick out the smallest value of the limits of all convergent subsequences. Thus, we can think of the objective function as computing a kind of ``worst-case'' long-term average reward achieved by the agent policy.}

\henrik{In the context of our electricity prediction example at the beginning of this section, one possible reward function could be $r(H_t, A_t, O_{t+1}) = -(p_{t+1} - A_t)^2$, the negative squared error between the agent's prediction and the electricity price $p_{t+1}$ at the next timestep. Note that in some applications, the reward function may not be differentiable. When using SGD to train a neural network, a proxy function may be used that is different from the reward function.}

\henrik{
\subsubsection*{Implications of Average Reward}
}

While some work on continual learning has advocated for average reward as an objective \citep{chen2022you,sharma2021autonomous}, discounted reward has attracted greater attention \citep{khetarpal2022towards,ring2005toward}.  Perhaps this is due to the technical burden associated with the design and analysis of agents that aim to maximize average reward.  We find average reward to better suit the spirit of continual learning as it emphasizes long-term performance.  Continual learning affords the possibility of learning very sophisticated skills that build on experience accumulated over a long lifetime, and emphasizing long-term behavior incentivizes the design of agents that learn such skills if possible.

The continual learning literature has spawned a multitude of other metrics for assessing agent performance. Rather than performing a holistic evaluation, prominent criteria tend to focus on detecting particular behaviors, \henrik{such as the ability to more quickly learn to perform a new task, the ability to generalize previously learned knowledge to new tasks, and the ability to perform an old task well after learning a new task~\citep{vallabha2022lifelong}}. %
Average reward subsumes such criteria when they are relevant to online operation over an infinite time horizon. \henrik{For instance, if knowledge from a previous task transfers to a future task, boosting the agent's initial performance on that task, average reward will increase. On the other hand, the ability to perform well on an old task that will never reappear is irrelevant, and average reward is appropriately insensitive to that.}

\henrik{
Assessing performance in terms of average reward has three important implications for desired agent behavior:
\begin{enumerate}
    \item The agent benefits from remembering information from previous tasks only to the extent that this information helps decision making at future timesteps. This may seem obvious, but we emphasize this implication since one of the most common metrics in the continual learning literature is performance on previous tasks. This metric may emphasize the agent's ability to remember past information more than necessary.
    \item Even if the agent forgets recurring information, the agent can do well if it can relearn quickly when needed. This is intuitive: a competent software engineer who forgets a programming language can, when required for a new year-long project, quickly relearn it and successfully complete the project.
    \item The agent can benefit from predicting changing patterns. Modeling dynamics can help the agent decide what skills to retain. For instance, certain recurrence is periodic and therefore predictable, like queries about ice cream during the summer. Something not recurring may also be predictable such as when an elected official finishes their term and steps out of the limelight. An agent could in principle learn to predict future events to prioritize skills. 
\end{enumerate}

Our theoretical and empirical analyses will further elucidate these implications of the average reward objective.

}

\henrik{
\subsubsection*{Implications on Finite Time Behavior}

In some environments, average reward may seem like a strange choice of objective for assessing agent performance. For instance, in the environment described in Example \ref{ex:coins-fixed-biases}, two policies that eventually converge on the better coin will achieve maximum average reward, \textit{regardless} of how quickly they converge. The fact that our objective cannot discriminate between these policies poses a limitation. 

It is therefore common to supplement the average reward objective function with additional criteria that are sensitive to finite-time performance. For instance, under the Blackwell optimality criterion, one policy is superior to another policy if there exists some $\lambda \in [0,1)$, such that for all discount factors greater than $\lambda$ the policy is optimal~\citep{dekker1988average,dewanto2020average}. This kind of objective is very selective and will take into account finite behavior. In fact, it will also imply average reward optimality.

While additional criteria may sometimes be useful, we suspect that in many continual learning environments, such criteria may not be necessary. We provide two reasons. First, in continual learning, we consider environments in which an agent ought to continually acquire new information rather than effectively complete its learning after some time.  As a result, our objective function will tend to distinguish between agents that learn quickly and agents that learn slowly; learning slowly will reduce average reward. As a concrete example, consider the following modification to Example \ref{ex:coins-fixed-biases} in which the agent is incentivized to continually learn.
\begin{example}
\label{ex:coins-evolving-biases}
{\bf (coin replacement)} Recall the environment of Example \ref{ex:coins-fixed-biases}, but suppose that, at each timestep $t$, before action $A_t$ is executed, each coin $a$ is replaced by a new coin with some fixed probability $q_a$.  Coin replacement events are independent, and each new coin's bias is independently sampled from its prior distribution. With this change, biases of the two available coins can vary over time.  Hence, we introduce time indices and denote biases by $(p_{t,1}, p_{t,2})$.
\end{example}
In this environment, agents that learn quickly attain higher average reward. In particular, there is benefit to quickly learning about coin biases and capitalizing on that knowledge for as long as possible before the coins are replaced. Our objective function appropriately incentivizes this kind of behavior.

That said, even in this modified coin-flipping environment, our objective function does not distinguish between certain policies that we may wish it would.  For instance, consider two agents that are identical except that the second agent \textit{only} takes action $1$ for the first $10,000$ timesteps, and therefore performs very poorly. Both agents will achieve the same average reward since the reward for any finite sequence of timesteps will have no impact on the longterm average reward. As we mentioned earlier, one solution to this limitation is to introduce a second criterion. However, we hypothesize that due to capacity constraints on the agent, a secondary criterion may be unnecessary. We elaborate on this in Appendix \ref{app:objective}.

}

\henrik{We apply our formalism of environment, agent, and reward function to some simple examples of continual learning agents in Appendix~\ref{app:cl_agents}.}

\subsection{Computational Constraints}\label{sec:computational_constraints}

\henrik{A continual learning agent processes an endless data stream.} An agent with bounded computational resources, regardless of scale, cannot afford to query every data point in its history $H_t$ at each timestep because this dataset grows indefinitely.  Instead, such an agent must act based on a more concise representation of historical information.  In complex environments, this representation will generally forgo knowledge, some of which could have been used given greater computational resources.  The notion that more compute ought to always be helpful in complex environments may be intuitively obvious.  Nevertheless, it is worth pointing out corroboration by extensive empirical evidence from training large models on text corpi, where performance improves steadily along the range of feasible compute budgets  \citep{Brown202Language,hoffmann2022training,rae2022scaling,smith2022using,thoppilan2022lamda}. \henrik{Note that if data is limited, more computation will likely come with significant diminishing returns; that is, after some point, additional computational resources add very little benefit. In this monograph, we have in mind very complex environments where data is abundant relative to the complexity of the agent.}

In order to reflect the gating nature of computational resources in continual learning, we introduce a per-timestep computational constraint, as considered, for example, by \cite{bagus2022supervised,lesort2020continual,prabhu2023computationally}.  This specializes the continual learning problem formulation of \cite{Abel2023definition}, which recognizes that constraints on the set of feasible agent policies give rise to continual learning behavior but does not focus on computation as the gating resource.  As an objective for continual learning we propose maximization of average reward subject to this constraint on per-timestep computation.  We believe that such a constraint is what gives rise to salient challenges of continual learning. 
\henrik{
To see this, let us compare some consequences of having unlimited versus limited compute resources.
\begin{enumerate}
    \item \textbf{Unlimited computational resources.} If given an unlimited computational budget (and assuming unlimited memory), an agent can both remember everything and continually ingest new information.  With unlimited computational budgets, catastrophic forgetting as it occurs with neural networks could be avoided by just retraining the whole network from scratch after receiving every data point. This would also remove any loss of plasticity experienced over time. Even better, with unlimited resources the agent could perform perfect Bayesian inference.
    \item \textbf{Limited computational resources.} With limited resources, and assuming that the environment is ``changing" over time, forgetting old knowledge in order to ingest new knowledge can improve the agent's performance. The agent will be incentivized to retain and acquire information that will help it do well in the future and forget information that will not recur or will not be useful until the far future. 
\end{enumerate}
}

Constraining per time step compute gives rise to our formulation of computation-constrained reinforcement learning:
\begin{equation}\label{eq:objective}
\begin{split}
\max_\pi & \quad \overline{r}_\pi \\
\text{s.t.} & \quad \text{computational constraint}.
\end{split}
\end{equation}
of practical computational constraints and which ones are binding vary with prevailing technology and agent designs.  For example, if calculations are carried out in parallel across an ample number of processors, it could be the channels for communication among them that pose a binding constraint on overall computation.  Or, if agents are designed to use a very large amount of computer memory, that can become the binding constraint.  Rather than study the capabilities of contemporary computer technologies to accurately identify current constraints, our above formulation intentionally leaves the constraint ambiguous.  For the purposes of theoretical analysis and case studies presented in the remainder of the monograph, we will usually assume all computation is carried out on a single processor that can execute a fixed number of serial floating point operations per timestep and that this is the binding constraint.  We believe that insights generated under this assumption will largely carry over to formulations involving other forms of computational constraints.

While maximizing average reward subject to a computational constraint offers a coherent objective, an exact solution even for simple, let alone complex, environments is likely to be intractable.  Nevertheless, a coherent objective is valuable for assessing and comparing alternative agent designs.  \henrik{Indeed, algorithmic ingredients often embedded in RL agents are helpful because they enable a favorable tradeoff between average reward and computation. These algorithmic ingredients include Q-learning, neural networks, SGD, and exploration schemes such as optimism and Thompson sampling.
For instance, it seems like using SGD with Adam to train neural networks enables a favorable trade-off as compared to many other second order optimization methods that use more compute without offering much better performance.}  Further, these \henrik{techniques} are scalable in the sense that they can leverage greater computational resources when available.  As \cite{sutton2019bitter} argues, with steady advances in computer technology, agent designs that naturally improve due to these advances will stand the test of time, while those that do not will phase out.

\henrik{Returning to our electricity price prediction example from the start of this section, the computational constraint may take the form of a weekly compute budget. Specifically, the agent has a budget for how many floating point operations it can do per week. For simplicity, we can think of the computational constraint as being in terms of number of parameters multiplied by number of gradient steps. An important implication of this constraint is that it is not possible to retrain the agent's neural network from scratch every day if the network is sufficiently large.}

\henrik{
\subsubsection*{Continual Learning in ``Stationary" Environments}
\citet{sutton2007role} highlights that even if the world can be thought of as ``unchanging" in the sense that it is identified by a finite number of bits, computational constraints may incentivize continual learning behavior. Roughly, suppose the world is much more complex than the agent, and the data stream has temporal correlations. The agent will be incentivized to retain and acquire information that will be useful in the somewhat near-term future and forget information that will only be useful in the far future.  We elaborate on this in Section \ref{sec:subsection_cl}.

To make this more concrete, let us return again to our electricity price prediction example. Let us assume the environment is ``stationary" in the following sense: electricity price movement is seasonal, and there are only two seasons. Consider a resource-rich agent which takes as input the current season and the past $k$ observations. The parameters of this agent will eventually converge, having learned the two season-modulated functions of the past $k$ observations. Since the agent eventually stops learning, it is not an example of what we think of as a continual learning agent. 

Now, consider a smaller agent with a small neural network that cannot store both functions. This smaller agent will be incentivized to learn each season's function anew as the season arrives. This agent will be what we would consider a continual learning agent. More generally, this type of agent is incentivized to acquire information that will recur in the somewhat shorter term and forget information that will not recur, or will not recur in the short term. (Humans are constrained agents also, and therefore exhibit re-learning behavior; for instance, a person may not remember exactly how to fill out their tax forms and may have to relearn every year. This is despite the task being the same.)
}

\subsection{Continual Supervised Learning}\label{sec:csl_special_case}
Much of the literature on continual learning focuses on supervised learning (SL).  As noted by \cite{khetarpal2022towards}, continual SL is a special case of reinforcement learning.  In particular, \henrik{in supervised classification problems, we} take each observation to be a data pair $O_t = (Y_t, X_t)$, consisting of (1) a label $Y_t$ assigned to the previous input $X_{t-1}$ and (2) a next input $X_t$.  Labels take values in a finite set $\yspace$ and inputs take values in a set $\xspace$ which could be finite, countable, or uncountable.  The set $\ospace$ of observations is a product $\yspace \times \xspace$.  Take each action $A_t$ to be a predictive distribution $P_t$, which assigns a probability $P_t(y)$ to each label $y \in \yspace$.  Hence, $P_t$ takes values in a unit simplex $\Delta_\yspace$.  We view $P_t$ as a prediction of the label $Y_{t+1}$ that will be assigned to the input $X_t$.  The observation probability function $\rho$ samples the next data pair $O_{t+1} = (Y_{t+1}, X_{t+1})$ in a manner that depends on history only through past observations, not past actions.  Finally, take the reward function to be
$r(H_t, A_t, O_{t+1}) = \ln P_t(Y_t)$.
With this formulation, average reward is equivalent to average negative log-loss:
\begin{equation}\label{eq:negative_log_loss}
\overline{r}_\pi = \liminf_{T \rightarrow \infty} \E_\pi\left[\frac{1}{T} \sum_{t=0}^{T-1} \ln P_t(Y_{t+1}) \right].
\end{equation}
This is a common objective used in %
\henrik{online supervised classification~\citep{fogel2017problem} and sequential prediction problems~\citep{shkel2018sequential}.}
The following agent is designed to minimize log-loss.  Recall that $\Delta_{\mathcal{Y}}$ denotes the unit simplex or, equivalently, the set of probability vectors with one component per element of $\mathcal{Y}$.
\begin{example}
\label{ex:SGD}
{\bf (\henrik{online classification with neural networks})}
Consider an input space $\mathcal{X}$ and finite set of labels $\mathcal{Y}$.  This agent is designed to interface with actions $\aspace = \Delta_{\mathcal{Y}}$ and observations $\ospace = \mathcal{Y} \times \mathcal{X}$.  Consider a 
neural network with a softmax output %
\henrik{layer.}
The inference process maps an input $X_t$ to a predictive distribution $P_t(\cdot) = f_{\theta_t}(\cdot|X_t)$.  Here, $f$ is an abstract representation of the neural network and $\theta_t$ is the vector of parameters (weights and biases) at timestep $t$.  Trained online via stochastic gradient descent (SGD) with a fixed stepsize to reduce log-loss, these parameters evolve according to
$$\theta_{t+1} = \theta_t + \alpha \nabla \ln f_{\theta_t}(Y_{t+1}|X_t).$$
This is a special case of our general reinforcement learning formulation, with action $A_t = P_t$, observation $O_{t+1} = (Y_{t+1}, X_{t+1})$, and reward $r(H_t, A_t, O_{t+1}) = \ln P_t(Y_{t+1})$.
\end{example}
Note that we could alternatively consider average accuracy as an objective by taking the action to be a label $A_t = \hat{Y}_{t+1}$.  This label  could be generated, for example, by sampling uniformly from $\argmax_{y \in \mathcal{Y}} P_t(y)$.  A reward function $r(H_t, A_t, O_{t+1}) = \mathbbm{1}(Y_{t+1} = \hat{Y}_{t+1})$ can then be used to express accuracy.  This is perhaps the objective most commonly used in classification.

The online classification agent is designed for a prototypical supervised learning environment where the relationship between inputs and labels is characterized by a random latent function $F$\henrik{$: \mathcal{X} \rightarrow \Delta_{\mathcal{Y}}$}. In particular, \henrik{as a random variable, $F$ is independent of the inputs} and 
$\Prob(Y_{t+1} \in \cdot | F, H_t) = F(\cdot|X_t)$.

The agent can also be applied to a nonstationary supervised learning environment, where the latent function varies over time, taking the form of a stochastic process $(F_t: t \in \mathbb{Z}_+)$.  With this variation, %
\henrik{$F_t$ is independent of the inputs}
and $\Prob(Y_{t+1} \in \cdot | F_t, H_t) = F_t(\cdot|X_t)$.  However, due to loss of plasticity, incremental learning with neural networks does not perform as well in such an environment as one would hope \citep{dohare2021continual}.  In particular, while the nonstationarity makes it important for the agent to continually learn, its ability to learn from new data degrades over time.  A simple alternative addresses this limitation by periodically replacing the model under use with a new one, trained from scratch on recent data.
\begin{example}
\label{ex:rejuvenation}
{\bf (model replacement)}
Given a neural network architecture and algorithm that trains the model on a fixed batch of $N$ data pairs, one can design a continual supervised learning agent as follows.  At each timestep $t = 0, \tau, 2\tau, \ldots$, reinitialize the neural network parameters and train on the $N$ most recent data pairs $(X_{t-n}, Y_{t+1-n}: n = 1,\ldots,N)$.  No further training occurs until time $t+\tau$, when the model is reinitialized and retrained.  Each prediction $A_t = P_t$ is given by $P_t(\cdot) = f_{\theta_t}(\cdot|X_t)$, where $\theta_t$ parameterizes the most recent model.  In particular, $\theta_{t+1} = \theta_t$ unless $t$ is a multiple of $\tau$.  The hyperparameters $\tau$ and $N$ specify the replacement period and number of data pairs in each training batch.
\end{example}
This approach to continual learning is commonly used in production systems. Consider the example at the start of this section in which an agent predicts the next day's average electricity price $Y_{t+1} $\henrik{(the set $\mathcal{Y}$ is not necessarily finite)}. A prototypical system might, at the end of each month, initialize a neural network and train it on, \henrik{for example}, the preceding twelve months of data. Then, this model could be used over the subsequent month, at the end of which the next replacement arrives. The reason for periodically replacing the model is that very recent data is most representative of future price patterns, which evolve with the changing electricity market. 

The reason for not replacing the model more frequently is the cost of training.  There are a couple reasons for training only on recent history, in this case over the past twelve months.  One is that recent data tends to best represent patterns that will recur in the future.  However, this does not in itself prevent use of more data; given sufficient computation, it may be beneficial to train on all history, with data pairs suitably weighted to prioritize based on recency.  The binding constraint is on computation, which scales with the amount of training data.  

While model replacement is a common approach to continual learning, it is wasteful in and limited by its use of computational resources.  In particular, each new model does not leverage computation invested in past models because it is trained from scratch.  Developing an incremental training approach that affords a model benefits from all computation carried out since inception remains an important challenge to the field. 
\henrik{Further, the limitations of the model replacement approach} highlight the need for computational constraints in formulating a coherent objective for continual learning that incentivizes better agent designs.

\subsection{Learning Complex Skills over a Long Lifetime}\label{sec:learning_skills}

An aspiration of continual learning is to design agents that exhibit ever more complex skills, building on skills already developed \citep{Ring:1994}.  As opposed to paradigms that learn from a fixed data set, this aspiration is motivated by the continual growth of the agent's historical dataset and thus, information available to the agent.  With this perspective, continual learning researchers often ask whether specific design ingredients are really needed or if they should be supplanted by superior skills that the agent can eventually learn.  For example, should an agent implement a hard-coded exploration scheme or learn to explore?  Or ought an agent apply SGD to update its parameters rather than learn its own adaptation algorithm?

\henrik{There is always room to improve average reward by designing the agent to learn more sophisticated skills.}  However, as we will discuss in the next section, this complexity is constrained by the agent's information capacity, which is gated by computational constraints.  There are always multiple ways to invest this capacity.  For example, instead of maintaining statistics required by a stepsize adaptation scheme, a designer could increase the size of the neural network, which might also increase average reward.

Related to this is the stability-plasticity dilemma \citep{mccloskey1989catastrophic,ratcliff1990connectionist}, a prominent subject in the continual learning literature.  Stability is the resilience to forgetting useful skills, while plasticity is the ability to acquire new skills.  Empirical studies demonstrate that agents do forget and that, as skills accumulate, become less effective at acquiring new skills \citep{dohare2021continual,goodfellow2013empirical,kirkpatrick2017overcoming,lesort2020continual}.   Researchers have worked toward agent designs that improve stability and plasticity.  But limited information capacity poses a fundamental tradeoff.  For example, \cite{mirzadeh2022wide} and \cite{dohare2021continual} demonstrate that larger neural networks forget less and maintain greater plasticity.  And as we will further discuss in Section \ref{se:stability_plasticity}, in complex environments, constrained agents must forget and/or lose plasticity, with improvements along one dimension coming at a cost to the other.

\begin{summary}
\begin{itemize}
\item An {\bf environment} is characterized by a tuple $(\aspace, \ospace, \rho)$, comprised of a set of {\bf actions}, a set of {\bf observations}, and an {\bf observation probability function}.
\item The agent's experience through time $t$ forms a {\bf history} $H_t = (A_0,O_1,\ldots,A_{t-1},O_t)$.
\item \henrik{Observations are sampled as}
$$O_{t+1} \sim \rho(\cdot | H_t, A_t).$$
\item The behavior of an agent is characterized by an {\bf agent policy} $\pi$.  \henrik{Actions are sampled as} $$A_t \sim \pi(\cdot|H_t).$$
\item The designer's preferences are encoded in terms of a {\bf reward function} $r$, which generates {\bf rewards} 
$$R_{t+1} = r(H_t, A_t, O_{t+1}).$$
\item The {\bf average reward} attained by an agent policy $\pi$ in an environment $(\aspace, \ospace, \rho)$ is
$$\overline{r}_\pi = \liminf_{T \rightarrow \infty} \E_\pi\left[\frac{1}{T} \sum_{t=0}^\infty R_{t+1}\right].$$
\item Design of a continual learning agent can be framed as {\bf maximizing average reward subject to a per-timestep computational constraint}:
\begin{equation*}
\begin{split}
\max_\pi & \quad \overline{r}_\pi \\
\text{s.t.} & \quad \text{computational constraint}.
\end{split}
\end{equation*}

\henrik{An agent policy is induced by a choice of algorithm. This algorithm, however, may require more computational resources than the computation constraint allows for. Therefore, by restricting the choice of algorithm, the computational constraint also restricts what agent policies are acceptable.}

\item {\bf Continual supervised learning with log-loss} is a special case in which 
\begin{itemize}
\item $\ospace = \mathcal{X} \times \mathcal{Y}$, where $\mathcal{X}$ and $\mathcal{Y}$ are input and label sets,
\item $\aspace = \Delta_{\mathcal{Y}}$ is the unit simplex of predictive distributions,
\item each observation is a pair $O_{t+1} = (Y_{t+1}, X_{t+1})$ comprising a label assigned to the previous input $X_t$ and the next input $X_{t+1}$,
\item the observation distribution $\rho(\cdot|H_t,A_t)$ depends on $(H_t,A_t)$ only through past observations $O_{1:t}$,
\item the reward function expresses the negative log-loss $R_{t+1} = \ln P_t(Y_{t+1})$.    
\end{itemize}
Another common reward function used in supervised learning expresses the accuracy $R_{t+1} = \mathbbm{1}(Y_{t+1} = \hat{Y}_{t+1})$, where the action is a label $A_t = \hat{Y}_{t+1}$.
\end{itemize}
\end{summary}
\clearpage

\section{Agent State and Information Capacity}
\label{sec:information}

Practical agent designs typically maintain a bounded summary of history, which we refer to as the {\it agent state} and which is used to select actions.  Information encoded in the agent state is constrained to regulate computational requirements.  In this section, we formalize these concepts in information-theoretic terms, along the lines of \cite{jeon2023informationtheoretic,lu2021reinforcement}, and explore their relation to agent performance.  The tools we develop allow us to more clearly distinguish continual from convergent learning and define and analyze stability and plasticity, as we do in Sections \ref{se:stability_plasticity} and \ref{se:continual-vs-convergent}.

\subsection{Agent State}

Computational constraints prevent an agent from processing every element of history at each timestep because the dataset grows indefinitely.  To leverage more information than can be efficiently accessed from history, the agent needs to maintain a representation of knowledge that enables efficient computation of its next action.  In particular, the agent must implement a policy $\pi$ that depends on a statistic $U_t$ derived from $H_t$, rather than directly on $H_t$ itself.  Such a policy samples each action according to
$$A_t \sim \pi(\cdot|U_t).$$

The statistic $U_t$ must itself be computed using budgeted resources. An agent that computes $U_t$ directly from $H_t$ would run into constraints of the same sort that motivated construction of agent state in the first place.  In particular, the agent cannot access all history within a timestep and must maintain an agent state that facilitates efficient computation of the next agent state. To this end, $U_t$ serves two purposes: computation of $A_t$ and $U_{t+1}$.  Specifically, there must be an update function $\psi$ such that
$$U_{t+1} \sim \psi(\cdot|U_t, A_t, O_{t+1}).$$
This sort of incremental updating allows $U_{t+1}$ to selectively encode historical information while amortizing computation across time. Since $U_t$ includes all information about $H_t$ that the agent will subsequently use, it can be thought of as a state; thus the term {\it agent state}. \henrik{So far, we have not added any information constraint, which means that the agent could in principle store the full history. In subsequent sections, we will introduce a constraint on the agent state $U_{t}$ limiting how much information the agent can store.}

\henrik{The agent presented in the \henrik{online classification} example in Section~\ref{sec:csl_special_case} (Example \ref{ex:SGD}) uses a highly compressed representation of history as agent state. In particular, the agent state $U_t = (X_t, \theta_t)$ consists of the current input vector and a vector of neural network parameters.
If the agent were to also maintain a replay buffer $B_t$ of recent action-observation pairs for supplemental training, the replay buffer would also reside within the agent state $U_t = (X_t, \theta_t,B_t)$.  In each of these examples, the agent state is updated incrementally according to $U_{t+1} \sim \psi(\cdot|U_t, A_t, O_{t+1})$ for some function $\psi$. In Appendix \ref{app:cl_agents}, we present additional examples of continual learning agents (tracking, bandit learning, and Q-learning), all of which maintain agent states as compressed representations of history.}

\subsection{Information Content}
Intuitively, the agent state retains information from history to select actions and update itself.  But how ought this information be quantified?  In this section, we offer a formal approach based on information theory \citep{shannon1948mathematical}.  Attributing precise meaning to {\it information} affords coherent characterization and analysis of information retained by an agent as well as forgetting and plasticity, which are subjects we will discuss in subsequent sections.

\textbf{Quantifying information as length of lossless encoding.} \\
The agent state $U_t$ is a random variable since it depends on other random variables. First, it depends on the agent's experience $H_t$ up till time $t$. Second, it may also depend on algorithmic randomness arising in agent state transitions.  An intuitive way of quantifying the information content of $U_t$ is in terms of the minimal number of bits that suffices for a lossless encoding.  Or, using a unit of measurement more convenient for analysis of machine learning, the number of nats; there are $\log_2 e$ bits per nat.

\textbf{Entropy as a proxy for nats required.} \\
\henrik{To facilitate analysis, we will use entropy $\Ent(U_t)$ as a proxy for lossless encoding length; that is, the number of nats that is necessary and sufficient for storing the agent state $U_t$. Thus, we will use entropy $\Ent(U_t)$ as a proxy for information content. }

\henrik{The entropy of a random variable $B$ with countable range is defined by
$$\Ent(B) = \sum_{b \in \mathrm{range(B)}} - \Prob(B=b) \ln \Prob(B=b).$$

More generally, if the range is uncountable then entropy is defined by $\Ent(B) = \sup_{f \in \mathcal{F}_{\mathrm{finite}}} \Ent(f(B))$, where $\mathcal{F}_{\mathrm{finite}}$ is the set of functions that map $\mathcal{U}$ to a finite range.\footnote{This derives from the master definition of mutual information $\Ent(B) = \sup_{f \in \mathcal{F}_{\mathrm{finite}}} \Ent(f(B))$ \citep[Chapter 8]{cover2006elements} and the relation $\Ent(B) = \mathbb{I}(B,B)$.}
}

\henrik{
\textbf{Short motivation.} \\
We offer a short story to motivate why entropy is a plausible proxy for the lossless encoding length.  Suppose the agent state $U_t$ can be expressed in plain English.  This expression can modeled as a stochastic process, which generates a sequence of characters $X_1, X_2, X_3, \ldots$.  There is some fixed number of characters $N$ that suffices to describe the agent's knowledge; that is, $U_t$ is fully expressed by $X_1,\ldots,X_N$.

Let $p_n(\cdot) = \Prob(X_1,\ldots,X_n = \cdot)$ be the probability mass function of the first $n$ characters generated by the stochastic process.  For each $n$, there is a code that uses at most one more than $- \log p_n(X_1,\ldots,X_n)$ bits \citep[Section~5.3]{cover2006elements}, or $- \ln p_n(X_1,\ldots,X_n)$ nats. This code is optimal in the sense that no other code attains smaller expected encoding length \citep[Section~5.4]{cover2006elements}. Under suitable technical conditions, 
$\frac{1}{n} \ln p_n(X_1,\ldots,X_n)$ converges to the entropy rate of the stochastic process, which is defined by $h = \lim_{n\rightarrow\infty} \frac{1}{n} \Ent(X_1,\ldots,X_n)$ \citep{mcmillan1953basic,breiman1957individual}.  Convergence holds in multiple senses, one of which is convergence in probability; in particular, for all $\epsilon > 0$ and $\delta > 0$, there exists some $n^*$ such that for all $n \geq n^*$,
$$\Prob\left(\left|-\frac{1}{n}\ln p_n(X_1,\ldots,X_n) - h \right| > \epsilon\right) < \delta.$$
Setting $\epsilon = \delta = 0.01$, this implies that after some number of characters $n$, there is a $99\%$ chance that the optimal lossless encoding length differs from $h n$, which varies with the realization of $X_1,\ldots,X_n$, by no more than $1\%$ of $n$.  Fleshing out this story calls for answers to a couple questions: (1) how large does $n$ need to be to attain this low level encoding length variability and (2) is it plausible for the number of characters $N$ that describe the agent state to exceed that.

There is no definitive answer to (1).  However, \cite{takahira2016entropy} suggests that, even with an uninformative prior, less than ten gigabytes of text data is sufficient to estimate the entropy rate.  We would expect the designer's prior to be fairly informed, especially when it comes to comprehension of English text.  Based on this, it is plausible that, in the context of describing agent state, a number of characters $n$ far less than a trillion ought to suffice for attaining a very low level of variability in encoding length.

To answer (2), we need to speculate on the number of characters required to describe the state of an agent.  Modern large language models such as GPT4 use a trillion or so trainable parameters to learn from a trillion or so words.  So it is natural to consider an agent state that requires a trillion words to describe.  Typical English text averages close to five characters per word, so we would be talking about $N$ being five trillion characters or so.}

\subsection{Information Capacity}
\label{se:information-capacity}

Recall our continual learning objective:
\begin{align*}
\max_{\pi} & \quad \overline{r}_\pi \\
\text{s.t.} & \quad \text{computational constraint}.
\end{align*}
The nature of the computational constraint was purposely left ambiguous. If computer memory is binding, that directly constrains information content. \henrik{For many large-scale applications, such as training large language models, compute resources are binding rather than memory in the form of disk space. Training the model is much more costly than the disk space necessary for storing the model and data. Therefore,} we will assume for the remainder of this monograph that the binding constraint is the number of floating point operations (FLOP) that can be carried out per timestep.  This does not necessarily constrain information content.  For example, even if we take the agent state to be $U_t = H_t$ and the entropy $\Ent(U_t)$ grows indefinitely, an agent can efficiently select each action based on sparsely queried data, perhaps by randomly sampling a small number of action-observation pairs from history.  However, as a practical matter, common agent designs apply computation in ways that limit the amount of information that the agent retains.  We refer to the constraint on information content of agent state as the {\it information capacity}.

\saurabh{
Constraining the information content limits how much information from the history can be retained. This constraint is independent of how the agent chooses to represent its knowledge. Consider an agent that observes a sequence of coin flips and tries to predict the next coin flip at each timestep. Clearly, if the information content is constrained, the agent cannot store the full history of observations. But often times, for good performance, it is sufficient that the agent stores the average of the past observations. Even storing the average is ruled out by a constraint on information content. Indeed, even when storing the average, the memory requirements are still a function of history length: as the number of coin flips increases, the number of possible sample averages increases. This means that the encoding length of the sample average would increase. Thus, the agent cannot actually store the average of an ever-growing sequence. If the agent uses a floating point representation of numbers, this will correspond to saying that the agent has finite floating-point precision.
}

\saurabh{
\subsection{Compute versus Information Capacity}\label{sec:compute_vs_infocapacity}

In practical agent designs, computational constraints bind information capacity. To see this, we will consider three common ways to increase information capacity in AI agents using neural networks:   
\begin{enumerate}
    \item Increase the number of neural network parameters.
    \item Increase the number of neural networks in a sparse mixture of experts. 
    \item Increase the size of an external memory buffer.
\end{enumerate}
For each approach, we will discuss how computational constraints bind information capacity.

\textbf{Single neural network.} First, consider an agent equipped with a single neural network. As described in Example \ref{ex:SGD} in Section~\ref{sec:csl_special_case}, the agent carries out a single SGD step over each $t$th timestep:
$$\theta_{t+1} = \theta_t + \alpha \nabla \ln f_{\theta_t}(Y_{t+1}|X_t).$$
Each data pair $(X_t, Y_{t+1})$ is immediately processed when observed, then discarded. Compute per timestep is determined by and grows proportionally with the number of model parameters. Hence, a computation constraint restricts the number of model parameters. The number of model parameters, in turn, will constrain information content. If each parameter is encoded by $K$ nats, a neural network with $N$ parameters can encode $N K$ nats of information.  We refer to this as the {\it physical capacity} of the neural network. In this example, the information capacity grows with the size of the neural network which quickly increases the compute required for making a prediction as well as updating the parameters. 

\textbf{Sparse mixture of experts.} In our second example, we consider an agent implemented as a so-called sparse mixture of experts. With such an approach, the agent consists of a set of $M$ neural networks with each network having $N$ parameters. Each neural network is said to be an \textit{expert}. A \textit{gating mechanism} selects which expert to use for a given input. There are many ways a gating mechanism can be implemented. We consider a gating mechanism implemented as a neural network that for each input outputs $M$ scores, one per expert. The expert with the highest score is selected (see~\citet{shazeer2017outrageously} for a broader discussion of how such neural networks might be implemented and trained). At each timestep, the agent uses the gating mechanism to select one of the $M$ networks and updates that network via an SGD step. As we increase the number of experts, the information capacity grows. However, the compute requirement also grows: the gating network needs to output scores for more experts. 

\textbf{Agent with external memory.} 

\henrik{
As mentioned before, computational constraints do not, in principle, bind the agent's information capacity. For instance, consider an agent equipped with, in addition to a neural network, an external memory used to store a large number of data pairs.  Technically, this agent's information capacity can scale with the size of the external memory in a manner that does not depend on compute resources.  However, for practical agent designs that make use of such a memory, computational constraints do induce capacity constraints.
Compared to using only gradient descent, the memory may enable the agent to utilize each data sample more effectively by storing and retrieving them as needed. Given a new sample, the agent first queries similar examples from the memory that may help in making predictions on the new sample. The neural network then receives both the current sample and the similar examples as input. Using an external memory in a similar manner is a common practice when designing chatbots~\citep{lewis2020retrieval,gao2023retrieval}. However, as we increase the size of the memory, the compute required to search over the memory to select relevant data points increase. The exact nature of the scaling factor depends on how the data is inserted and retrieved from the memory. Thus, a constraint on compute induces an effective capacity constraint. 
}

\textbf{Further remarks.}
More generally, at any point in time, the agent must, implicitly or explicitly, implement a function that maps the current observation to an action. In the case of supervised learning, this is a function mapping from $X_t$ to an estimate of $Y_{t+1}$. Importantly, this function must be implementable on a computer. Suppose we adopt a general model of computation, such as the Turing machine. Then, a computational constraint corresponds to imposing an upper bound on the number of operations the machine can execute per timestep. If we impose such a bound, then the set of implementable functions is necessarily finite. Consequently, since the number of functions is finite, the agent's information capacity must also effectively be bounded. In this sense, computational constraints would bind information capacity.

A similar argument can be extended to other models of computation. For example,~\citet{sontag1998vc} bound the VC dimension of a 1-layer neural networks in terms of the number of neurons, or equivalently, the number of elementary operations. In this case, the set of possible functions is infinite, implying that the total information content may be unbounded. However, since the VC dimension is finite, the ``effective" complexity remains bounded.
}

\subsection{Physical Capacity versus Information Capacity}
While information content is constrained not to exceed the physical capacity, large neural networks trained via SGD \henrik{typically} %
use only a fraction of their physical capacity to retain information garnered from data. Much of the \henrik{neural network capacity (the $NK$ nats) instead serves to facilitate optimization}. %
\henrik{For instance, in some applications, even with over-parameterized neural networks, scaling the neural network improves performance. As another example, consider the case of \textit{distillation} \citep{hinton2015distilling}. You train a small network on the original dataset but use as labels the outputs of the large network. Then, this small network typically performs much better relative to the same small network trained directly on the original labels.} \henrik{All in all,} the physical capacity \henrik{often} constrains the information capacity to far fewer nats: $\Ent(U_t) \ll N K$.

Similar reasoning applies if the agent state is expanded to include a replay buffer.  In this case, the physical capacity is $N K + B$ nats, if $B$ nats are used to store the replay buffer. Again, the information capacity is \henrik{often} constrained to far fewer nats than the physical capacity. As before, only a fraction of the neural network's physical capacity stores information content. But also, a replay buffer that stores raw data from history can typically be compressed losslessly to occupy a much smaller number of nats. \henrik{This suggests that the information content in the replay buffer is much smaller than the physical capacity of the buffer.}

\subsection{Performance versus Information Capacity}
\label{se:information-capacity-vs-performance}

It follows from the definitions of performance and information capacity that if the agent makes efficient use of its information capacity, its performance will increase as this constraint is loosened as long as there is room left for improvement. Information theory offers an elegant \henrik{characterization} of this relation.  To illustrate this, let us work through this \henrik{characterization} for the case of continual supervised learning (SL).

\subsubsection{Prediction Error}

Recall that in our continual SL formulation the agent's action is a predictive distribution $A_t = P_t$ and the reward is taken to be $r(H_t, A_t, O_{t+1}) = \ln P_t(Y_{t+1})$.  Hence, the objective is to minimize average log-loss.  To enable an elegant analysis, we define a prediction
\begin{equation}\label{eq:prediction}
P^*_t = \Prob(Y_{t+1} = \cdot | H_t) = \argmax_{Q \in \Delta_{\mathcal{Y}}} \E[\ln Q(Y_{t+1}) | H_t]
\end{equation}
as a {\it gold standard}. The expected reward $\E[\ln P^*_t(Y_{t+1}) | H_t]$ represents the largest that a computationally unconstrained agent can attain given the history $H_t$.  The difference between this gold standard value and the expected reward attained by the agent is expressed by the KL-divergence:
\henrik{
\begin{align*}
\KL(P^*_t \| P_t) &=  \sum_{y \in \mathcal{Y}} P^*_t(y) \ln \frac{P^*_t(y)}{P_t(y)} \\
&= \E[\ln \frac{ P^*_t(Y_{t+1})}{ P_t(Y_{t+1})} | H_t]   \text{   (by Equation \ref{eq:prediction})} \\
&= \E[\ln P^*_t(Y_{t+1}) - \ln P_t(Y_{t+1}) | H_t].
\end{align*}}

This KL-divergence serves as a measure of error between the agent's prediction $P_t$ and the gold standard $P^*_t$.  Maximizing \henrik{expected} average reward is equivalent to minimizing this prediction error, since
\henrik{
\begin{equation}
\underbrace{\E[\ln P_t(Y_{t+1})]}_{\text{expected reward}} = \underbrace{\E[\ln P^*_t(Y_{t+1})]}_{\text{optimal expected reward}} - \underbrace{\E[\KL(P^*_t \| P_t)]}_{\text{expected prediction error}}
\end{equation}
}
and $\E[\ln P^*_t(Y_{t+1})]$ does not depend on $P_t$.

When making its prediction $P_t$, the agent only has information supplied by its agent state $U_t$.  The best prediction that can be generated based on this information is 
\begin{equation}
\tilde{P}_t = \Prob(Y_{t+1} = \cdot | U_t) = \argmax_{Q \in \Delta_{\mathcal{Y}}} \E[\ln Q(Y_{t+1}) | U_t].
\end{equation}
If the agent's prediction $P_t$ differs from $\tilde{P}_t$, the error attained by the agent decomposes into informational versus inferential components, as established by the following result.
\begin{restatable}{theorem}{error-decomposition}\label{th:error-decomposition}
For all $t$,
$$\underbrace{\E[\KL(P^*_t \| P_t)]}_{\mathrm{prediction\ error}} = \underbrace{\E[\KL(P^*_t\|\tilde{P}_t)]}_{\mathrm{informational\ error}} + \underbrace{\E[\KL(\tilde{P}_t\|P_t)]}_{\mathrm{inferential\ error}}.$$
\end{restatable}
\begin{proof}
For all $t$,
\begin{align*}
\E[\KL(P^*_t \| P_t)]
=& \E[\ln P^*_t(Y_{t+1}) - \ln P_t(Y_{t+1})] \\
=& \E[\E[\ln P^*_t(Y_{t+1}) - \ln \tilde{P}_t(Y_{t+1})|H_t]] + \E[\E[\ln \tilde{P}_t(Y_{t+1}) - \ln P_t(Y_{t+1}) | U_t]] \\
=& \E[\KL(P^*_t\|\tilde{P}_t)] + \E[\KL(\tilde{P}_t\|P_t)].
\end{align*}
\end{proof}

\subsubsection{Informational Error Quantifies Absent Information}

The informational error can be interpreted as historical information absent from the agent state $U_t$ \henrik{that} would be useful for predicting $Y_{t+1}$. \saurabh{We can also think of informational error as being the error incurred due to capacity constraints.} This can be expressed in elegant information-theoretic terms. 

\saurabh{First, a comment on notation: since the entropy function $\Ent$ depends on $\Prob$, which is the probability measure over histories induced by $\rho$ and $\pi$, $\Ent$ is also induced by $\rho$ and $\pi$. Just as we write $\Prob(O_{t+1} \in \cdot|H_t,A_t)$ as shorthand for  $\Prob_{\rho, \pi}(O_{t+1} \in \cdot|H_t,A_t)$, we write $\Ent(U_t)$ as shorthand for $\Ent_{\rho, \pi}(U_t)$.
}

We now review a few information measures for which Figure \ref{fig:info_venn} illustrates intuitive relationships.  Let $B$ and $C$ be random variables, and to simplify, let us assume for this discussion that each has countable range. The concepts extend to uncountable ranges. The conditional entropy of $B$ conditioned on $C$ is defined by
\henrik{$$\Ent(B|C) = \E\left[- \ln \Prob_{B|C}(B| C)\right]$$
where $\Prob_{B|C}(b | c) = \Prob(B=b | C=c)$.}

It follows from the definition of conditional probability that $\Ent(B|C) = \E\left[- \ln \Prob_{B,C}(B,C) + \ln \Prob_C(C)\right] = \Ent(B,C) - \Ent(C)$.
This represents the expected number of nats that remain to be revealed by $B$ after $C$ is observed, or the union of the two discs in the venn diagram minus the content of the blue disc.  The mutual information between $B$ and $C$ is defined by
$$\I(B; C) = \Ent(B) - \Ent(B|C) = \Ent(C) - \Ent(C|B) = \I(C; B).$$
This represents the number of nats shared by $B$ and $C$, depicted as the intersection between the two discs.  If the variables are independent then $\I(B; C) = 0$.  Finally, the mutual conditional information between $B$ and $C$, conditioned on a third random variable $D$, is defined by
$$\I(B; C | D) = \I(B; C, D) - I(B; D).$$
This represents information remaining \henrik{between $B$ and $C$ after $D$ is observed}. %

\begin{figure}
\centering
\includegraphics[width=3in]{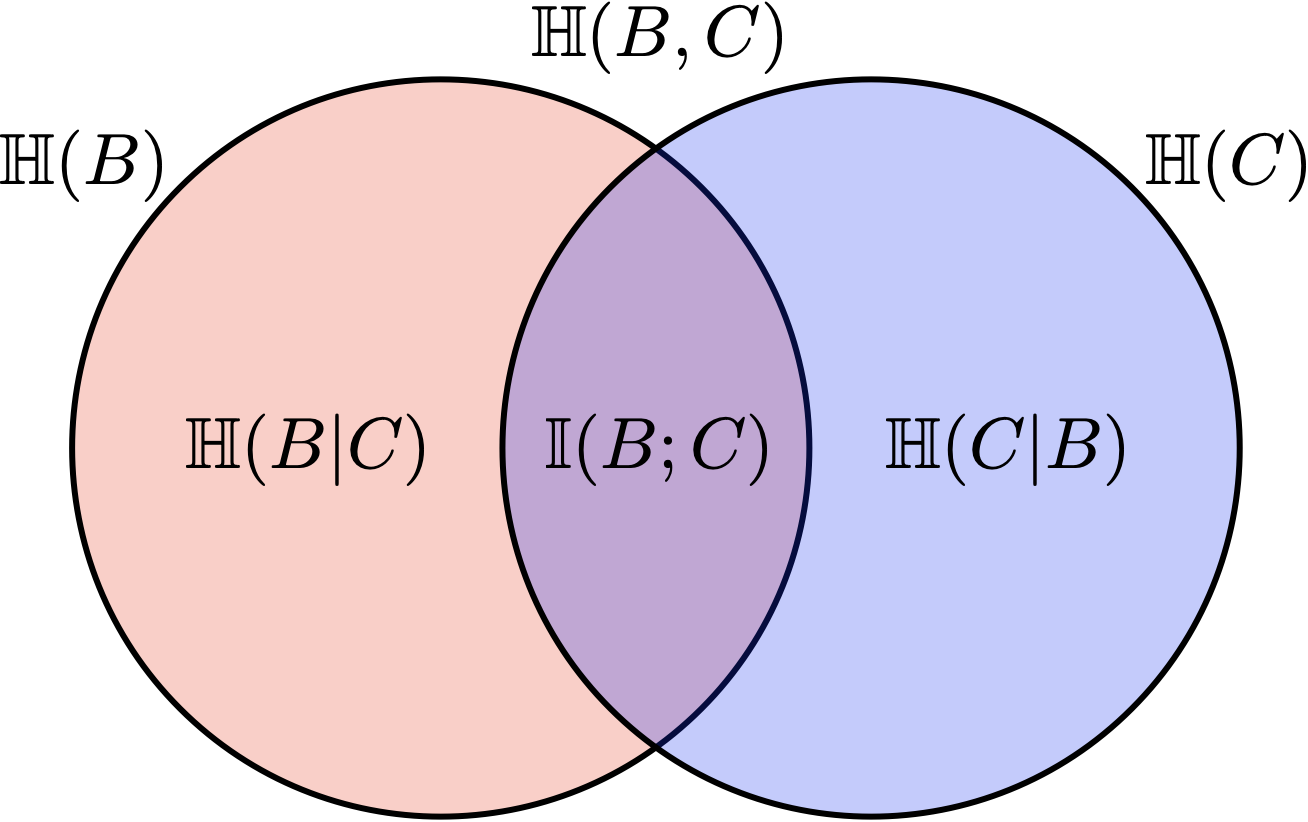}
\caption{Venn diagram \henrik{showing the additive relationships between various information measures. For instance, the mutual information between $B$ and $C$ is defined by
$\I(B; C) = \Ent(B) - \Ent(B|C)$.}}
\label{fig:info_venn}
\end{figure}

The following result establishes that the informational error $\E[\KL(P^*_t \| \tilde{P}_t)]$ equals the information $\I(Y_{t+1};H_t|U_t)$ that the history $H_t$ presents about $Y_{t+1}$ but that is absent from $U_t$.
\begin{restatable}{theorem}{error-information}\label{th:error-information}
For all $t$,
$$\underbrace{\E[\KL(P^*_t \| \tilde{P}_t)]}_{\mathrm{informational\ error}} = \underbrace{\I(Y_{t+1};H_t|U_t)}_{\mathrm{absent\ info}}.$$
\end{restatable}
\begin{proof}
From the definitions of conditional entropy, $P^*_t$, and $\tilde{P}_t$, we have
$\Ent(Y_{t+1}|H_t) = \E[-\ln P^*_t(Y_{t+1})]$ and $\Ent(Y_{t+1}|U_t) = \E[-\ln \tilde{P}_t(Y_{t+1})]$.
It follows that,
\begin{align*}
\E[\KL(P^*_t \| \tilde{P}_t)] 
=& \E[\ln P^*_t(Y_{t+1}) - \ln \tilde{P}_t(Y_{t+1})] \\
=&  \Ent(Y_{t+1} | U_t) - \Ent(Y_{t+1} | H_t) \\
=& \Ent(Y_{t+1} | U_t) - \Ent(Y_{t+1} | U_t, H_t) \\
=& \I(Y_{t+1};H_t|U_t).
\end{align*}
The third equality follows from the fact that $Y_{t+1} \perp U_t | H_t$.
\end{proof}

\subsubsection{Information Capacity Constrains Performance}

It is natural to think that information capacity can constrain performance. This relationship is formalized by the following result.
\begin{restatable}{theorem}{error-capacity}\label{th:error-capacity}
For all $t$,
$$\underbrace{\E[\KL(P^*_t\|P^*_0)]}_{\mathrm{initial\ error}} - \underbrace{\E[\KL(P^*_t\|\tilde{P}_t)]}_{\mathrm{informational\ error}} = \underbrace{\I(Y_{t+1};U_t)}_{\mathrm{useful\ info}} \leq \underbrace{\Ent(U_t)}_{\mathrm{info}}.$$
\end{restatable}
\begin{proof}
For all $t$,
\begin{align*}
\E[\KL(P^*_t\|\tilde{P}_t)]
=& \I(Y_{t+1};H_t|U_t) \\
=& \I(Y_{t+1};H_t,U_t) - \I(Y_{t+1};U_t) \\
=& \I(Y_{t+1};H_t) - \I(Y_{t+1};U_t) \\
=& \E[\KL(P^*_t\|P^*_0)] - \I(Y_{t+1};U_t).
\end{align*}
We arrive at the \henrik{equality} in our result by rearranging terms. The \henrik{inequality} follows from the fact that mutual information between two random variables is bounded by the entropy of each.
\end{proof}
Note that $P^*_0 = \Prob(Y_{t+1} = \cdot | H_0) = \Prob(Y_{t+1} = \cdot)$ is an uninformed prediction, which is based on no data. The left-hand-side expression $\E[\KL(P^*_t\|P^*_0)] - \E[\KL(P^*_t\|\tilde{P}_t)]$ is the reduction in error relative to an uninformed prediction. 

\henrik{
The theorem above can be interpreted in two steps:
\begin{enumerate}
    \item If the agent capacity $\Ent(U_t)$ is increased and that capacity is maximally used then $\I(Y_{t+1}, U_t)$ is increased. The expression $\I(Y_{t+1}; U_t)$ is the degree to which the agent state $U_t$ informs the agent about $Y_{t+1}$.
    \item If $\I(Y_{t+1};U_t)$ increases, then the error of the idealized agent, $\E[\KL(P^*_t\|\tilde{P}_t)]$ goes down.
\end{enumerate}
}

\begin{summary}
\begin{itemize}
\item An {\bf agent state} $U_t$ is a summary of the history $H_t$ maintained to facilitate efficient computation of each action
$$A_t \sim \pi(\cdot| U_t),$$
and subsequent agent state
$$U_{t+1} \sim \psi(\cdot|U_t, A_t, O_{t+1}).$$
The agent policy of an agent designed in this way is characterized by the pair $(\psi,\pi)$.
\item The {\bf information content} of agent state is the number of nats required to encode it.  At large scale and under plausible technical conditions, this is well-approximated by the {\bf entropy} $\Ent(U_t)$.
\item An agent's {\bf information capacity} is a constraint on the information content and is typically limited by computational resources.
\item Information capacity limits agent performance, which we measure in terms of average reward.
\item In the special case of {\bf continual supervised learning} with rewards $R_{t+1} = \ln P_t(Y_{t+1})$, expected reward is determined by prediction error via
$$\underbrace{\E[\ln P_t(Y_{t+1})]}_{\text{reward}} = \underbrace{\E[\ln P^*_t(Y_{t+1})]}_{\text{optimal reward}} - \underbrace{\E[\KL(P^*_t \| P_t)]}_{\text{prediction error}},$$
where $P_t$ denotes the agent's prediction and $P_t^*(\cdot) = \Prob(Y_t = \cdot|H_t)$ is the optimal prediction.
Prediction error decomposes into informational and inferrential errors:
$$\underbrace{\E[\KL(P^*_t \| P_t)]}_{\mathrm{prediction\ error}} = \underbrace{\E[\KL(P^*_t\|\tilde{P}_t)]}_{\mathrm{informational\ error}} + \underbrace{\E[\KL(\tilde{P}_t\|P_t)]}_{\mathrm{inferential\ error}},$$
where $\tilde{P}_t(\cdot) = \Prob(Y_t = \cdot|U_t)$ is the best prediction that can be produced based on the agent state.  
The informational error is equal to the information absent from agent state that would improve the prediction:
$$\underbrace{\E[\KL(P^*_t \| \tilde{P}_t)]}_{\mathrm{informational\ error}} = \underbrace{\I(Y_{t+1};H_t|U_t)}_{\mathrm{absent\ info}}.$$
\henrik{Reduction in error is bounded} by the information capacity according to
$$\underbrace{\E[\KL(P^*_t\|P^*_0)]}_{\mathrm{initial\ error}} - \underbrace{\E[\KL(P^*_t\|\tilde{P}_t)]}_{\mathrm{informational\ error}} = \underbrace{\I(Y_{t+1};U_t)}_{\mathrm{useful\ info}} \leq \underbrace{\Ent(U_t)}_{\mathrm{info}}.$$
\end{itemize}
\end{summary}
\clearpage

\section{Stability Versus Plasticity}\label{se:stability_plasticity}

\henrik{The subjects of {\it catastrophic forgetting} and {\it loss of plasticity} have attracted a great deal of attention in the continual learning literature.} Catastrophic forgetting refers to elimination of useful information. Loss of plasticity refers to an inability to ingest useful new information.  In this section, we build on information-theoretic tools introduced in the previous section to formalize these concepts and clarify the interaction between information capacity, stability, and plasticity. To keep the analysis simple, we restrict attention to the special case of continual supervised learning, as presented in Section~\ref{sec:csl_special_case}.  Recall that each action is  a predictive distribution $A_t = P_t$ and each reward $R_{t+1} = \ln P_t(Y_{t+1})$ is the logarithm of the probability assigned to the realized label.

\henrik{Note that in this section, we study idealized agents that perform perfect predictions given the information retained by the agent. Thus, our treatment abstracts away the optimization challenges specific to continual learning, such as running SGD on a non-iid data stream.}

\subsection{Stability-Plasticity Decomposition}

If the agent had infinite information capacity and could maintain history as its agent state $U_t = H_t$ then the informational error would be $\E[\KL(P^*_t \| \tilde{P}_t)]  = \I(Y_{t+1}; H_t|H_t) = 0$.  However, with an agent state that retains partial information, the error will typically be larger.  In this way, and as expressed by Theorem \ref{th:error-capacity}, a constraint on capacity can induce error.  This happens through requiring that the agent either forget some old information, forgo ingestion of some new information, or both.  

We can formalize the relation between these quantities in information-theoretic terms.  At each timestep $t$, the informational error $\E[\KL(P^*_t \| \tilde{P}_t)]$ arises as a consequence of useful information forgotten or forgone due to implasticity over previous timesteps.  As shorthand, let $H_{t-k:t} = (O_{t-k}, \ldots, O_t)$.
The error due to information forgotten $k$ timesteps earlier can be expressed as $I(Y_{t+1};U_{t-k-1}|U_{t-k}, H_{t-k:t})$: the information about $Y_{t+1}$ that is available in $U_{t-k-1}$ but not the next agent state $U_{t-k}$ or the subsequent experience $H_{t-k+1:t}$.  In other words, this is the information lost in transitioning from $U_{t-k-1}$ to $U_{t-k}$ and not recoverable by timestep $t$.  The error due to implasticity $k$ timesteps earlier can be expressed as $\I(Y_{t+1};O_{t-k}|U_{t-k}, H_{t-k+1:t})$: the information about $Y_{t+1}$ that is available in $O_{t-k}$ but absent from the preceding agent state $U_{t-k}$ and the subsequent experience $H_{t-k+1:t}$.  In other words, this is information presented by $O_{t-k}$ but not $U_{t-k}$ and that is not otherwise available through timestep $t$.

The following theorem provides a formal decomposition, attributing error to forgetting and implasticity.  Note that actions are omitted because in the case of supervised learning they do not impact observations.

\begin{restatable}{theorem}{capacityErrorDecomp}\label{th:cap_error_decomp}
For all $t\in \Z_{+}$,
$$\underbrace{\E[\KL(P^*_t \| \tilde{P}_t)]}_{\mathrm{error}} = \sum_{k=0}^t \left(\underbrace{\I(Y_{t+1};U_{t-k-1}|U_{t-k}, H_{t-k:t})}_{\mathrm{forgetting\ at\ lag\ } k} + \underbrace{\I(Y_{t+1};O_{t-k}|U_{t-k}, H_{t-k+1:t})}_{\mathrm{implasticity\ at\ lag\ } k}\right).$$
\end{restatable}
\begin{proof}
We have
\begin{align*}
\E[\KL(P^*_t \| \tilde{P}_t)]
&\overset{(a)}{=} \I(Y_{t+1}; H_t | U_t) \\
&\overset{(b)}{=} \I(Y_{t+1}; H_t, U_{0:t-1} | U_t) \\
&\overset{(c)}{=} \sum_{k=0}^t \I(Y_{t+1}; O_{t-k}, U_{t-k-1} | U_{t-k:t}, H_{t-k+1:t}) \\
&\overset{(d)}{=} \sum_{k=0}^t \I(Y_{t+1}; O_{t-k}, U_{t-k-1} | U_{t-k}, H_{t-k+1:t}) \\
&\overset{(e)}{=} \sum_{k=1}^t \left(\I(Y_{t+1}; U_{t-k-1} | U_{t-k}, H_{t-k:t}) + \I(Y_{t+1}; O_{t-k} | U_{t-k}, H_{t-k+1:t}) \right),
\end{align*}
where $(a)$ follows from Theorem \ref{th:error-information}; $(b)$ follows from $U_{0:t-1} \perp Y_{t+1} | U_t, H_t$; $(c)$ follows from the chain rule of mutual information \henrik{(reproduced in Lemma~\ref{lem:chain-rule-mutual} in Appendix \ref{app:information_theory})}, and we take $Y_0$ and $U_{-1}$ to be the singleton $\emptyset$; \henrik{$(d)$ follows from the fact that for all $j$, $U_{j+1}$ is a function of $(U_{j}, O_{j+1})$;} and $(e)$ follows from chain rule of mutual information \henrik{(Lemma~\ref{lem:chain-rule-mutual})}.
\end{proof}

Implasticity and forgetting terms are indexed by a lag $k$ relative to the current timestep $t$.  In the process of updating the agent state from $U_{t-k-1}$ to $U_{t-k}$ in response to an observation, information may be ingested and/or forgotten. The implasticity and forgetting terms measure how this ultimately impacts the agent's ability to predict $Y_{t+1}$.  Summing over lags $k$ produces the immediate error $\I(Y_{t+1};H_t|U_t)$.

The implasticity at lag $k$ measures the amount of information presented by $O_{t-k}$ that is useful for predicting $Y_{t+1}$ that the agent fails to ingest and is absent from the intermediate data $H_{t-k+1:t}$.  Note that, due to the conditioning on the data $H_{t-k+1:t}$, this term only penalizes the inability to extract information about $Y_{t+1}$ \henrik{from $O_{t-k}$} that will not be again available from observations between \henrik{timesteps $t-k+1$ and $t$}.

The forgetting at lag $k$ derives from the agent ejecting information that is relevant to predicting $Y_{t+1}$ when updating its agent state from $U_{t-k-1}$ to $U_{t-k}$. Again, due to conditioning on $H_{t-k:t}$, this term only penalizes for forgotten information that cannot be recovered from other observations to be made before timestep $t$.

Our decomposition indicates that errors due to implasticity and forgetting are {\it forward looking} in the sense that they only impact predictions at subsequent timesteps; thus the lag $k$ relative to prediction error. 
 This is in contrast with much of the continual SL literature, which aims to develop agents that remember information that would have been useful in their past, even if that information is unlikely to be useful to their future.  Indeed, the term {\it catastrophic forgetting} typically refers to loss of useful information, whether useful in the future or past.

Our expressions for implasticity and forgetting are complicated by sums over lags. The expressions simplify greatly when each input $X_{t+1}$ is independent of history, i.e., $X_{t+1}\perp H_t$, and when the sequence of agent states and observations forms a stationary stochastic process\footnote{\henrik{Note that we overload the word stationary here to refer to a stochastic process whose law doesn't change under time shifts. In other places in the monograph, the term stationary will be used loosely to refer to cases when there is no fixed learning target upon which the agent can converge. See Section \ref{se:continual-vs-convergent} for an elaboration.}}:
\henrik{
\begin{definition}[\emph{stationary} stochastic process]\label{def:stationary_stochastic_process}
    A \emph{stationary} stochastic process is a stochastic process $\{X_t\}_{t=0}^\infty$ s.t. for all $\tau \geq 0$ and $t \in \Z_{+}$,
    $$\Prob(X_{0}, \ldots, X_{t})\ =\ \Prob(X_{\tau}, \ldots, X_{t+\tau}).$$
\end{definition}
}
\henrik{
    \begin{figure}[h]
        \centering
        \includegraphics[width=0.45\linewidth]{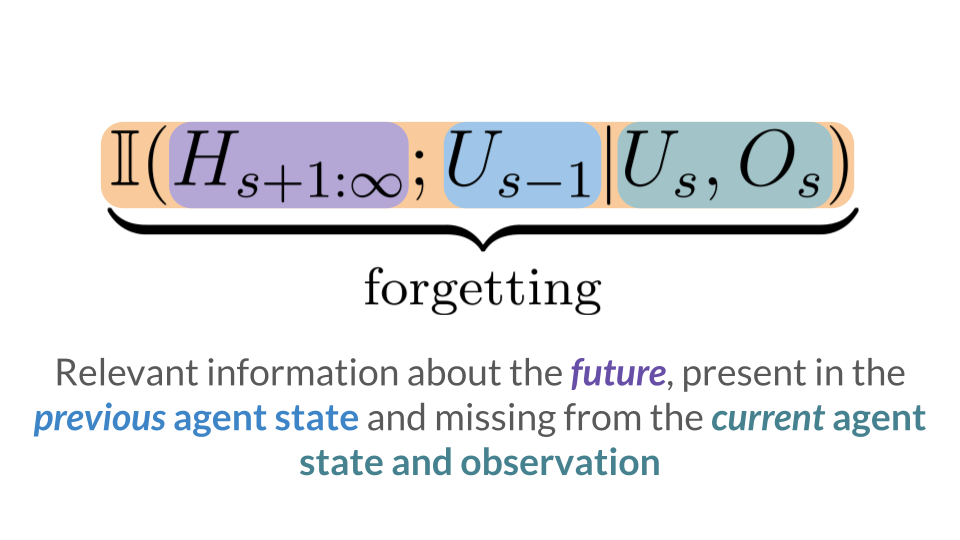}
        \includegraphics[width=0.45\linewidth]{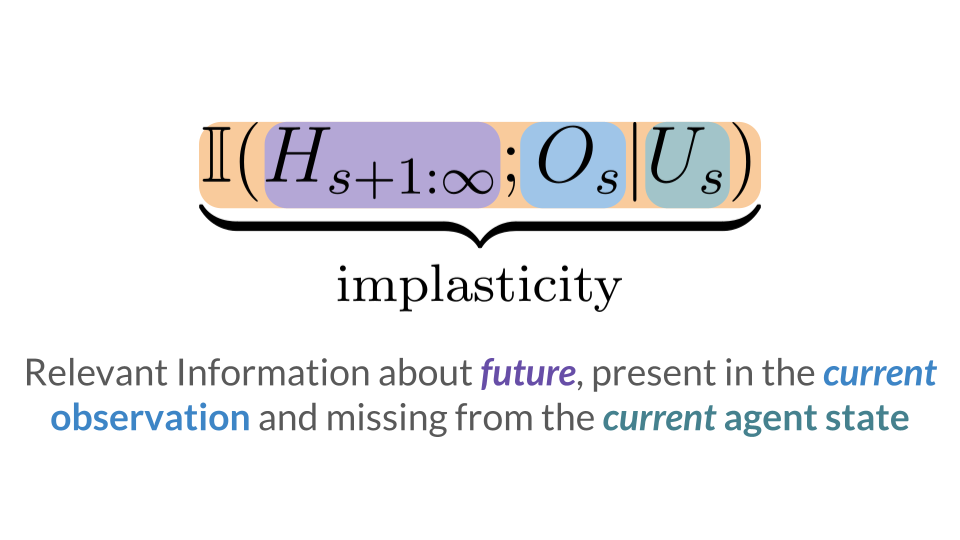}
        \caption{The above illustrates the intuitive definitions of forgetting and implasticity afforded by our framework.}
        \label{fig:forget_plasticity}
    \end{figure}
}

\henrik{This corresponds to a nonstationary supervised learning problem in which the inputs are independently and identically distributed but the function mapping inputs to outputs is changing over time according to a stationary stochastic process (see Section~\ref{sec:lms-ar1} for an example).}

As a reminder, $\mathbb{Z}_+$ denotes the set of non-negative integers and $\mathbb{Z}_{++}$ denotes the set of positive integers.

\henrik{
\begin{theorem}\label{thm:errors-simplification}
Suppose $((U_t, O_t): t \in \mathbb{Z}_+)$ is a stationary process and, for all $t \in \mathbb{Z}_{++}$, $X_{t+1} \perp H_t$.  Then, for all $s \in \mathbb{Z}_{++}$
$$\lim_{t\to\infty}\underbrace{\E[\KL(P^*_t \| \tilde{P}_t)]}_{\mathrm{error}} = \underbrace{\I(H_{s+1:\infty}; U_{s-1}|U_s, O_s)}_{\mathrm{forgetting}} + \underbrace{\I(H_{s+1:\infty};O_s|U_s)}_{\mathrm{implasticity}}.$$
\end{theorem}

\begin{proof}
We establish the result by taking limits of the expressions from Theorem \ref{th:cap_error_decomp} for cumulative forgetting and implasticity.
The forgetting term yields
    \begin{align*}
        \lim_{t\to\infty}\sum_{k=0}^t\I(Y_{t+1};U_{t-k-1}|U_{t-k}, H_{t-k:t})
        & \overset{(a)}{=} \lim_{t\to\infty}\sum_{k=0}^t\I(Y_{t+k+1}; U_{t-1}|U_{t}, H_{t:t+k})\\
        & \overset{(b)}{=} \lim_{t\to\infty}\sum_{k=0}^t\I(X_{t+k+1}, Y_{t+k+1}; U_{t-1}|U_{t}, H_{t:t+k})\\
        & \overset{(c)}{=} \lim_{t\to\infty}
        \I(H_{t+1:2t+1}; U_{t-1}|U_{t}, O_t)\\
        & \overset{(d)}{=} \lim_{t\to\infty}
        \I(H_{s+1:s+t+1}; U_{s-1}|U_{s}, O_s)\\
        & =
        \I(H_{s+1:\infty}; U_{s-1}|U_{s}, O_s),
    \end{align*}
    where $(a)$ follows from the fact that since $(U_t, O_t)$ is stationary, $\I(Y_{t+1};U_{t-k-1}|U_{t-k}, H_{t-k:t})$ is a constant independent of $t$; $(b)$ follows Lemma~\ref{lem:mutual-information-unchanged-1} in Appendix~\ref{app:information_theory}, taking $A=Y_{t+k+1}$, $B=U_{t-1}$, $C=(U_t, H_{t:t+k})$, $D=X_{t+k+1}$, and noting that $X_{t+k+1}\perp U_{t-1} | (U_t, H_{t:t+k}, Y_{t+k+1})$; $(c)$ follows from the chain rule of mutual information (Lemma~\ref{lem:chain-rule-mutual}); and $(d)$ follows from stationarity.
    
The implasticity term yields
\begin{align*}
        \lim_{t\to\infty}\sum_{k=0}^t \I(Y_{t+1};O_{t-k}|U_{t-k}, H_{t-k+1:t})
        & \overset{(a)}{=} \lim_{t\to\infty}\sum_{k=0}^t \I(Y_{t+k+1};O_{t}|U_{t}, H_{t+1:t+k})\\
        & \overset{(b)}{=} \lim_{t\to\infty}\sum_{k=0}^t \I(X_{t+k+1}, Y_{t+k+1};O_{t}|U_{t}, H_{t+1:t+k})\\
        & \overset{(c)}{=}\lim_{t\to\infty}\I(H_{t+1:2t+1};O_{t}|U_{t})\\
        & \overset{(d)}{=}\lim_{t\to\infty}\I(H_{s+1:s+t+1};O_{s}|U_{s})\\
        & = \I(H_{s+1:\infty};O_{s}|U_{s}),
    \end{align*}
    where $(a)$ follows from the fact that since $(U_t, O_t)$ is stationary, $\I(Y_{t+1};O_{t-k}|U_{t-k}, H_{t-k+1:t})$ is a constant independent of $t$; $(b)$ follows Lemma~\ref{lem:mutual-information-unchanged-1} in Appendix~\ref{app:information_theory}, taking $A=Y_{t+k+1}$, $B=O_t$, $C=(U_t, H_{t+1:t+k})$, $D=X_{t+k+1}$, and noting that $X_{t+k+1}\perp O_t | (U_t, H_{t+1:t+k}, Y_{t+k+1})$; $(c)$ follows from the chain rule of mutual information (Lemma~\ref{lem:chain-rule-mutual}); and $(d)$ follows from stationarity.
\end{proof}
}
The first term on the right-hand-side equates error due to forgetting with the information available in the previous agent state $U_{t-1}$, but neither in the subsequent agent state $U_t$ or the subsequent observation $O_t$.
The second term characterizes error due to implasticity as the information about the future $H_{t+1:\infty}$ available in the current observation $O_t$ but not ingested into the current agent state $U_t$.

\subsection{A Didactic Example}
\label{se:lms}

We will illustrate concretely through a simple example how agent and environment dynamics influence implasticity and forgetting errors.  
We focus here on insight that can be drawn from analysis of the example, deferring details of the analysis to Appendix \ref{apdx:lms}.
\henrik{
We do \emph{not} argue for the practical adoption of the particular algorithms proposed in this section.  Rather, we use them as illustrative examples to show how standard algorithms might need to be adapted when considering capacity constraints.
}
\subsubsection{LMS with an AR(1) Process}
\label{sec:lms-ar1}

We consider a very simple instance of continual SL in which the input set is a singleton $\mathcal{X} = \{\emptyset\}$ and the label set $\mathcal{Y} = \mathbb{R}$ is of real numbers.  Since inputs $X_0,X_1,\ldots$ are uninformative, we take the history to only include labels $H_t = (Y_1, \ldots, Y_t)$.  Each label is generated according to
$$Y_{t+1} = \theta_t + W_{t+1},$$
where $\theta_t$ is a latent variable that represents the state of the process and $W_{t+1}$ is a sample of an iid $\mathcal{N}(0, \sigma^2)$ sequence.  The sequence $\theta_t$ is initialized with $\theta_0 \sim \mathcal{N}(0,1)$ and evolves according to
$$\theta_{t+1} = \eta \theta_t + V_{t+1},$$
where $\eta$ is a fixed parameter and $V_{t+1}$ is a sample of an iid $\mathcal{N}(0,1-\eta^2)$ sequence.  Note that, for all $t$, the marginal distribution of $\theta_t$ is standard normal.

Consider an agent that maintains a real-valued agent state, initialized with $U_0 \sim \mathcal{N}(0,1)$ and updated according to the least mean squares (LMS) algorithm \citep{Widrow1960LMS}
$$U_{t+1} = U_t + \alpha (Y_{t+1} - U_t),$$
with a fixed stepsize $\alpha \in (0,1)$.  This agent can then generate predictions
$P_t(\cdot) = \Prob(Y_{t+1} \in \cdot|U_t)$, which are Gaussian distributions.
This agent does not necessarily make optimal use of history, meaning that $\KL(P^*_t\|P_t) > 0$.  However, for a suitably chosen stepsize, the agent attains zero error in steady state: $\lim_{t \rightarrow \infty} \KL(P^*_t\|P_t) = 0$.  With this optimal stepsize the LMS algorithm becomes a steady-state Kalman filter \citep{Kalman1960Filter}.

\subsubsection{Constraining Information Content}

Note that, for our agent, the forgetting error $\I(Y_{t+1}; U_{t-k-1} | U_{t-k}, H_{t-k:t})$ is always zero because $U_{t-k-1}$ is determined by the next agent state $U_{t-k}$ and the label $Y_{t-k}$, which makes up part of $H_{t-k:t}$, according to
$U_{t-k-1} = (U_{t-k} - \alpha Y_{t-k})/(1-\alpha)$.
This degeneracy stems from the agent's infinite capacity -- since the agent state is a real number, it can encode an infinite amount of information.  As discussed earlier, practical designs of agents subject to computational constraints limit information capacity.  As a microcosm that may yield insight relevant to such contexts, we will consider a variant of LMS with bounded information capacity.

More relevant qualitative behavior emerges when the agent restricts the information content $\Ent(U_t)$ of the agent state to operate with limited information capacity.  Such a constraint arises, for example, if instead of retaining a continuous-valued variable the agent must quantize, encoding an approximation to the agent state with a finite alphabet.  The state of such an agent might evolve according to
\begin{equation}
\label{eq:lms-agent-update}
U_{t+1} = U_t + \alpha (Y_{t+1} - U_t) + Q_{t+1},
\end{equation}
where $Q_{t+1}$ represents quantization error. We can think of this error as quantization noise, which perturbs results of agent state updates.  In particular, it is the difference between the real value $U_t + \alpha (Y_{t+1} - U_t)$ that the agent would store as $U_{t+1}$, if it could, and the quantized value that it actually stores.

Because the effects of quantization noise can be difficult to quantify, for the purpose of analysis, it is common in information theory to approximate quantization noise as an independent Gaussian random variable.  \henrik{Often this does not impact qualitative insights (see e.g., Theorem 10.3.2 in \citet{cover2012elements}).}  For analytical tractability, we will approximate the quantization noise $(Q_t: t \in \mathbb{Z}_{++})$ as an iid $\mathcal{N}(0, \delta^2)$ sequence, for some $\delta > 0$, which we will refer to as the {\it quantization noise intensity}.  

Like actual quantization, this Gaussian noise moderates information content of the agent state.  However, while actual quantization keeps the information content $\Ent(U_t)$ bounded, with Gaussian noise, $\Ent(U_t)$ becomes infinite.  This is because $U_t$ encodes infinite irrelevant information expressed by the Gaussian noise itself.  When Gaussian noise is used to approximate quantization, rather than $\Ent(U_t)$, it is more appropriate to take the information content to be $\I(U_t; H_t)$.  This represents the number of nats of information from the history $H_t$ retained by the agent state $U_t$.  In particular, $\I(U_t; H_t)$ excludes irrelevant information expressed by the Gaussian noise $Q_t$.  

As one would expect, for our setting of LMS with an AR(1) process, $\I(U_t; H_t)$ is infinite when $\delta = 0$, then monotonically decreases, vanishing as the quantization noise intensity $\delta$ increases.  We impose a constraint $\I(U_t; H_t) \leq C$ to express a fixed information capacity $C$.  Given a choice of stepsize $\alpha$, we take the quantization noise intensity to be the value $\delta_*$ so that this constraint is binding: $\I(U_t; H_t) = C$.  This models a quantization scheme that maximizes information content subject to the information capacity.  As established by Theorem \ref{thm:lms-capacity-noise} in Appendix \ref{apdx:lms-reparam}, 
\begin{equation}
\label{eq:quantization-noise-adaptation}
\delta^2_*(\alpha)= \alpha^2  \frac{\sigma^2 (1 - \eta+\eta\alpha) + 1 + \eta-\eta\alpha}{1 - \eta+\eta\alpha} \frac{\exp(-2C)}{1-\exp(-2C)}.
\end{equation}

\subsubsection{Analysis}

Figure~\ref{fig:lms-errors} plots errors due to implasticity $\I(Y_{t+1:\infty};Y_{t}|U_{t})$ and forgetting $\I(Y_{t+1:\infty}; U_{t-1}|U_{t}, Y_{t})$ versus stepsize, for asymptotically large $t$ and observation noise standard deviation $\sigma=0.5$, an information capacity $C=2$, and autoregressive model coefficients $\eta=0.9, 0.95, 0.99$.
\begin{figure}
\centering
\begin{tabular}{ccc}
\begin{subfigure}{0.32\linewidth}
\includegraphics[width=\linewidth]{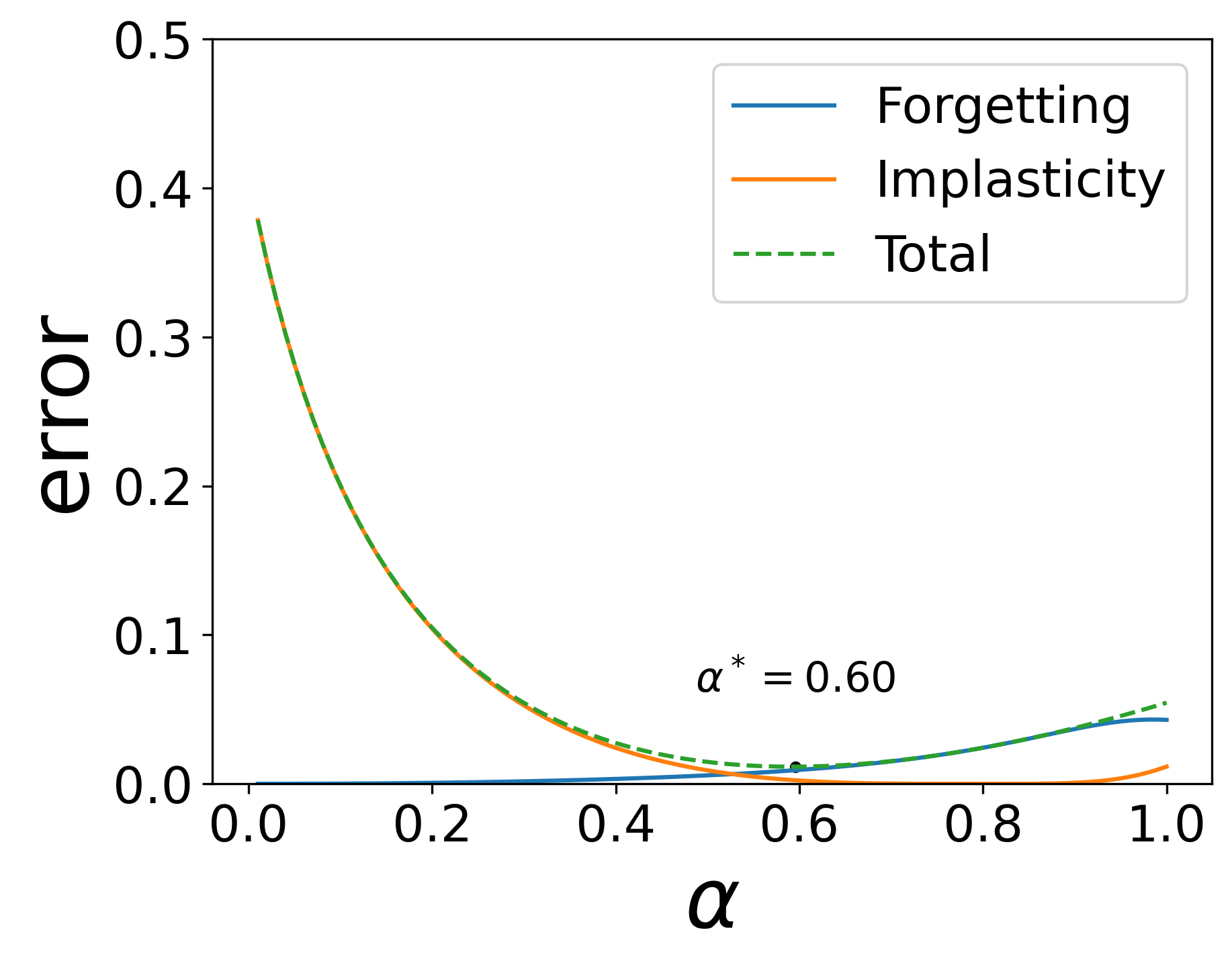}
\caption{$\eta=0.9$}
\end{subfigure} 
& \begin{subfigure}{0.32\linewidth}
\includegraphics[width=\linewidth]{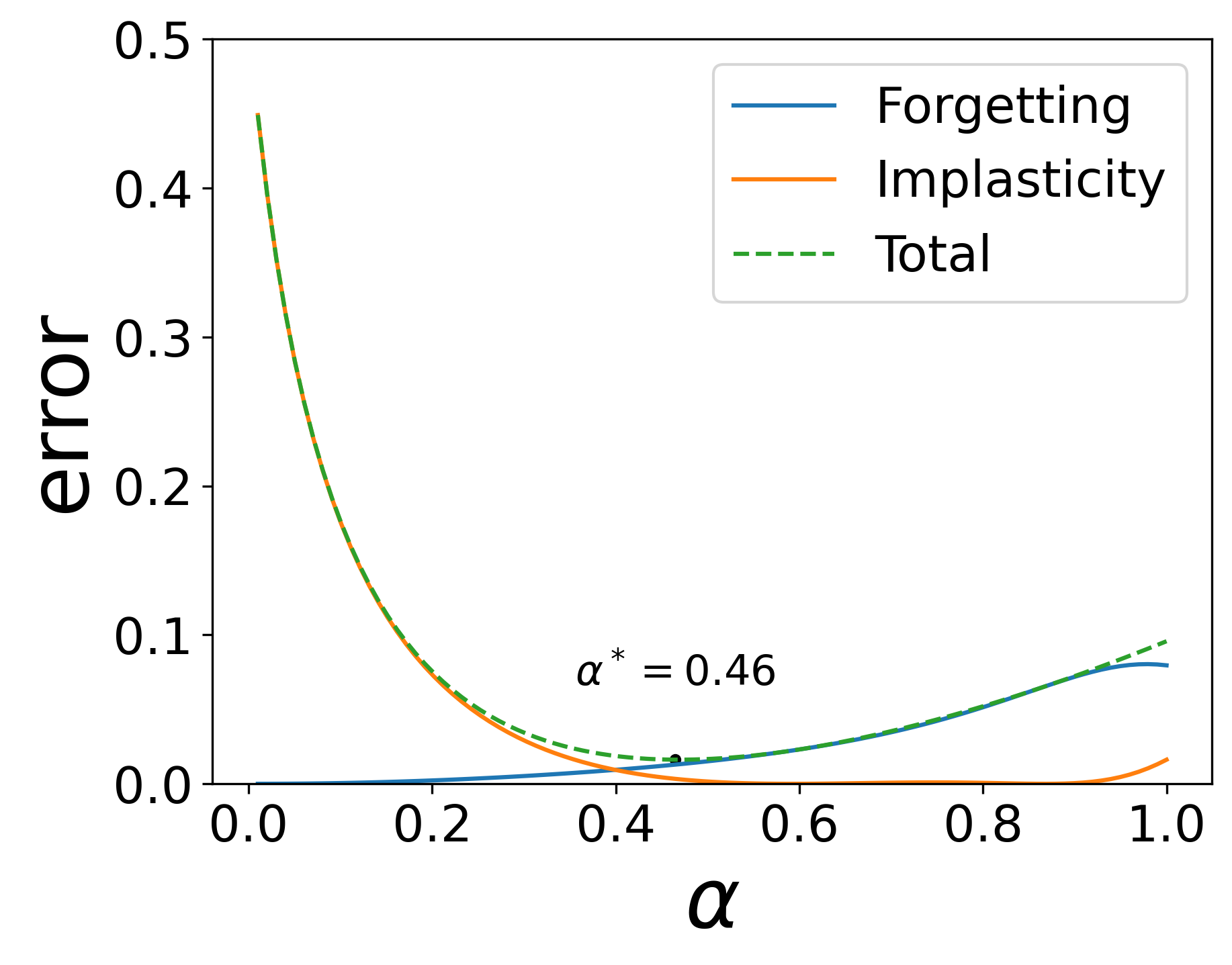}
\caption{$\eta=0.95$}
\end{subfigure}
& \begin{subfigure}{0.32\linewidth}
\includegraphics[width=\linewidth]{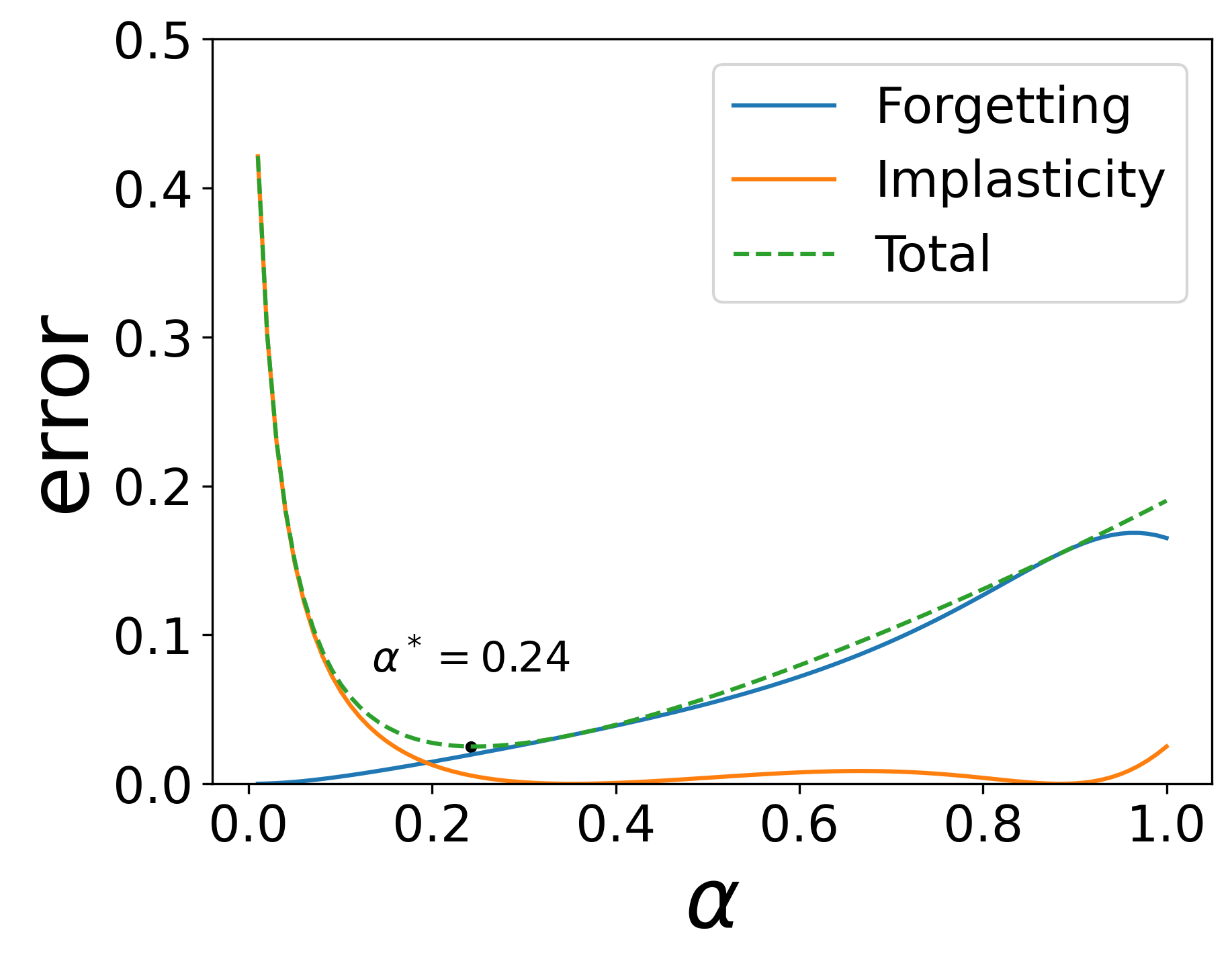}
\caption{$\eta=0.99$}
\end{subfigure}
\end{tabular}
\caption{Average forgetting and implasticity errors versus the stepsize $\alpha$ for environments with observation noise standard deviation $\sigma=0.5$, agent capacity $\I(U_t;H_t)=2$, and autoregressive model coefficients $\eta=0.9, 0.95, 0.99$.  At first, as $\alpha$ increases, the forgetting error increases and the implasticity error decreases. Then, when $\alpha$ is close to $1$, the forgetting error begins to decrease and the implasticity error increases, albeit to much lesser extents.  As $\eta$ increases, the optimal stepsize $\alpha^*$ decreases.
}
\label{fig:lms-errors}
\end{figure}
As $\alpha$ increases, the forgetting increases and the implasticity decreases. Then, as $\alpha$ approaches $1$, the former decreases and the latter increases, albeit to much lesser extents.  It is natural for a small stepsize to reduce forgetting since with a small stepsize there is less weight placed on recent versus previously observed data. For the same reason, as the figure indicates, the optimal stepsize $\alpha^*$ decreases as $\eta$ increases.  \henrik{This is because a larger coefficient $\eta$ makes the $\theta_t$ change slower, and this warrants placing more weight on less recent observations to increase the duration over which they influence predictions.}

Figure~\ref{fig:lms-eta} plots the \henrik{optimal stepsize that minimizes limiting error} as a function of $\eta$ (Figure~\ref{fig:lms-eta-alpha}), the information capacity (Figure~\ref{fig:lms-eta-capacity}), and the quantization noise intensity $\delta$ (Figure~\ref{fig:lms-eta-quantization}). Note that we use $\alpha^*$ to represent the optimal stepsize for fixed $C$ and $\tilde{\alpha}$ for fixed $\delta$.
As suggested by Figure \ref{fig:lms-errors}, as $\eta$ increases, the optimal stepsize decreases.  
\henrik{
This makes sense since $\theta_t$ changes slower as $\eta$ increases, while the duration over which LMS averages observations scales with $1/\alpha$. Hence, as $\eta$ increases, the optimal stepsize $\alpha^*$ decreases to induce averaging over longer durations.} Interestingly, as demonstrated by the second plot, the optimal stepsize does not vary with the information capacity $C$.  This observation is generalized and formalized by Theorem \ref{thm:lms-learning-rate-independent}.  It is interesting to note, however, that this does not imply invariance of the optimal stepsize to the quantization noise intensity $\delta$.  
In particular, when information content is unconstrained, the optimal stepsize varies with $\delta$, as demonstrated in Figure~\ref{fig:lms-eta-quantization}.
\henrik{
\begin{restatable}{theorem}{lmsLearningRateIndep}\label{thm:lms-learning-rate-independent}
    For all AR(1) processes parameterized by $\eta$ and $\sigma$, the optimal learning rate $\alpha^*$ is independent of the information capacity $\I(U_t;H_t)$. In particular, it is equal to the optimal learning rate for the infinite capacity agent without any quantization noise.
\end{restatable}
\begin{proof}
    Proof can be found in Appendix \ref{apdx:opt_learning_rate}.
\end{proof}
}
\begin{figure}
\centering
\begin{tabular}{ccc}
\begin{subfigure}{0.3\linewidth}
\includegraphics[width=\linewidth]{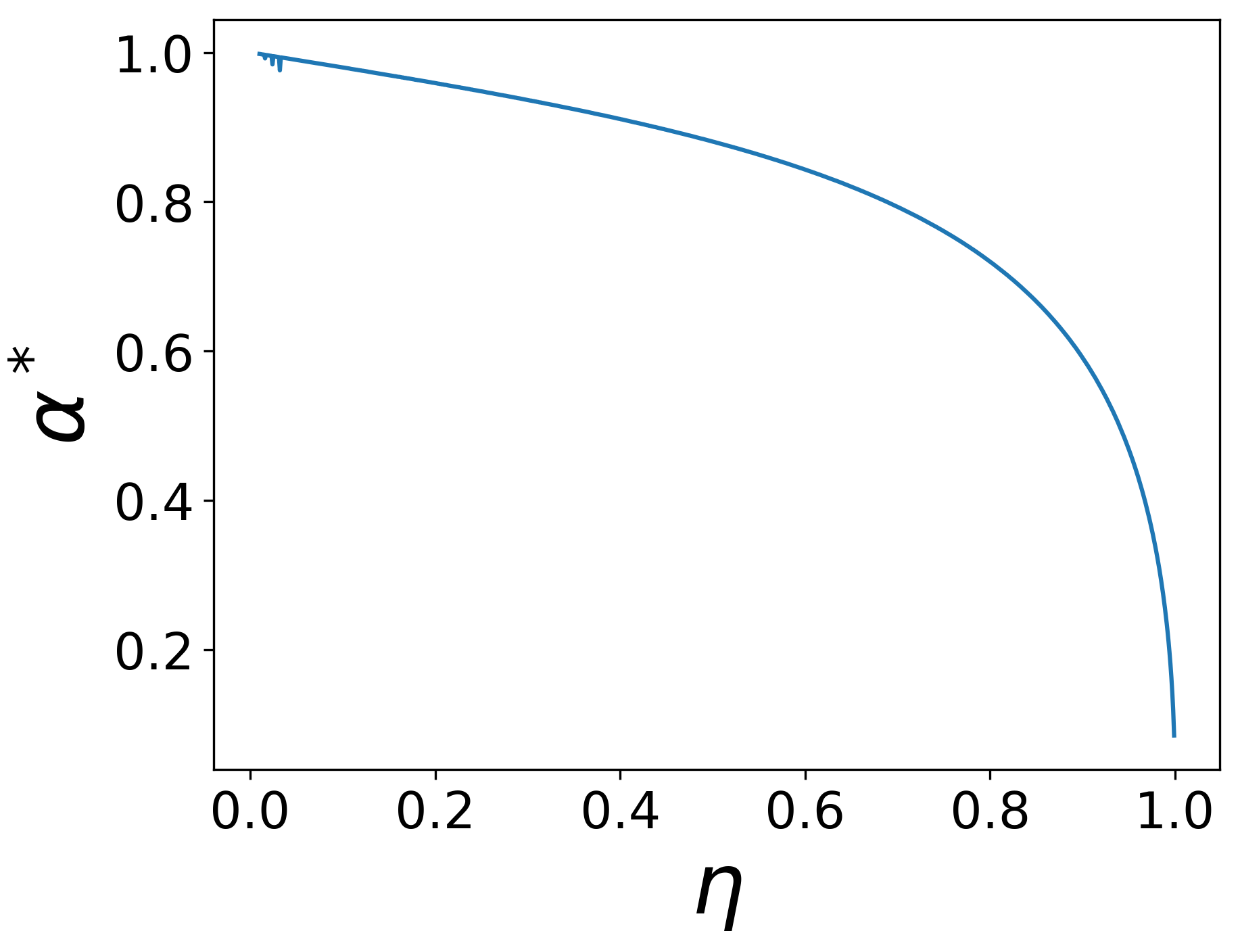}
\caption{$\alpha^*$ decreases $\eta$ increases.}
\label{fig:lms-eta-alpha}
\end{subfigure} 
& \begin{subfigure}{0.32\linewidth}
\includegraphics[width=\linewidth]{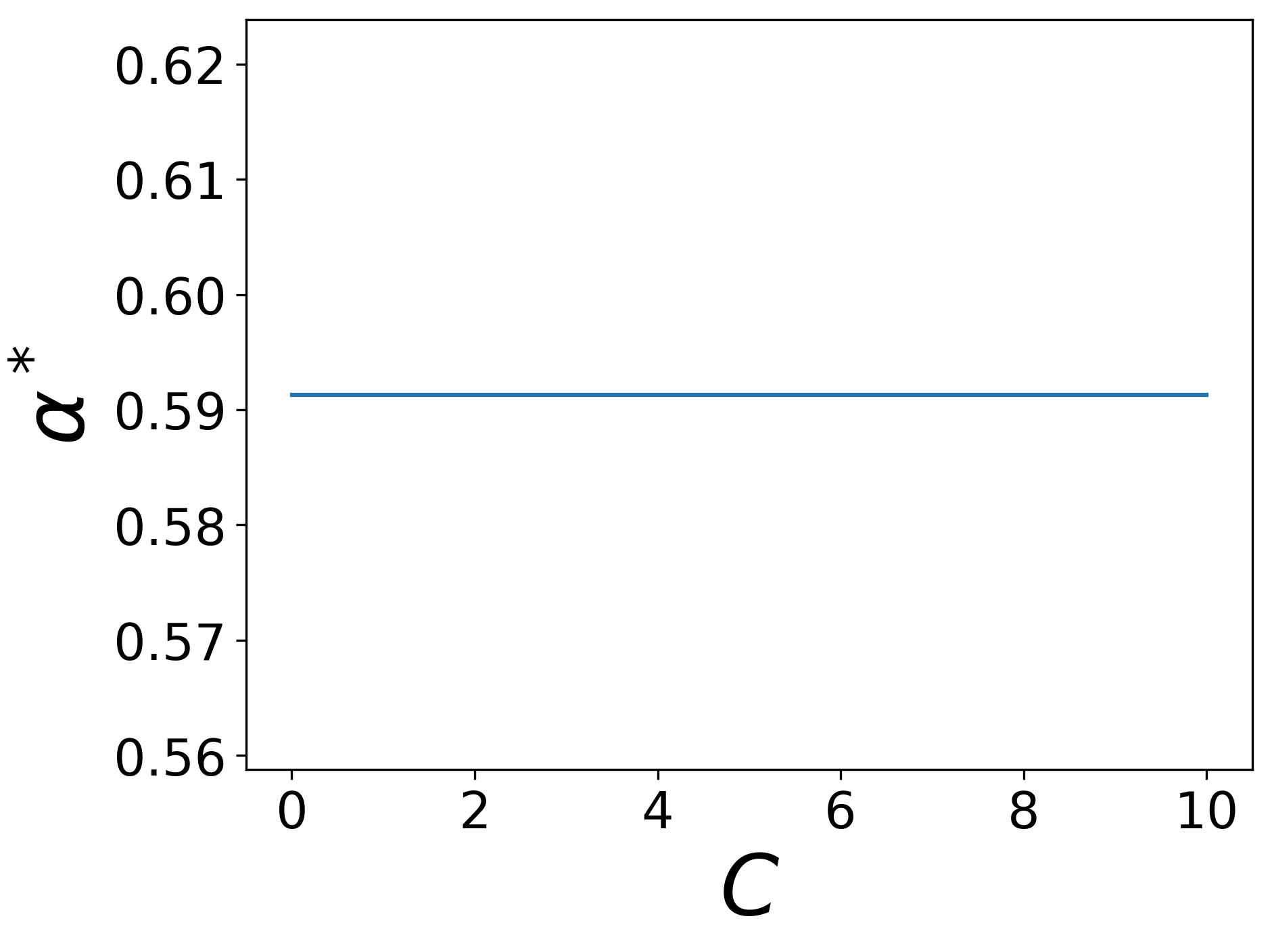}
\caption{$\alpha^*$ does not depend on $C$.}
\label{fig:lms-eta-capacity}
\end{subfigure}
& \begin{subfigure}{0.32\linewidth}
\includegraphics[width=\linewidth]{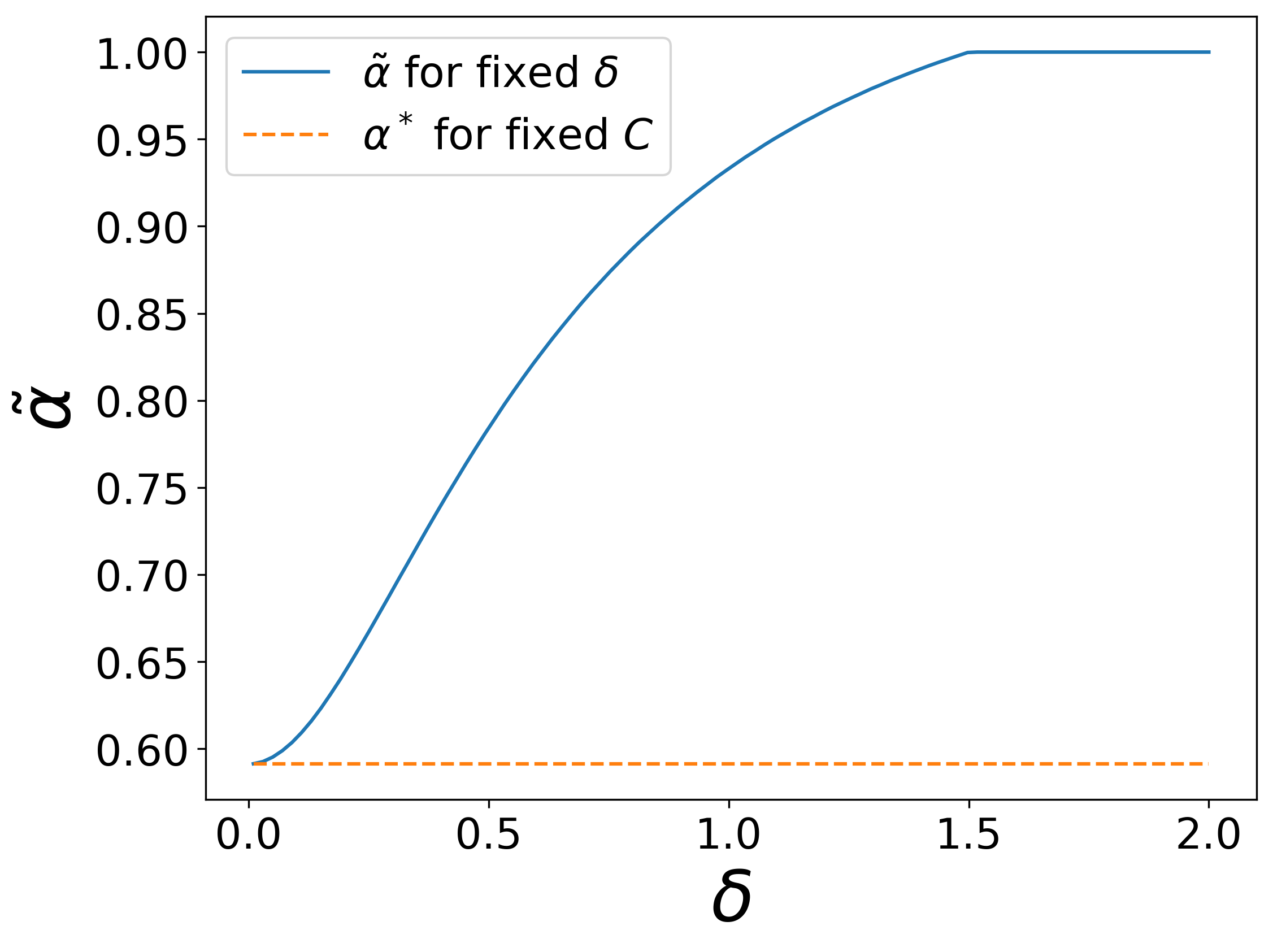}
\caption{$\tilde{\alpha}$ increases as $\delta$ increases.}
\label{fig:lms-eta-quantization}
\end{subfigure}
\end{tabular}
\caption{\henrik{As $\eta$ increases, $\theta_t$ changes slower.} The optimal stepsize $\alpha^*$ decreases in response to induce averaging over a longer duration.  The information capacity $C$ does not impact the optimal stepsize $\alpha^*$.  However, the optimal stepsize $\tilde{\alpha}$ increases with the quantization noise standard deviation $\delta$.  These plots are generated with observation noise $\sigma=0.5$.  The center and right plots are generated with $\eta=0.9$.  For the center plot, $\delta$ is chosen to meet the capacity constraint.}
\label{fig:lms-eta}
\end{figure}

\subsubsection{Stepsize Adaptation}

Given a capacity constraint, calculating $\alpha^*$ requires knowledge of $\eta$.  A stepsize adaptation scheme can alleviate this need, instead incrementally adjusting the stepsize to produce a sequence, aiming to converge on $\alpha^*$.  As an example, we consider a variation of IDBD \citep{Sutton1992IDBD}.  In particular, suppose we update the agent state according to
$$U_{t+1} = U_t + \alpha_{t+1} (Y_{t+1} - U_t) + Q_{t+1},$$
where $\alpha_{t+1} = e^{\beta_{t+1}}$ and $(\beta_t: t\in \mathbb{Z}_{++})$ is a scalar sequences generated according to
$$\beta_{t+1} = \beta_t + \zeta (Y_{t+1} - U_t) h_t - \frac{1}{2} \zeta \alpha_t \frac{d}{d \alpha} \delta_*^2(\alpha_t),$$
$$h_{t+1} = \alpha_{t+1} (Y_{t+1} - U_t) + (1 - \alpha_{t+1})_+ h_t.$$
The standard version of IDBD \citep{Sutton1992IDBD} does not include the term that depends on $\delta_*$.  This term serves to adjust the quantization noise variance in response to changes in the stepsize $\alpha_t$.  For any particular stepsize $\alpha$, $\delta_*(\alpha)$ is the intensity at which the information capacity constraint $\I(U_t; H_t) \leq C$ becomes binding.  Hence, ours is a capacity-constrained version of IDBD.  The functional form of our extra term is derived in Appendix \ref{se:capacity-constrained-IDBD-derivation}.

Figure \ref{fig:IDBD-stepsize-convergence} demonstrates that capacity-constrained IDBD converges on the optimal stepsize for any given information capacity constraint.  This limit of convergence identifies not only a stepsize $\alpha^*$ but also a quantization noise intensity $\delta_*(\alpha^*)$.  The standard version of IDBD, even if executed with quantization noise intensity fixed at $\delta = \delta_*(\alpha^*)$, converges on a different stepsize, which is optimal for that intensity but not optimal subject to the capacity constraint.  Consequently, as Figure \ref{fig:IDBD-error-convergence} indicates, the error attained  by capacity-constrained IDBD approaches the optimal error under the agent capacity but does not with the standard version of IDBD.

\begin{figure}
\centering
\begin{tabular}{cc}
\begin{subfigure}{0.4\linewidth}
\includegraphics[width=\linewidth]{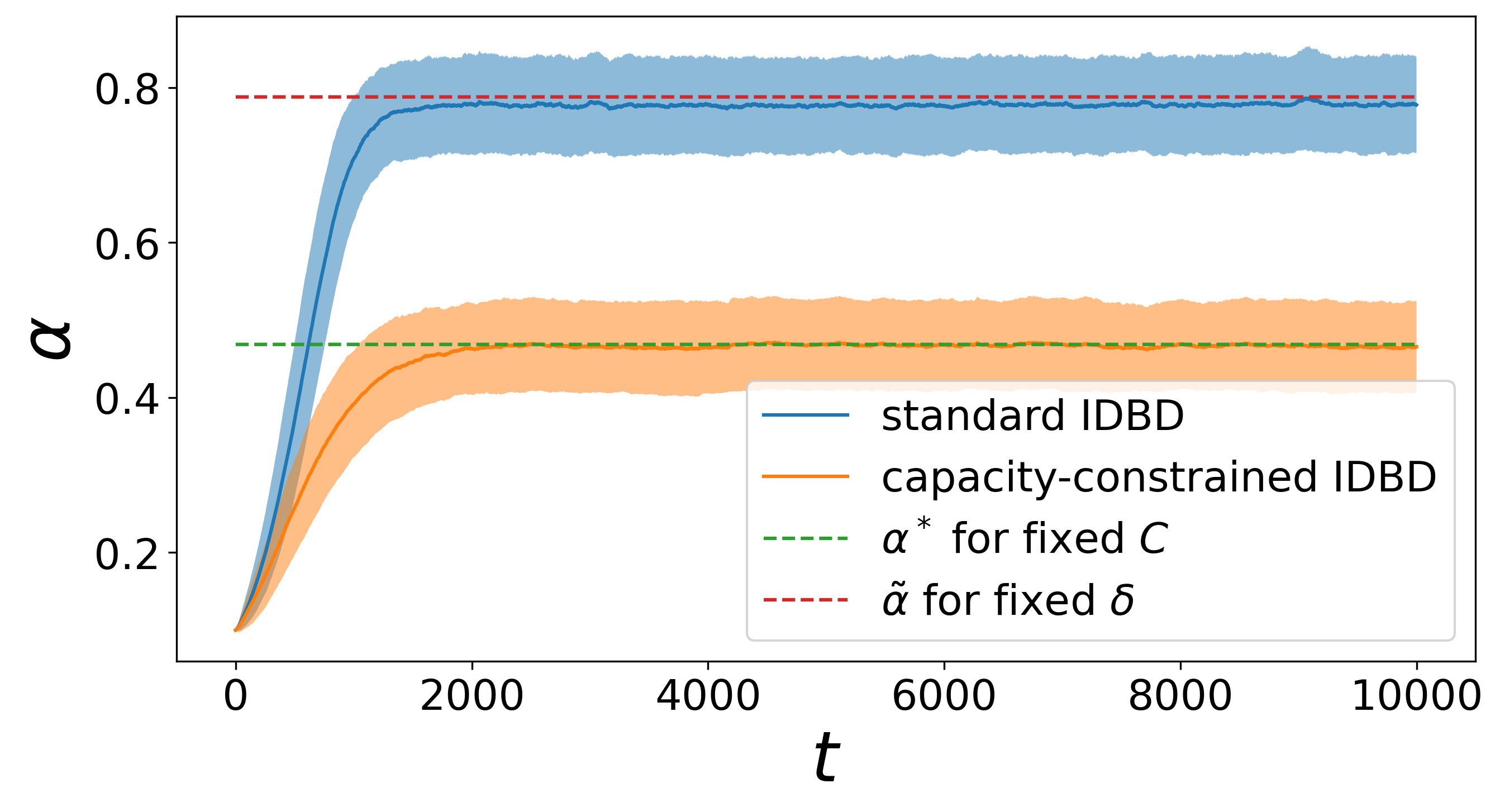}
\caption{stepsize convergence.}
\label{fig:IDBD-stepsize-convergence}
\end{subfigure} 
& \begin{subfigure}{0.4\linewidth}
\includegraphics[width=\linewidth]{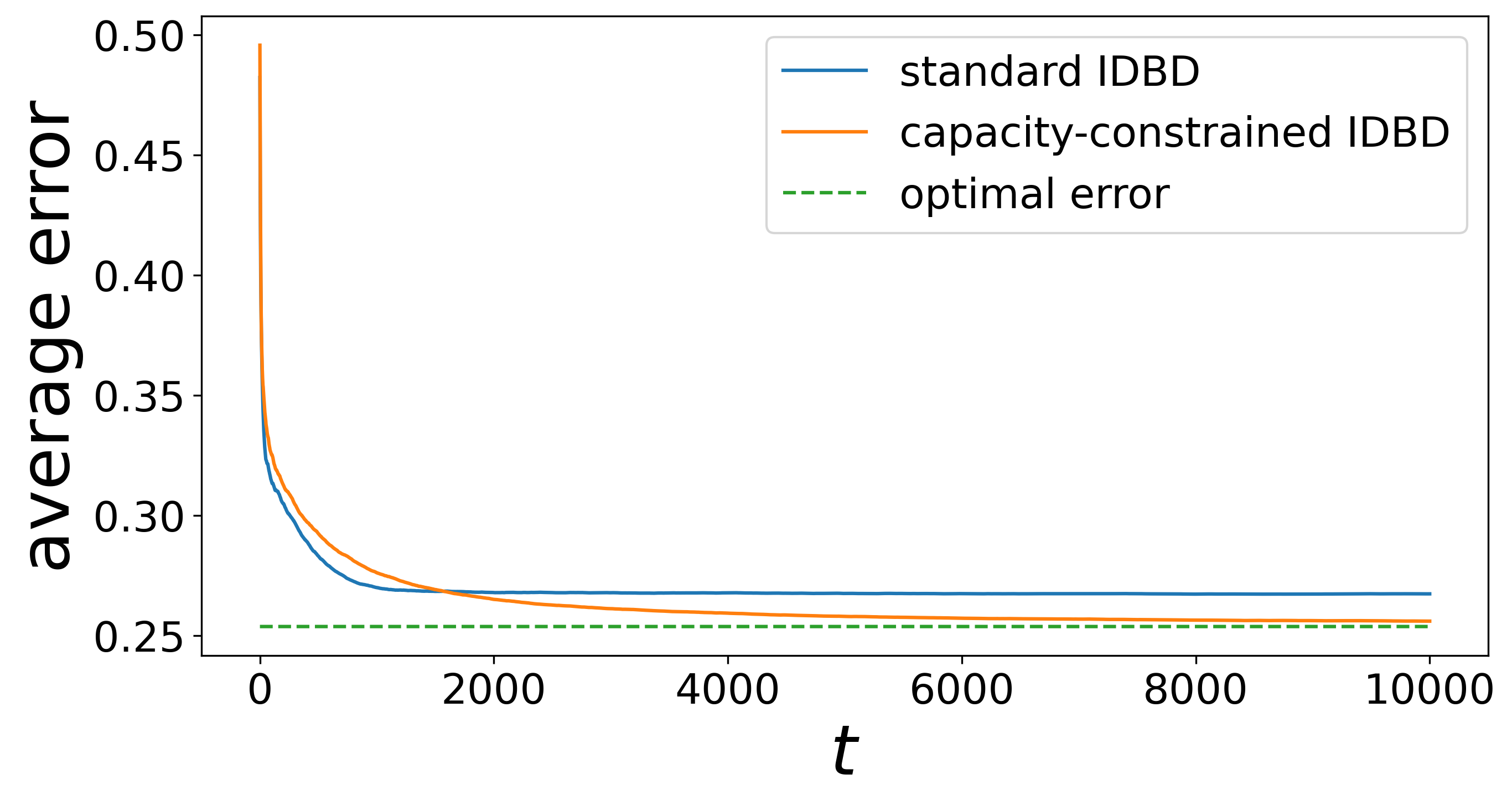}
\caption{error convergence.}
\label{fig:IDBD-error-convergence}
\end{subfigure}
\end{tabular}
\caption{Capacity-constrained IDBD converges on the optimal stepsize for the given capacity constraint.  Even if the quantization noise intensity $\delta$ is fixed at the limiting value associated with this convergence, the standard version of IDBD converges on a different stepsize, which is optimal for that intensity.  As such, while the error attained  by capacity-constrained IDBD approaches the optimal error under the agent capacity, it does not with the standard version of IDBD.
The plots are generated with $\zeta=0.01$, $\eta=0.95$, $\sigma=0.5$, $C=0.5$, and corresponding $\delta$ given by Theorem~\ref{thm:lms-capacity-noise}.}
\label{fig:IDBD-convergence}
\end{figure}

Updating $\beta_t$ relies on knowledge of the function $\delta_*$, which depends on $\eta$.  This runs counter to the purpose of stepsize adaptation schemes, which ought to arrive at $\alpha^*$ without knowledge of $\eta$.  However, the Gaussian noise $Q_{t+1}$ represents a simplified abstraction of quantization effects that would manifest if an agent were to encode information in finite memory.  With a real encoding algorithm, the capacity constraint is physical and inevitable.  As such, the agent need not itself derive the quantization noise intensity and thus may not require knowledge of $\eta$.  How to devise effective stepsize adaptation schemes for practical capacity-constrained agents presents an interesting problem for future research.

A flaw in the narrative of this section is that it does not account for information capacity required to maintain $\alpha_t$, $\beta_t$, and $h_t$, which should be incorporated into the agent state $U_t$ if the agent implements capacity-constrained IDBD.  In that event, the information capacity $C$ must support a joint quantized estimate of $\theta_t$ and these additional parameters. How information capacity should be managed in such situations remains an interesting subject for future research.

\begin{summary}
\begin{itemize}
    \item In continual supervised learning, the informational error $\E[\KL(P^*_t \| \tilde{P}_t)]$ can be decomposed into \textbf{forgetting} and \textbf{implasticity}:
    $$\underbrace{\E[\KL(P^*_t \| \tilde{P}_t)]}_{\mathrm{error}} = \sum_{k=0}^t \left(\underbrace{\I(Y_{t+1};U_{t-k-1}|U_{t-k}, H_{t-k:t})}_{\mathrm{forgetting\ at\ lag\ } k} + \underbrace{\I(Y_{t+1};O_{t-k}|U_{t-k}, H_{t-k+1:t})}_{\mathrm{implasticity\ at\ lag\ } k}\right).$$
    \begin{itemize}
    \item
    The \emph{implasticity} at lag $k$ measures the amount of information presented by $O_{t-k}$ that is useful for predicting $Y_{t+1}$ that the agent fails to ingest and is absent from the intermediate data $H_{t-k+1:t}$.
    \item
    The \emph{forgetting} at lag $k$ measures the amount of information useful for predicting $Y_{t+1}$ that the agent ejects when updating its agent state from $U_{t-k-1}$ to $U_{t-k}$ and is absent from the observations to be made before timestep $t$, $H_{t-k:t}$.
    \end{itemize}
    \item 
    If each input $X_t$ is independent of history, i.e., $X_{t+1}\perp H_t$, and the sequence of agent states and observations forms a stationary stochastic process, the stability-plasticity decomposition implies that
    $$\lim_{t\to\infty} \underbrace{\E[\KL(P^*_t \| \tilde{P}_t)]}_{\mathrm{error}} = \lim_{t\to\infty} \left(\underbrace{\I(H_{t+1:\infty}; U_{t-1}|U_t, O_t)}_{\mathrm{forgetting}} + \underbrace{\I(H_{t+1:\infty};O_t|U_t)}_{\mathrm{implasticity}}\right).$$
    \begin{itemize}
    \item 
    The first term equates error due to \emph{forgetting} with the information available in the previous agent state $U_{t-1}$, but unavailable in either in the subsequent agent state $U_t$ or the subsequent observation $O_t$.
    \item
    The second term characterizes error due to \emph{implasticity} as the information about the future $H_{t+1:\infty}$ available in the current observation $O_t$ but not ingested into the current agent state $U_t$.
    \end{itemize}
\end{itemize}
\end{summary}
\clearpage

\section{Vanishing-Regret Versus Continual Learning}
\label{se:continual-vs-convergent}

\henrik{In this section, we contrast \textit{vanishing-regret} learning with continual learning. We summarize the distinction as follows:
\begin{enumerate}
    \item In vanishing-regret learning, we view the agent as accumulating knowledge about a fixed latent variable, which can be referred to as a {\it learning target}~\citep{lu2021reinforcement,arumugam2021deciding,arumugam2021value}.  For example, in supervised learning, the learning target is typically taken to be a mapping from input to label probabilities. With infinite data, the agent converges on the learning target. Agents designed in this vein are viewed as essentially ``done'' with learning, eventually.  %
    \item In continual learning, there is no fixed learning target that the agent is learning about; rather, the agent can be interpreted as tracking a time-varying learning target. Consequently, in contrast to vanishing-regret learning, the pace of learning does not taper off. Instead, in a never-ending process, some information is retained and some forgotten as new information is ingested about the time-varying learning target.
\end{enumerate}

Computational constraints play a key role in this section: even if the environment can be characterized by a fixed latent, constraints may necessitate that the agent tracks a time-varying learning target. %

In this section, we elaborate on this distinction between vanishing-regret and continual learning.  To do so, we formalize the notion of a learning target, \textit{regret}, and what it means for regret to vanish.  We then discuss how computational constraints incentivize different qualitative behaviors, which offers insight into how and why continual learning agents should be designed differently.  
}

\subsection{Vanishing-Regret Learning}%

We begin by motivating the notion of a learning target and then discuss the role of vanishing regret in traditional machine learning.  We leverage information theory as a lens that affords general yet simple interpretations of these concepts.

\subsubsection{Learning Targets, Target Policies, and Vanishing Regret}\label{se:learning-targets}

Recall the coin tossing environment of Example \ref{ex:coins-fixed-biases}.  It is natural to think of an agent as learning about the coin biases from toss outcomes, with an eye toward settling on a simple policy that selects the coin with favorable bias.  This framing can serve to guide comparisons among alternative agents and give rise to insight on how design decisions impact performance.  For example, one could study whether each agent ultimately behaves as though it learned the coin biases.  Such an analysis assesses agent performance relative to a benchmark, which is framed in terms of an agent with privileged knowledge of a {\it learning target} comprised of coin biases and a {\it target policy} that selects the favorable coin.

In the most abstract terms, a learning target is a random variable, which we will denote by $\chi$.  A target policy $\tilde{\pi}_\chi$ assigns to each realization of $\chi$ a policy.
In particular, $\tilde{\pi}_\chi$ is a random variable that takes values in the set of policies.  Intuitively, $\chi$ represents an interpretation of what an agent aims to learn about, and $\tilde{\pi}_\chi$ is the policy the agent would use if $\chi$ were known.  Note that the target policy selects actions with privileged knowledge of the learning target.  We will denote the average reward of the target policy conditioned on the learning target by
\begin{equation}
\overline{r}_{\chi, \tilde{\pi}} = \liminf_{T \rightarrow \infty} \E_{\tilde{\pi}_\chi}\left[\frac{1}{T} \sum_{t=0}^{T-1} R_{t+1} \Big| \chi\right],
\end{equation}
where the subscript $\tilde{\pi}_\chi$ indicates the policy under which the conditional expectation is evaluated.  Note that $\overline{r}_{\chi, \tilde{\pi}}$ is a random variable because it depends on $\chi$.
This level of reward may be not attainable by any viable agent policy.  This is because it can rely on knowledge of $\chi$, which the agent does not observe.  An agent policy, on the other hand, generates each action $A_t$ based only on the history $H_t$.  Rather than offer a viable agent policy, the learning target and target policy serve as conceptual tools that more loosely guide agent design and analysis.

For each duration $T \in \mathbb{Z}_{++}$, let
$$\overline{r}_{\pi,T} = \E_\pi\left[\frac{1}{T} \sum_{t=0}^{T-1} R_{t+1}\right] \qquad \text{and} \qquad \overline{r}_{\chi, \tilde{\pi},T} = \E_{\tilde{\pi}_\chi}\left[\frac{1}{T} \sum_{t=0}^{T-1} R_{t+1} \Big| \chi\right],$$
so that $\overline{r}_\pi = \liminf_{T\rightarrow \infty} \overline{r}_{\pi,T}$ and $\overline{r}_{\chi, \tilde{\pi}} = \liminf_{T\rightarrow \infty} \overline{r}_{\chi, \tilde{\pi}, T}$.  
Like $\overline{r}_{\chi, \tilde{\pi}, T}$, the limit $\overline{r}_{\chi, \tilde{\pi}}$ is a random variable.  We define the {\it average regret} over duration $T$ incurred by $\pi$ with respect to $(\chi, \tilde{\pi})$ to be
\begin{equation}
    \overline{\mathrm{Regret}}_{\chi, \tilde{\pi}}(T | \pi) = \E[\overline{r}_{\chi, \tilde{\pi},T} - \overline{r}_{\pi,T}].
\end{equation}
This represents the expected per-timestep shortfall of the agent policy $\pi$ relative to the target policy $\tilde{\pi}_\chi$, which is afforded the advantage of knowledge about the learning target $\chi$.

We say $\pi$ exhibits {\it vanishing regret} with respect to $(\chi, \tilde{\pi})$ if $\limsup_{T \rightarrow \infty} \overline{\mathrm{Regret}}_{\chi, \tilde{\pi}}(T | \pi) \leq 0$.  If rewards are bounded and limits of $\overline{r}_{\chi, \tilde{\pi},T}$ and $\overline{r}_{\pi,T}$ exist then, by the dominated convergence theorem, $\limsup_{T \rightarrow \infty} \E[\overline{r}_{\chi,\tilde{\pi},T} - \overline{r}_{\pi,T}] \leq \E[\overline{r}_{\chi, \pi} - \overline{r}_{\pi}]$, and therefore, vanishing regret is implied by $\E[\overline{r}_{\chi, \tilde{\pi}} - \overline{r}_\pi] \leq 0$.
Intuitively, the notation of vanishing regret indicates that $\pi$ eventually performs at least as well as the target policy $\tilde{\pi}_\chi$ in spite of the latter's privileged knowledge of $\chi$.

\henrik{The choice of learning target and target policy are not uniquely determined by an agent-environment pair.}
Rather, they are chosen \henrik{only} as means to interpret the agent's performance in the environment. In particular, given $(\chi, \tilde{\pi})$, we can analyze how the agent learns about $\chi$ and uses that knowledge to make effective decisions.  In this regard, three properties make for useful choices:
\begin{enumerate}
\item Given knowledge of the learning target $\chi$, an agent can execute $\tilde{\pi}_\chi$ in a computationally efficient manner.
\item The target policy $\tilde{\pi}_\chi$ attains a desired level of average reward.
\item An agent can learn enough about $\chi$ in reasonable time to perform about as well as $\tilde{\pi}_\chi$.
\end{enumerate}
The first property ensures that knowledge of $\chi$ is actionable, and the second requires that resulting actions are performant to a desired degree.  The third property ensures that acquiring useful knowledge about $\chi$ is feasible.  These properties afford analysis of performance in terms of whether and how quickly an agent learns about $\chi$.  We will further explore this sort of analysis in Section \ref{se:regret-analysis}.  But we close this section with a simple example of logit data and an agent that produces optimal predictions conditioned on history.  This serves as an introduction to the notion of a learning target and what makes one useful.
\begin{example}
\label{ex:logit-learning-targets}
{\bf (learning targets for logit data)}
Consider a binary observation sequence $(O_t: t\in \mathbb{Z}_{++})$ that is iid conditioned on a latent variable $\theta$.  In particular, let $\ospace = \{0,1\}$, $\Prob(O_{t+1} = 1 | \theta, H_t,A_t) = e^\theta / (1+e^\theta)$, and $\Prob(\theta \in \cdot) \sim \mathcal{N}(0,1)$.  Each action is a predictive distribution $A_t = P_t$ and results in reward $R_{t+1} = \ln P_t(O_{t+1})$.  We consider an optimal agent policy $\pi_*$, which generates predictive distributions $P^*_t \sim \pi_*(\cdot|H_t)$ that perfectly condition on history: $P^*_t(\cdot) = \Prob(O_{t+1} = \cdot | H_t)$.  We will consider three different choices of learning target $\chi$ together, in each case, with the target policy  $\tilde{\pi}_\chi$ that assigns all probability to $P_t^\chi = \Prob(O_{t+1} = \cdot | \chi, H_t)$.  This target policy executes actions that are optimally conditioned on knowledge of $\chi$ in addition to history.

For the ``obvious'' learning target $\chi = \theta$, $\overline{\mathrm{Regret}}_{\chi, \tilde{\pi}}(T | \pi)$ vanishes at a reasonable rate, which we will characterize in the next section.  Intuitively, this is because, as data accumulates, the agent is able to produce estimates of $\theta$ that suffice for accurate predictions.

For contrast, let us now consider two poor choices of learning targets, each of which represents a different extreme.  One is the ``uninformative'' learning target $\chi = \emptyset$, for which $\overline{r}_{\chi,\tilde{\pi},T} = \overline{r}_{\pi,T}$ for all $T$.  Convergence is as fast as can be, and in fact, instant.  However, the role of $\chi$ is vacuous.  At the other extreme, suppose the learning target $\chi = (O_{t+1}: t \in \mathbb{Z}_+)$ includes all observations from the past, present, and future.  With this privileged knowledge, predictions $P^\chi_t(\cdot) = \mathbbm{1}(O_{t+1}=\cdot)$ perfectly anticipate observations.  
However, with this ``overinformative'' learning target, $\overline{\mathrm{Regret}}_{\chi, \tilde{\pi}}(T | \pi)$ does not vanish and instead converges to $\E[\ln(1+e^{-\theta})] > 0$.  This is because the agent never learns enough to compete with the target policy, which has privileged access to future observations. 
\end{example}

\subsubsection{Regret Analysis}
\label{se:regret-analysis}

The machine learning literature offers a variety of mathematical tools for regret and sample complexity analysis.  These tools typically study how agents accumulate information about a designated learning target and make decisions that become competitive with those that could be made with privileged knowledge of the learning target.  To more concretely illustrate the nature of this analysis and the role of learning targets, we will present in this section examples of results in this area.  In order to keep the exposition simple and transparent, we will restrict attention to supervised learning.

Rather than cover results about specific agents, we will review results about an optimal agent policy, which produces predictions $P^*_t(\cdot) = \Prob(Y_{t+1} = \cdot | H_t)$.  Such an agent maximizes $\E_\pi[\frac{1}{T} \sum_{t=0}^{T-1} R_{t+1}]$ over every duration $T$, with rewards given by $R_{t+1} = \ln P^*_t(Y_{t+1})$.  Hence, the results we review pertain to what is possible rather than what is attained by a particular algorithm.

For any learning target $\chi$, we will take the target policy to be that which generates actions $P_t^\chi(y) = \Prob(Y_{t+1} = y | \chi, H_t)$.  This represents an optimal prediction for an agent with privileged knowledge of $\chi$ in addition to the history $H_t$.  For this target policy, average regret satisfies
\begin{align*}
\overline{\mathrm{Regret}}_{\chi, \tilde{\pi}}(T | \pi)
=& \overline{r}_{\chi, \tilde{\pi}, T} - \overline{r}_{\pi,T} \\
=& \E\left[\E\left[\frac{1}{T} \sum_{t=0}^{T-1} \ln P^\chi_t(Y_{t+1}) \Big| \chi\right]\right] - \E\left[\frac{1}{T} \sum_{t=0}^{T-1} \ln P^*_t(Y_{t+1}) \right] \\
=& \E\left[\frac{1}{T} \sum_{t=0}^{T-1} \E\left[\ln \frac{P^\chi_t(Y_{t+1})}{P^*_t(Y_{t+1})} \Big| \chi\right] \right] \\
=& \E\left[\frac{1}{T} \sum_{t=0}^{T-1} \KL(P_t^\chi \| P_t^*)\right].
\end{align*}

To characterize regret incurred by an optimal agent, we start with a basic result of \cite[Theorem 9]{jeon2023informationtheoretic}:
\begin{equation}
\label{eq:entropy-bound}
\overline{\mathrm{Regret}}_{\chi, \tilde{\pi}}(T | \pi) \leq \frac{\Ent(\chi)}{T}.
\end{equation}
The right-hand side is the entropy of the learning target divided by the duration $T$.  If this entropy is finite then, as $T$ grows, the right-hand side, and therefore regret, vanishes.  That the bound increases with the entropy of the learning target is intuitive: if there is more to learn, regret ought to be larger for longer.

While the aforementioned regret bound offers useful insight, it becomes vacuous when the $\Ent(\chi) = \infty$.  Entropy is typically infinite when $\chi$ is continuous-valued.  In order to develop tools that enable analysis of continuous variables and that lead to much tighter bounds, even when $\Ent(\chi) = \infty$, let us introduce a new concept: the rate-distortion function.  Let
$$\Theta_\epsilon = \left\{\tilde{\chi}: \E\left[\KL\left(P_t^\chi \| P_t^{\tilde{\chi}}\right)\right] \leq \epsilon \text{ for all } t \right\}$$
be the set of all random variables $\tilde{\chi}$ that enable predictions close to those afforded by the $\chi$.  The \emph{rate} of $\chi$ with \emph{distortion} tolerance $\epsilon$ is defined by the rate-distortion function:
\begin{equation}
\label{eq:rate-distortion}
\Ent_\epsilon(\chi) = \inf_{\tilde{\chi} \in \Theta_\epsilon} \I(\chi; \tilde{\chi}).
\end{equation}
Loosely speaking, this is the number of nats required to identify a useful approximation to the learning target.  The rate-distortion function $\Ent_\epsilon$ serves an alternative upper bound as well as a lower bound \cite[Theorem 12]{jeon2023informationtheoretic}:
\begin{equation}
\label{eq:rate-distortion-bound}
\sup_{\epsilon \geq 0} \min\left(\frac{\Ent_\epsilon(\chi)}{T}, \epsilon\right) \leq \overline{\mathrm{Regret}}_{\chi, \tilde{\pi}}(T | \pi)
\leq \inf_{\epsilon \geq 0} \left(\frac{\Ent_\epsilon(\chi)}{T} + \epsilon\right).
\end{equation}
Even when $\Ent(\chi) = \infty$, the rate $\Ent_\epsilon(\chi)$ can increase at a modest pace as $\epsilon$ vanishes.  We build on Example \ref{ex:logit-learning-targets} to offer a simple and concrete illustration.  While that example was not framed as one of supervised learning, it can be viewed as such by taking each observation $O_{t+1}$ to encode only a label $Y_{t+1}$ resulting from a non-informative input $X_t$.  We will characterize the rate-distortion function and regret for each of the three learning targets considered in Example \ref{ex:logit-learning-targets}.
\begin{example}
{\bf (rate-distortion and convergence for logit data)}
\label{ex:rate-distortion-logit}
Recall the environment of Example \ref{ex:logit-learning-targets}, in which binary labels are generated according to $\Prob(O_{t+1} = 1 | \theta, H_t, A_t) = e^\theta / (1+e^\theta)$ based on a latent variable $\theta \sim \mathcal{N}(0,1)$.

For the ``obvious'' learning target $\chi = \theta$ and all $\tilde{\chi}$, $\E[\KL(P^{\chi}_t\|P^{\tilde{\chi}}_t)] \leq \E[(\chi - \E[\chi|\tilde{\chi}])^2]$. As a result, $\tilde{\chi} \sim \mathcal{N}(\chi, \epsilon)$ is an element
of $\Theta_\epsilon$, and therefore,
$$\Ent_{\epsilon}(\chi) \leq \frac{1}{2}\ln\left(1+\frac{1}{\epsilon}\right).$$
It follows from Equation \ref{eq:rate-distortion-bound} that
$$\overline{\mathrm{Regret}}_{\chi, \tilde{\pi}}(T | \pi) \leq\ \inf_{\epsilon \geq 0} \left(\frac{1}{2T} \ln\left(1+\frac{1}{\epsilon}\right) + \epsilon \right) \ \leq\ \frac{\ln\left(1+2T\right) + 1}{2T}.$$

For the ``uninformative'' learning target $\chi = \emptyset$, $\Ent_\epsilon(\chi) = 0$, and therefore $\overline{\mathrm{Regret}}_{\chi, \tilde{\pi}}(T | \pi) = 0$ for all $T$.  For the ``overinformative'' learning target $\chi = (O_{t+1} :t \in \mathbb{Z}_+)$, on the other hand, for $\epsilon < \E[\ln(1+e^\theta)/(1+e^\theta)]$, $\Ent_\epsilon(\chi) = \infty$ and $\overline{r}_{\chi, T} - \overline{r}_{*,T} > \E[\ln(1+e^\theta)/(1+e^\theta)]/2 > 0$.
\end{example}
For this example, though the ``obvious'' choice of learning target $\chi = \theta$ yields $\Ent(\chi) = \infty$, the rate $\Ent_\epsilon(\chi)$ is modest even for small positive values of $\epsilon$.  Because of this, an optimal agent converges quickly, as expressed by $O((\ln T)/T)$ bound.  Further, knowledge of $\theta$ enables efficient computation of predictions $P_t(1) = \Prob(Y_{t+1} = 1 | \theta, H_t) = e^\theta / (1+e^\theta)$.  On the other hand, the ``underinformed'' learning target $\chi = \emptyset$ is not helpful, and with respect to the ``overinformed'' learning target, regret does not vanish.

\subsection{Continual Learning}\label{sec:subsection_cl}

\henrik{Unlike in more common framings of machine learning with vanishing-regret}, continual learning does not generally afford an obvious choice of learning target.  \henrik{In vanishing-regret learning}, the agent is considered ``done'' with learning when it has gathered enough information about a learning target to make effective decisions.  \henrik{In contrast, performant continual learning agents perpetually ingest new information.}  %
\henrik{There are two reasons for why this unending process may be incentivized:
\begin{enumerate}
    \item There is no fixed learning target that an unconstrained agent can converge on. To do well, the agent need to continually learn.
    \item Even if the environment is characterized by a fixed latent variable, constraints may incentivize the agent to continually forget and learn anew.
\end{enumerate}}

\henrik{
Regret-based analysis of algorithms in special types of nonstationary environments has been conducted extensively in the literature. In particular, several works on nonstationary MDPs~\citep{domingues2021kernel,fei2020dynamic}, nonstationary bandits~\citep{min2023information,liu2023definition,luo2018efficient,bogunovic2016time}, and online convex optimization~\citep{zhang2018adaptive,van2016metagrad,besbes2015non,daniely2015strongly} study nonstationary environments. Several regret-based metrics have been designed for the nonstationary settings. An especially common type of benchmark for regret-based analysis in this context is a so-called ``dynamic regret." This metric computes regret with respect to a changing learning target.

In contrast to the previous work cited above, in this section we emphasize the role of computational constraints in incentivizing continual learning behavior. Further, we highlight that typically there is often no obvious fixed learning target nor a time-varying learning target that can serve as a useful benchmark for regret-based analysis. In addition, what constitutes a reasonable learning target will depend on the computational constraints of the agents under consideration. 
}

\subsubsection{Constraints Induce Persistent Regret}

Recall that practical agent designs typically maintain an agent state, with behavior characterized by functions $\psi$ and $\pi$.  In particular, the agent state evolves according to $U_{t+1} \sim \psi(\cdot | U_t, A_t, O_{t+1})$ and actions are sampled according to $A_t \sim \pi(\cdot | U_t)$.  Hence, the agent policy is encoded in terms of the pair $(\psi, \pi)$.  Let $\overline{r}_{\psi,\pi} = \E_{\pi,\psi}[\frac{1}{T} \sum_{t=0}^{T-1} R_{t+1}]$ denote the average reward attained by such an agent.

The information content of $U_t$ is quantified by the entropy $\Ent(U_t)$, which is constrained by the agent's information capacity.  As discussed in Section \ref{sec:compute_vs_infocapacity}, for common scalable agent designs, computation grows with this information content. 
 Hence, computational constraints limit information capacity.  To understand implications of this restriction, in this section, we focus on the problem of agent design with fixed information capacity $C$:
\begin{align}
\label{eq:capacity-constrained-rl}
\begin{split}
\max_{\psi,\pi} & \quad \overline{r}_{\psi,\pi} \\
\text{s.t.} & \quad \sup_t \Ent(U_t) \leq C.
\end{split}
\end{align}

The following simple example, which is similar to one presented in \cite[Section 2]{sutton2007role}, illustrates how such a constraint induces persistent regret.
\begin{example}
\label{ex:bit-flipping}
{\bf (bit flipping)}
Consider an environment that generates a sequence of binary observations $\mathcal{O} = \{0,1\}$, initialized with $O_1$ distributed $\mathrm{Bernoulli}(1/2)$, and evolving according to
$\Prob(O_{t+1} = 1 | p, H_t, A_t) = p (1 - O_t) + (1-p) O_t$, 
where $p$ is a random variable.  Note that $p$ governs the probability of a bit flip: \henrik{the probability of the next observation $O_{t+1}$ being different from the current observation $O_{t}$}.  Each action is a binary prediction $A_t$ of the next observation $O_{t+1}$ and yields reward $R_{t+1} = \mathbbm{1}(A_t = O_{t+1})$.  In other words, the agent predicts the next bit and receives a unit of reward if its prediction is correct.

It is natural to consider $\chi = p$ as a learning target.  In particular, with privileged knowledge of this learning target, an agent can act according to a target policy $\tilde{\pi}$ that generates a prediction $A_t = 1$ if and only if $\Prob(O_{t+1}=1 | p, H_t) > 1/2$ or, equivalently, $p (1-O_t) + (1-p) O_t > 1/2$.  The optimal unconstrained agent policy $\pi$ generates action $A_t = 1$ if and only if $\Prob(O_{t+1}=1 | H_t) > 1/2$ and satisfies $\overline{\mathrm{Regret}}_{\chi, \tilde{\pi}}(T | \pi) \rightarrow 0$.  This is because, if not hindered by constraints, an agent can identify $p$ over time from observations.

Now suppose the agent is constrained by an information capacity of $C = \ln 2 \mathrm{\ nats} = 1 \mathrm{\ bit}$. \henrik{In this case, the agent can at most store one observation from the past. Therefore,} encoding an accurate approximation of $p$ in the agent state becomes infeasible.  \henrik{Instead, the optimal solution is to store only the most recent observation, and use that together with the prior over $p$ to predict either $0$ or $1$.} \henrik{That is,} the optimal solution to Equation \ref{eq:capacity-constrained-rl} is \henrik comprised of an agent state update function $U_t = O_t$ and a policy function for which $A_t = 1$ if and only if $\E[p] (1-U_t) + \E[1-p] U_t > 1/2$.  In other words, the agent simply retains its most recent observation as agent state and predicts a bit flip with probability $\E[p]$.  This agent does not aim to identify $\chi = p$ nor, for that matter, any nontrivial learning target.  And regret with respect to $\overline{\mathrm{Regret}}_{\chi, \tilde{\pi}}(T | \pi)$ does not vanish.
\end{example}

\subsubsection{Nonstationary Learning Targets}

Continual learning is often characterized as addressing ``nonstationary" environments. \henrik{However, as discussed in the previous section, 
even in a ``stationary" environment (in the sense that it can be well-characterized by a fixed latent variable), a capacity-constrained agent will be incentivized to do continual learning.}   %
Motivated by this view, we interpret {\it nonstationarity} as indicating agent behavior that is naturally explained by a time-varying learning target.
Example \ref{ex:bit-flipping} illustrates how this sort of behavior is incentivized by restricting information capacity.  The optimal unconstrained agent learns about the latent variable $p$, which constitutes a fixed learning target.  On the other hand, the capacity-constrained agent learns at each time about $O_t$.  This bit of information is ingested into the agent state and can be viewed as a nonstationary learning target $\chi_t = O_t$.

Our narrative is aligned with that of \cite{sutton2007role}, who suggest that tracking a nonstationary learning target is warranted even in a stationary environment if the environment is sufficiently complex.  
Consider a supervised learning environment for which an enormous amount of information is required to attain reasonable performance.  In other words, for any reasonable $\epsilon$, $\Ent_\epsilon(\chi^*)$ is huge.  An agent with modest capacity $C \ll \Ent_\epsilon(\chi^*)$ ought not aim to accumulate all information required to attain error $\epsilon$ because of its insufficient information capacity.  However, it may be possible to attain error $\epsilon$ by only retaining at each timestep information that will be helpful in the near term.  This amounts to tracking a nonstationary learning target.

Our notion of a nonstationary learning target relates to the work of \cite{Abel2023definition}, which proposes a definition of continual reinforcement learning.  This definition offers a formal expression of what it means for an agent to never stop learning.  The characterization is subjective in that it defines non-convergence with respect to a {\it basis}, which is a fixed set of policies.  An agent is interpreted as searching among these and then of converging if it settles on one that it follows thereafter.  The work associates continual learning with non-convergence.  A learning target offers an alternative subjective construct for interpretation of agent behavior.  Loosely speaking, if an agent converges on an element of the basis, that can be thought of as converging to a fixed learning target.  Perpetually transitioning among elements of the basis is akin to pursuing a nonstationary learning target.

\henrik{
\subsubsection{On the Choice of Learning Target}
}

\henrik{In Bayesian framings of machine learning, it is common to characterize the environment as an unknown latent variable. The problem formulation consists of two primitives: a prior over an unknown latent variable $\chi$ and a likelihood model that generates observations. The agent is then interpreted as learning about the environment -- that is, this latent variable -- from its experience.  Indeed, the random environment can serve as a learning target in the vein of vanishing-regret learning.  
In supervised learning, if the input distribution is known, this latent variable may be the unknown function that maps inputs to distributions over labels. 
In reinforcement learning, it may be the unknown transition probability matrix of a Markov Decision Process (MDP).

\henrik{In supervised learning, the reason why the prior and likelihood model are taken as primitives, is because the stream of pairs are assumed to be \textit{exchangeable}}. The data $\{X_t, Y_{t+1}\}_{t=0}^\infty$ being exchangeable means that for all horizons $T$, all permutations of the sequence $\{X_t, Y_{t+1}\}_{t=0}^T$ have the same distribution. Consequently, de Finetti's Theorem establishes existence of a learning target $\chi$ conditioned on which the data pairs are iid \citep{de1929funzione}.  It is natural to think of $\chi$ as identifying an unknown environment and bringing to bear interpretation and analysis afforded by vanishing-regret learning.  Specifically, assuming the input distribution is known, $\chi$ is often taken to be an unknown function that maps inputs to distributions over labels. A similar story goes for the common MDP framing of reinforcement learning.  Under a weaker exchangeability condition presented in Appendix \ref{app:mdp_exchangeability}, it is natural to characterize dynamics in terms of an unknown MDP $\chi$.  In both cases, an agent can learn over time to do about as well as if $\chi$ were known. Thus, $\chi$ can serve as a useful learning target in regret-based analysis.}

\henrik{In contrast, environments of the sort considered in continual learning tend not to give rise to an obvious choice of fixed learning target. In particular, regret generally does not vanish with respect to any particular learning target that enables efficient computation of performant actions. 
For example, \citet{liu2023definition} discusses how regret does not vanish with respect to common choices of learning targets in non-stationary bandit learning. 
\henrik{This paper further discusses how dynamic regret can be a misleading metric depending on what the time-varying learning target is taken to be; there may not be an obvious choice of latent process that can serve as a useful time-varying learning target.}

To help interpret the behavior of continual learning agents, it can often be intuitive to characterize a specific environment in terms of a changing latent variable that the agent is tracking~\citep{sutton2007role}. While this latent variable may not serve as a useful learning target for regret-analysis, such a characterization can be helpful to facilitate agent design.  For instance, many of the examples in this monograph are described as latent AR(1) processes, such as the scalar tracking problem in Appendix \ref{sec:tracking}.  %
}\\

\begin{summary}
    \begin{itemize}
        \item A \textbf{learning target} (denoted by $\chi$) is random variable that represents what an agent aims to learn about.
        \item A \textbf{target policy} (denoted by $\tilde{\pi}_\chi$) is a random variable that represents the policy the agent would use if $\chi$ were known.
        \item For each duration $T \in \mathbb{Z}_{++}$,
        $$\overline{r}_{\pi, T} = \E_{\pi}\left[\frac{1}{T}\sum_{t=0}^{T-1}R_{t+1}\right]\qquad \text{and}\qquad
        \overline{r}_{\chi,\tilde{\pi},T} = \E_{\tilde{\pi}_\chi}\left[\frac{1}{T}\sum_{t=0}^{T-1}R_{t+1}\Big| \chi\right].$$
        \item The \textbf{average regret} over duration $T$ incurred by $\pi$ w.r.t $(\chi, \tilde{\pi})$ is
        $$\overline{\mathrm{Regret}}_{\chi, \tilde{\pi}}(T|\pi) = \E\left[\overline{r}_{\chi, \tilde{\pi}, T} - \overline{r}_{\pi, T}\right].$$
        \item $\pi$ exhibits \textbf{vanishing regret} w.r.t. $(\chi, \tilde{\pi})$ if $\limsup_{T\to\infty}\overline{\mathrm{Regret}}_{\chi, \tilde{\pi}}(T|\pi) \leq 0$.
        \item In the special case of \textbf{continual supervised learning}, $P_t^*(y) = \Prob(Y_{t+1}=y|H_t)$ maximizes $\overline{r}_{\pi, T}$ and $P_t^{\chi}(y) = \Prob(Y_{t+1}=y|\chi, H_t)$ maximizes $\overline{r}_{\chi, \tilde{\pi}, T}$ for every duration $T$.
        For the above optimal choices of $\pi, \tilde{\pi}_{\chi}$,
        $$\overline{\mathrm{Regret}}_{\chi, \tilde{\pi}}(T|\pi) = \E\left[\frac{1}{T}\sum_{t=0}^{T-1}\KL(P_t^{\chi}\|P_t^*)\right].$$
        For all durations $T$,
        $$\overline{\mathrm{Regret}}_{\chi, \tilde{\pi}}(T|\pi) \leq \frac{\Ent(\chi)}{T}.$$
        For continuous-valued $\chi$, $\Ent(\chi)$ is often $\infty$, leading to a vacuous upper bound. The bound can be improved via \emph{rate-distortion theory}.
        For any $\epsilon \geq 0$, let 
        $$\Theta_{\epsilon} = \left\{\tilde{\chi}: \E[\KL(P_t^{\chi}\|P_t^{\tilde{\chi}})] \leq \epsilon \text{ for all } t\right\}$$
        be the set of all random variables $\tilde{\chi}$ which enable predictions with \emph{distortion} at most $\epsilon$.
        The \emph{rate} of $\chi$ with \emph{distortion} tolerance $\epsilon$ is defined by the \textbf{rate-distortion function}:
        $$\Ent_{\epsilon}(\chi) = \inf_{\tilde{\chi}\in\Theta_\epsilon}\I(\chi;\tilde{\chi}).$$
        For all durations $T$,
        $$\sup_{\epsilon \geq 0} \min\left(\frac{\Ent_\epsilon(\chi)}{T}, \epsilon\right) \leq \overline{\mathrm{Regret}}_{\chi, \tilde{\pi}}(T | \pi)
\leq \inf_{\epsilon \geq 0} \left(\frac{\Ent_\epsilon(\chi)}{T} + \epsilon\right).$$
        \end{itemize}
\end{summary}
\newpage

\section{Case Studies}

Continual learning has broad scope, with potential applications ranging from recommendation and dialogue systems to robotics and autonomous vehicles. In Section~\ref{sec:computational_constraints}, we framed continual learning as computationally constrained RL with an objective of maximizing average reward subject to constraints (Equation~\ref{eq:objective}). 
This objective is designed to encompass real-world requirements of continual learning systems across applications. In this section, we study the implications of our objective on the design of performant continual learning agents.

To study these implications, we perform three case studies.
Each case study highlights a different facet of continual learning: continual supervised learning, continual exploration, and continual learning with delayed consequences.
In our first case study, we consider the special case of continual learning where the agent's actions do not influence future observations. %
In other environments, such as bandits, the agent's actions influence the agent's immediate observations and therefore what information the agent is exposed to. Therefore, when selecting its next action, an agent should not only take into account the immediate reward (like in supervised learning), but also what information it can gain in order to perform well in the long term. Selecting actions for the purpose of acquiring new information is known as \textit{exploration}. In our second case study, which is on \textit{continual exploration}, we study the implications of \henrik{nonstationary learning targets} %
for exploration. Our third case study focuses on the broader class of environments in which actions induce not only immediate consequences but also delayed consequences. We call this \textit{continual learning with delayed consequences}. In our final case study on continual auxiliary learning, we study the benefits of learning auxiliary tasks that are distinct from, though possibly helpful to, the primary task of maximizing average reward. %
In each case study, we perform simple illustrative experiments to develop intuitions for what behaviors a performant continual learning agent might need.

\subsection{Continual Supervised Learning} \label{sec:supervised_learning}
\newcommand{\noreset}{\texttt{No Reset}}
\newcommand{\noresetsmall}{\texttt{No Reset + Limited Memory}}
\newcommand{\reset}{\texttt{Reset}}
\newcommand{\largememory}{\texttt{Large Memory}}
\newcommand{\smallmemory}{\texttt{Small Memory}}

In this section, we consider the continual supervised learning setting, which has been the focus of most prior work in continual learning. In continual supervised learning, an agent receives a nonstationary\footnote{\henrik{Note that we use nonstationary in the loose sense of Section \ref{se:continual-vs-convergent} to refer to a case where there is no fixed learning target that the agent can converge on.}} sequence of data pairs from its environment, and the agent's actions do not influence the data sequence. \henrik{In such data sequences, the data distribution changes over time in different ways. The input distribution may change, the target function may change, or the label distribution may change. For an elaboration on these types of distribution shift, see Appendix~\ref{app:types_of_nonstationarity}}. The agent's goal is to make accurate predictions. As an example, consider a fraud detection system for credit card transactions. The goal of such a system is to classify transactions as fraudulent or not fraudulent. The techniques people use to commit fraud and evade detection may evolve over time, and the system must adapt to these changing fraud patterns. \henrik{In this case, each observation $O_{t+1} \in \mathcal{O}$ has two parts, the label $Y_t$ for the previous transaction and the next transaction $X_{t+1}$. The label $Y_t$ denotes whether the previous example was fraudulent (if that is known). The input $X_{t+1}$ is the transaction to be classified. The actions are $\mathcal{A} = \{ \mathrm{fraudulent}, \mathrm{not\ fraudulent} \}$. This problem can then be modeled as a tracking problem where the agent is tracking a time-varying learning target $f_t$ that maps each transaction to a probability of the transaction being fraudulent. The observation probability function $\rho$ captures everything that is known about how the input distribution as well $f_t$ evolves over time. As explained in Section \ref{sec:continual_interaction} it may not be posssible to write down $\rho$ explicitly.}

Research in continual supervised learning over the years has proposed synthetic problems for comparing and evaluating different agent designs. For instance, an agent may be learning to classify digits and over time the digits may rotate~\citep{buzzega2020dark}. 
On such problems, a common evaluation protocol is to periodically measure the agent’s performance on all types of data seen so far. In the digit classification example, the agent’s prediction accuracy is evaluated on digits with previously encountered rotation angles. Though not stated explicitly, this means that the agent’s objective is to learn about and remember everything in the past.

However, as we have mentioned before, this objective --- to remember everything --- 
is neither feasible nor necessary. In the real world, the goal is to do well on future tasks, not on those in the past. In the fraud detection example, some fraud techniques may become outdated or impossible to use due to new security measures, and the system may not need to remember how to handle those types of cases. Generally, in a dynamic environment, some knowledge becomes obsolete and not necessary to remember. \henrik{As an example, consider a question-answering system about computer repairs. If a new computer model is produced, eventually the new model becomes widely adopted. Thus, the agent can predict that information about repairs for the old model will become obsolete since there will be fewer and fewer questions about that as time progresses.} Further, intelligent systems will be computationally constrained in practice and cannot remember everything. Our average-reward objective under computational constraints tries to capture these real-world requirements. Importantly, this objective has multiple implications for continual supervised learning, especially with regard to forgetting and plasticity.

In this section, we perform a set of simple experiments on a synthetic problem to highlight some of the implications of our average-reward objective in the continual supervised learning setting. We specifically study the impact of forgetting when information recurs and when the agent is computationally constrained. Through our experiments, we make the following points: (1) a performant agent can forget non-recurring information, (2)  a performant agent can forget recurring information if it relearns that information quickly, and (3) under tight computational constraints, forgetting may be helpful.

\subsubsection*{Environment}

We perform our experimental evaluation on a modified version of Permuted MNIST, a common benchmark from the continual learning literature \citep{goodfellow2013empirical}. Permuted MNIST is characterized by a sequence of training datasets. Each dataset is constructed in two steps: (1) we randomly sample a permutation over all pixels, and (2) we apply this permutation to every image from the standard MNIST training dataset. Each such permuted training set is referred to as a {\it task}.

\saurabh{
Since each Permuted MNIST task applies a random permutation to the input pixels, neural networks may struggle to leverage information from previous tasks to do well on future tasks. That is, the amount of forward transfer may be low. To reduce the forward transfer further, we create a modified version of Permuted MNIST in which we additionally permute the labels randomly in each task. For instance, all the images with the label $3$ may get assigned a different label, such as $5$. All in all, we expect there to be little forward transfer when using neural networks. For additional discussion about forward transfer, see Appendix \ref{app:forward_transfer}.}

In this setup, a continual learning data sequence 
consists of a sequence of tasks. \henrik{In turn, each task lasts for $k$ timesteps, where at each timestep a single batch of images arrives. We call the number of timesteps $k$ a task occurs before switching to the next task the task \textit{duration}.}
Each incoming batch from the data sequence contains $b_\text{env}$ images, where the subscript env is shorthand for ``environment." 
In our experiments, $b_\text{env} = 16$.

While the Permuted MNIST benchmark is typically characterized by a sequence of permutations where each permutation occurs once, we consider a particular type of environment dynamics that exhibits \textit{periodic recurrence}. Specifically, in our version of Permuted MNIST, the first permutation recurs periodically: every alternate task is the first permutation, while all other tasks are determined by new permutations which do not recur. As an example, given $100$ permutations P1, P2, ..., P100 
the sequence of tasks may look like P1, P2, P1, P3, P1, P4, P1, P5, ....
\henrik{We hope to study the effect of information recurrence specifically induced by the first permutation recurring. We expect that this effect can be isolated from other sources of forward transfer since we have permuted both input pixels and labels, which should reduce forward transfer between permutations when using neural networks.}

\subsubsection*{Agents}
A common agent design that improves an agent's ability to retain past information in continual supervised learning is to equip the agent with a replay buffer~\citep{aljundi2019gradient,buzzega2020dark,chaudhry2019continual,chrysakis2020online,yoon2021online}. The agent may use this buffer to store past data pairs and continue extracting information from them. In line with this approach, we consider three simple agents which store data in a replay buffer and are trained using SGD. We constrain all agents to perform a single SGD step per timestep. The agents perform each SGD step on a batch containing both a batch of incoming data ($b_{\text{env}}$ data pairs) and a batch of data sampled uniformly from the replay buffer ($b_\text{replay}$ data pairs). 
All agents use so-called \textit{reservoir insertion} to add data to the replay buffer, which ensures that the full buffer is a random sample from the full history \citep{vitter1985random}. Specifically, given the buffer size $B$ and length $T$ of the entire history $H_T$, reservoir insertion guarantees each data pair has the same probability $\frac{B}{T}$ of being stored in the buffer, without knowing the length of the entire data stream in advance. This strategy is common practice in the design of continual learning agents with a limited capacity replay buffer \citep{buzzega2020dark, koh2023online}.

The three agents are \largememory \space, which has a buffer size of $B = 1$ million data pairs, \smallmemory, which can store only $B = 1,000$ data pairs, and \reset, which has the same replay buffer size as \largememory \space but periodically resets all agent components (the neural network parameters are re-initialized and the buffer is emptied). In experiments where we study the effect of resetting, we refer to the \largememory \space agent as \noreset. For all agents, we use a $2$ hidden layer neural network with hidden layer width $1000$, and we set $b_\text{replay} = 16$.

\subsubsection*{Evaluation Protocol}
Most works in continual supervised learning evaluate agents on their performance on previous tasks. In contrast, we evaluate agents on their ability to maximize average reward, which is in line with previous work that considers the online continual learning setting \citep{cai2021online,ghunaim2023real,prabhu2023online}. Because the primary goal in supervised learning is often to achieve good accuracy, we let the reward function be online accuracy. Our objective is thus to maximize average online accuracy under computational constraints. Using the notation presented in Section \ref{sec:objective}, this objective can be written as follows: 

\begin{align*}
\begin{array}{ll}
\max_\pi & \liminf_{T\rightarrow\infty} \E_\pi \left[\frac{1}{T} \sum_{t=0}^{T-1} \mathbbm{1}(Y_{t+1} = \hat{Y}_{t+1}) \right] \\
\text{s.t.} & \text{computational constraint}
\end{array}
\end{align*}

where at each timestep $t$, $Y_{t+1}$ is the true label of observation $X_t$, $\hat{Y}_{t+1} \in \argmax_y P_t(y \in \cdot)$ is the agent's predicted label, and $\mathbbm{1}$ is the indicator function.

Our average reward objective is the limit of an expectation that integrates over a growing sequence, which in most cases is infeasible to compute. Instead, we will evaluate an approximation by (1) limiting the duration of interaction and (2) approximating the expectation by averaging over a finite number of sequences. We propose a variation of the train/test protocol used in supervised learning. First, we split the data into two parts: a single development sequence and a set of evaluation sequences. We use the development sequence to tune an agent by varying hyperparameters. After selecting the best hyperparameters based on the development sequence, we reset the agent (re-initialize the neural network and replay buffer) and train the agent on the evaluation sequences. Finally, the agent's performance is averaged across all evaluation sequences.

In accordance with this evaluation protocol, on our modified Permuted MNIST environment, we split all data into two subsets, one used to generate the development sequence and the second used to generate evaluation sequences. Each subset contains $100$ permutations, and the two subsets do not share any permutations. Consequently, there is no overlap in permutations between the development sequence and the evaluation sequences. Each permutation has $400$ unique data pairs, for a total of $40,000$ unique data pairs per subset of data.
On each subset of data, we train agents over $3$ random seeds, where the seed determines both the initialization of the agent and the generated sequence of data pairs the agent receives from that subset. We perform hyperparameter tuning on the development sequence and report evaluation performance averaged over all $3$ seeds (and therefore $3$ evaluation sequences).
For additional details on the environments and evaluation protocol, see Appendix \ref{apx:sl-experiments}.

\subsubsection*{Results}

\begin{figure}
\centering
\begin{tabular}{cc}
\begin{subfigure}{0.48\linewidth}
\includegraphics[width=\linewidth]{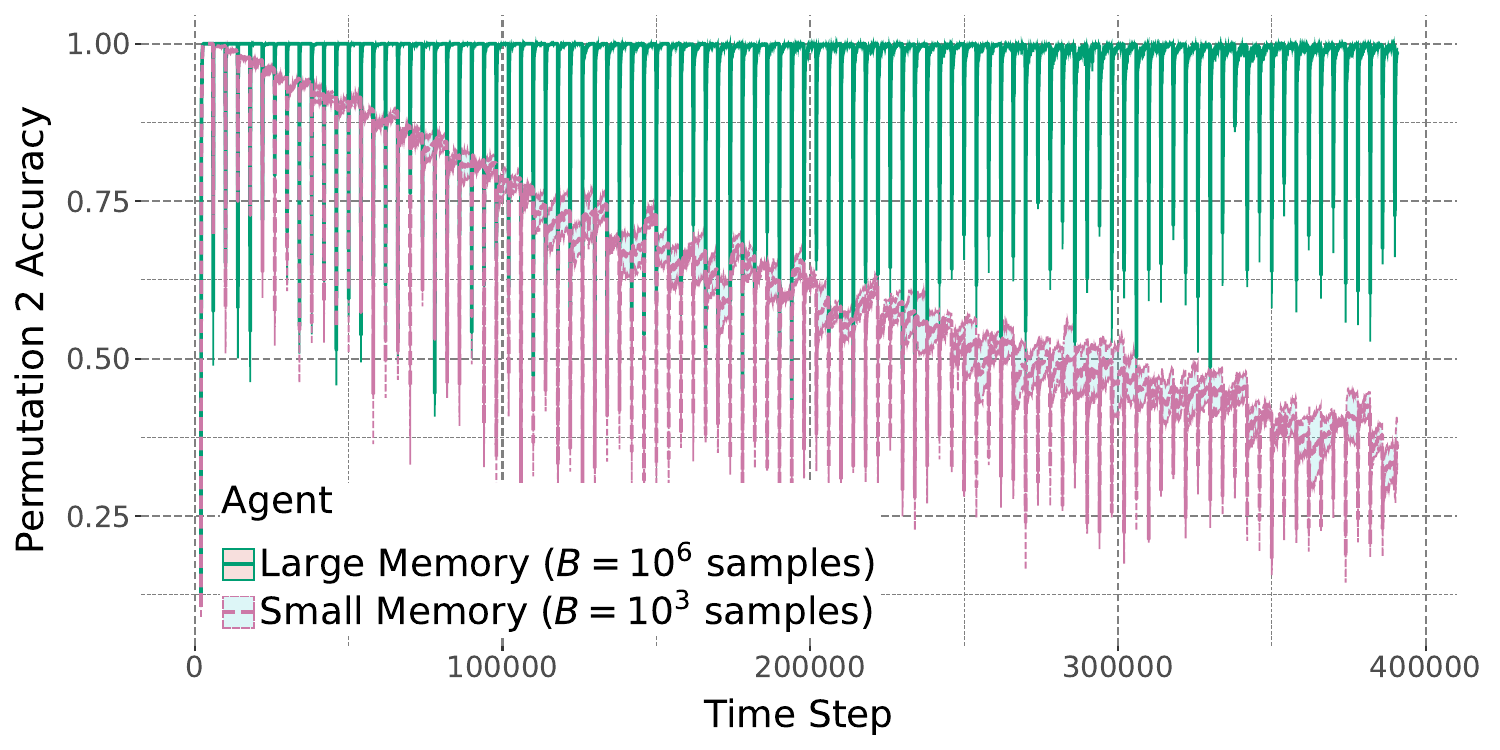}
\caption{Performance on Permutation $2$}
\end{subfigure} 
& \begin{subfigure}{0.48\linewidth}
\includegraphics[width=\linewidth]{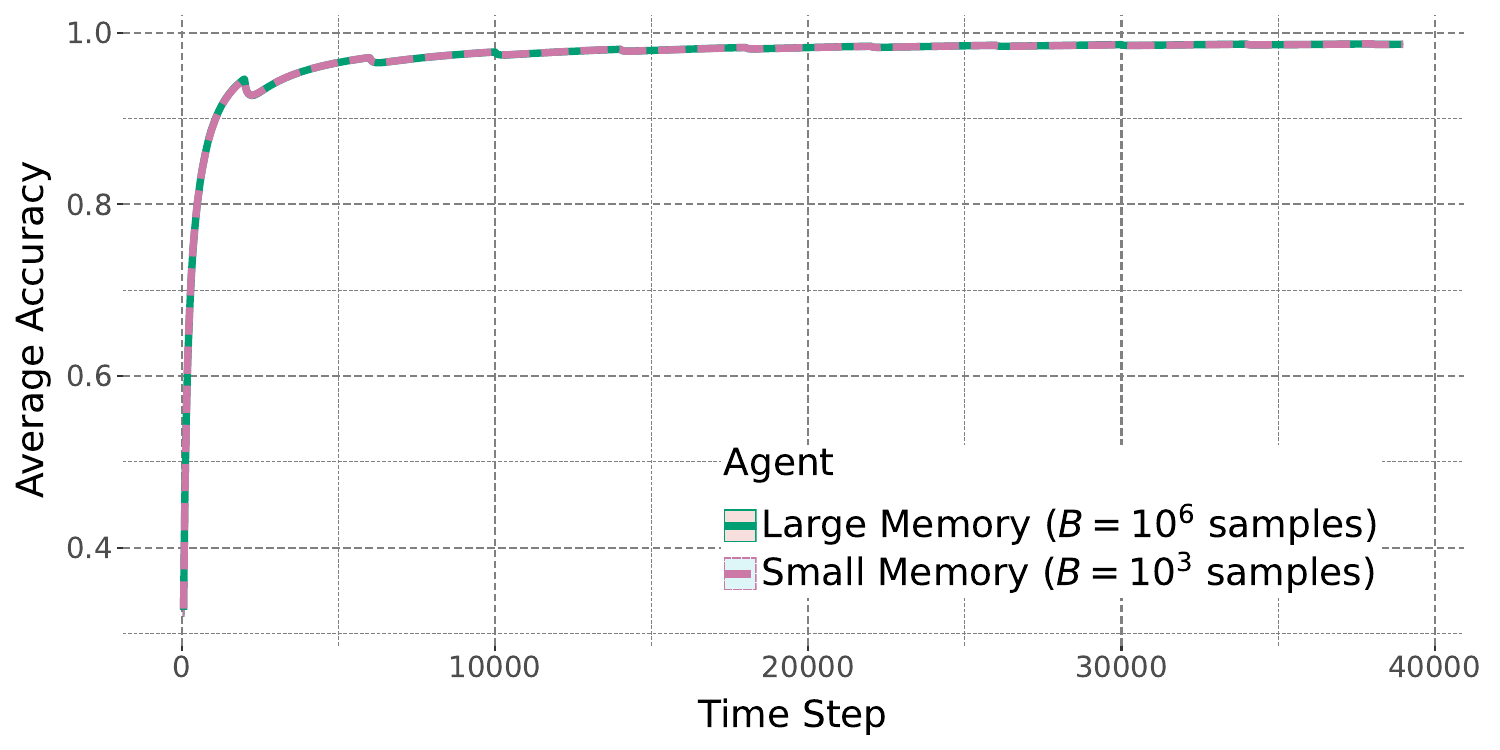}
\caption{Average Accuracy}
\end{subfigure}
\end{tabular}
\caption{(a) Performance on Permutation $2$, which is a non-recurring task. For this plot, the agent's accuracy is computed and averaged across all data points of Permutation $2$. \largememory \space remembers previous information, whereas \smallmemory \space forgets. \henrik{The variability in performance is due to the instability of SGD.} (b) Average accuracy. Despite this difference in forgetting, the two agents perform similarly under our objective. The key reason for this result is that in this environment there is little benefit to remembering non-recurring tasks in order to perform well.}
\label{fig:forgetting-is-ok}
\end{figure}

\textbf{A performant agent can forget non-recurring information.}
We consider a case in which each task has a duration \henrik{of $k=2,000$ timesteps}. This corresponds to the agent seeing each data pair from a task $80$ times \henrik{($2,000 \text{ timesteps} * 16 \text{ samples per timestep} / 400 \text{ unique samples} = 80$)}, which allows the agent to achieve $100 \%$ accuracy on each task. We consider this setting so that we can study the effects of forgetting after an agent completely learns each task.

On this variant of Permuted MNIST, we evaluate the \largememory \space and \smallmemory \space agents (Figure \ref{fig:forgetting-is-ok}). While \smallmemory \space forgets previous non-recurring information quickly relative to \largememory, both agents perform similarly under our objective. This result highlights that a performant agent can forget information that is \textit{non-recurring}. In particular, if a task does not occur more than once and the information required to successfully complete the task is not useful for any future tasks, there is no benefit to remembering this information. Since our version of Permuted MNIST permutes both input pixels and labels, we expect there to be little transfer of information between classifying digits with one permutation versus another. Therefore, it is reasonable to forget how to accurately predict labels for non-recurring permutations. 

While this experiment is simple, the result helps clarify the role of forgetting in continual supervised learning. In the continual learning literature, catastrophic forgetting is highlighted as a critical issue, and performance on previous tasks is the primary metric used in prior work to evaluate methods~\citep{wang2023comprehensive}.  
We argue that when discussing forgetting, it is important to recognize that the usefulness of a (perhaps large) subset of information in real-world applications is transient. Forgetting this information is not catastrophic.

\begin{figure}
\centering
\begin{tabular}{cc}
\begin{subfigure}{0.48\linewidth}
\includegraphics[width=\linewidth]{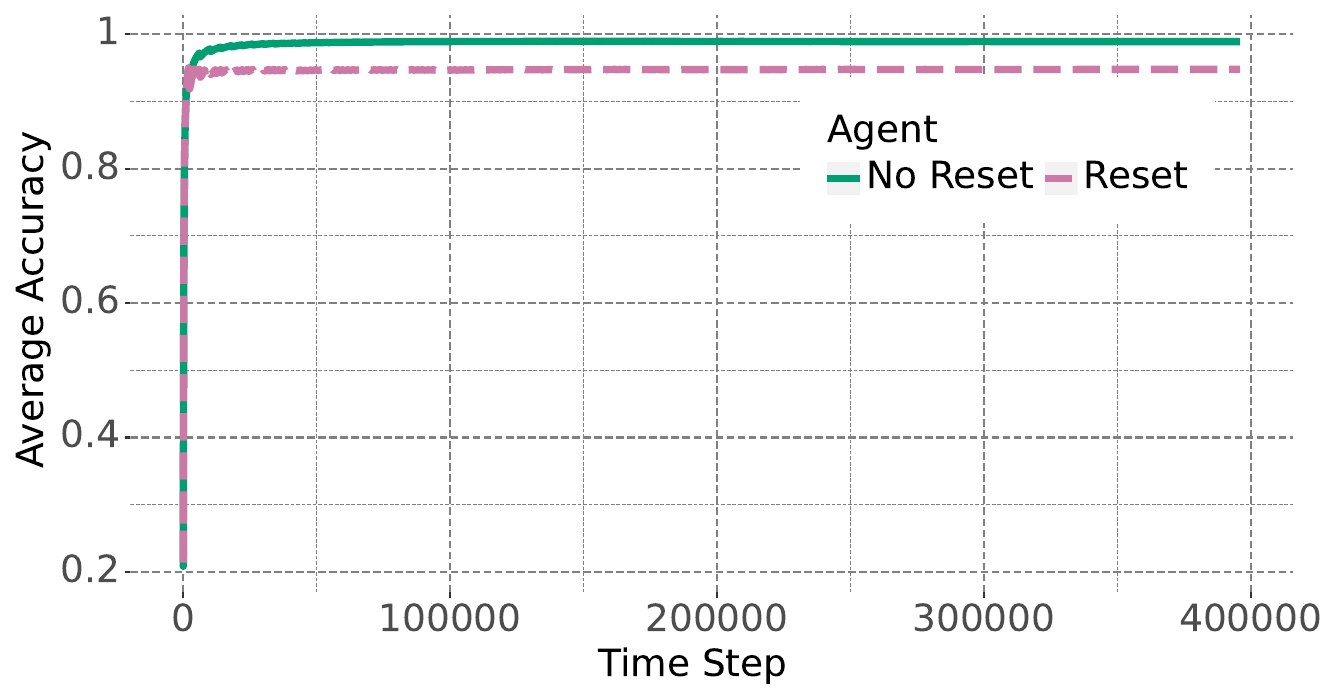}
\caption{Short Task Duration ($2$k timesteps)}
\end{subfigure} 
& \begin{subfigure}{0.48\linewidth}
\includegraphics[width=\linewidth]{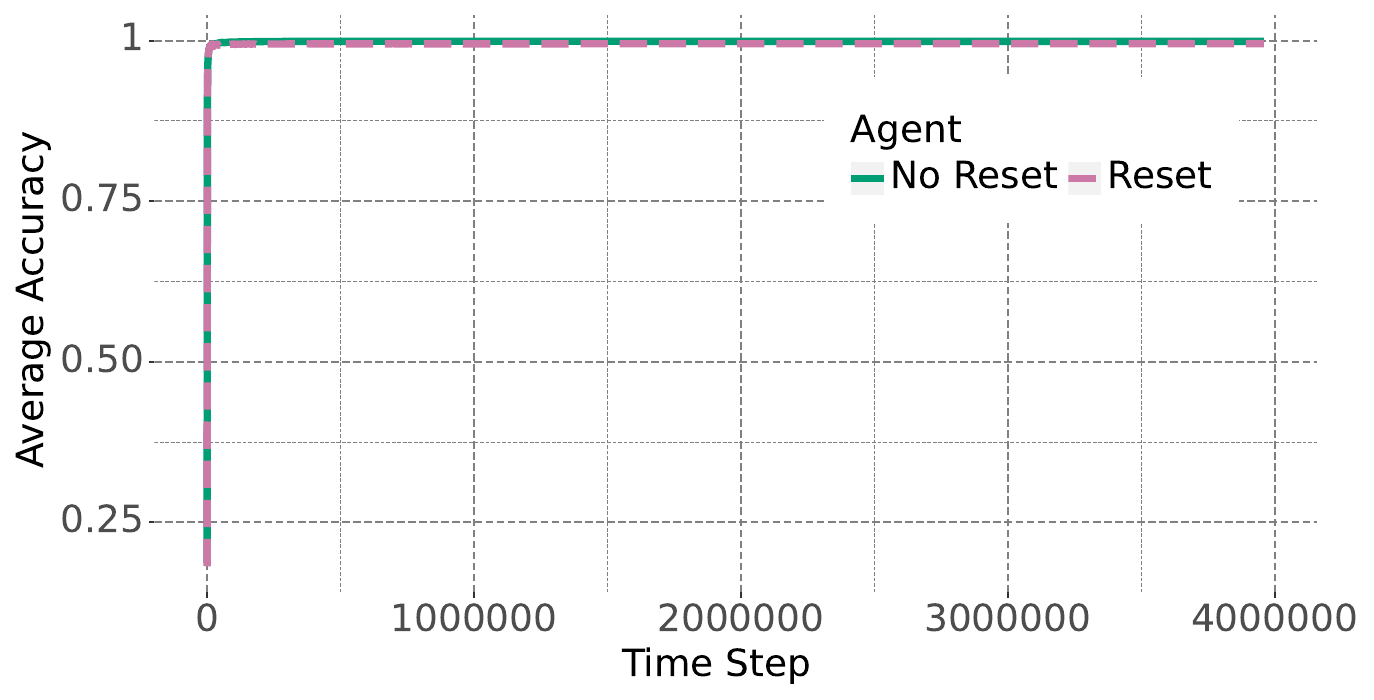}
\caption{Long Task Duration ($20$k timesteps)}
\end{subfigure} 
\end{tabular}
\caption{(a) Average accuracy when each task duration is short. When each task duration is short there is relatively little time to exploit information after it has been learned, and the performance of \reset \space suffers relative to \noreset, which doesn't forget. (b) Average accuracy when each task duration is long. When each task duration is long, the time to learn is short relative to the task duration. Therefore, \reset \space and \noreset \space perform similarly.} 
\label{fig:learning-quickly}
\end{figure}

\textbf{A performant agent can forget recurring information if it relearns that information quickly.} We evaluate the \noreset \space and \reset \space agents with two different permutation durations: $2,000$ timesteps and $20,000$ timesteps (Figure \ref{fig:learning-quickly}). The \reset \space agent forgets all information, both recurring and non-recurring, after each permutation. 
However, when permutation durations are long, its performance under our objective only suffers slightly compared to the performance of \noreset, which remembers recurring information. This is because the \reset \space agent is able to relearn the recurring task quickly relative to the duration of the task. 
This result is in line with \citet{ashley2021does} which highlights that we need to consider different measures when considering forgetting, including retention and the ability to relearn quickly. Further, the average reward objective resolves the extent to which each of these agent characteristics matters. For instance, the duration of the information's utility is an important factor.

\textbf{Under tight computational constraints, forgetting may be helpful.} 
In machine learning, an agent is often parameterized by a neural network, and all parameters of the neural network are updated when taking gradient steps on data pairs. Given this protocol, a computational constraint per timestep, which limits the number of FLOPs when updating the neural network, effectively limits the physical capacity of the neural network the agent can use. Tight computational constraints therefore induce capacity constraints.

To study this setting, we consider reductions in the size of the neural network so that the SGD step of each iteration can be executed within tighter computation budgets.
Concretely, in addition to a hidden size of $1000$, we use smaller hidden sizes of $100$, $25$, and $10$. With these smaller network architectures, we evaluate the \noreset \space and \reset \space agents where each task has a duration of $20,000$ timesteps (Figure \ref{fig:varying-capacity}). We find that as the capacity decreases, \reset \space begins to outperform \noreset. We see that when the hidden size is $10$, the performance of \noreset \space decreases over time as there is insufficient capacity to both remember everything and continue updating on new data. In particular, \noreset \space suffers from loss of plasticity, a characteristic of neural networks that has been studied in recent work~\citep{dohare2021continual,lyle2023understanding,nikishin2023deep}. In contrast, because \reset \space re-initializes the neural network and replay buffer periodically, it retains high plasticity and therefore outperforms \noreset. 

\begin{figure}
\centering
\begin{tabular}{cc}
\begin{subfigure}{0.48\linewidth}
\includegraphics[width=\linewidth]{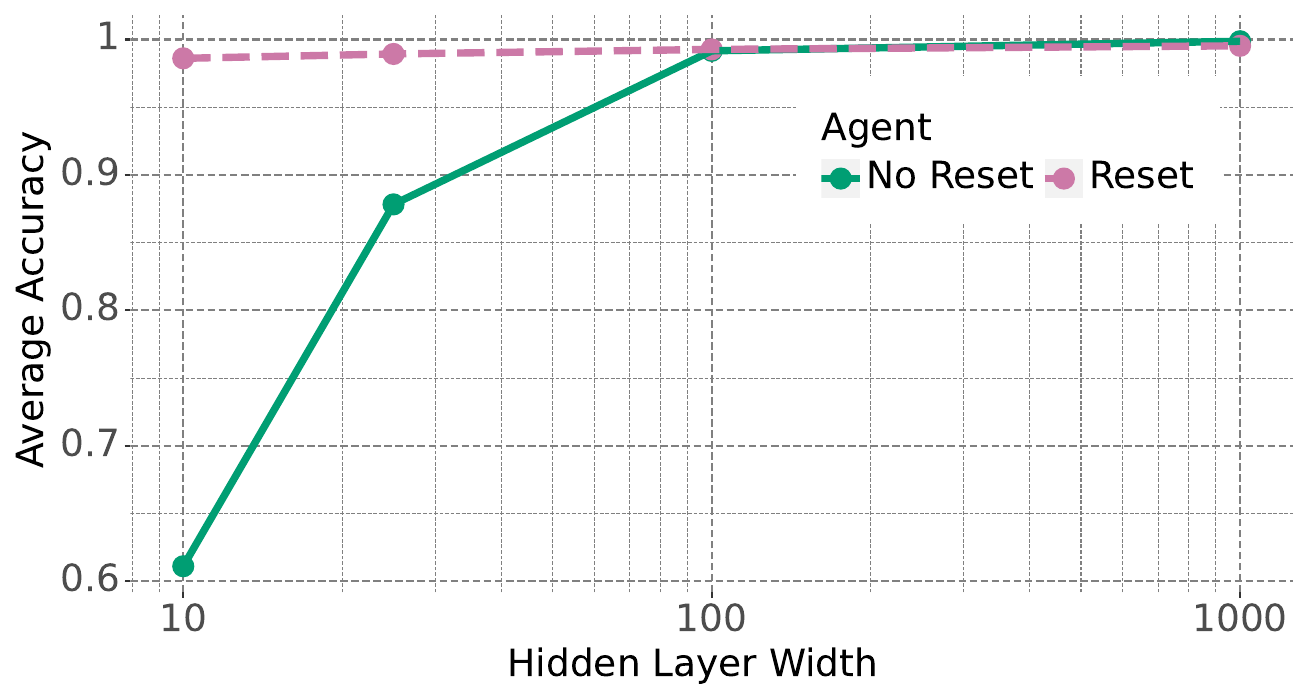}
\caption{Hidden Layer Width vs Average Accuracy}
\end{subfigure} 
& \begin{subfigure}{0.48\linewidth}
\includegraphics[width=\linewidth]{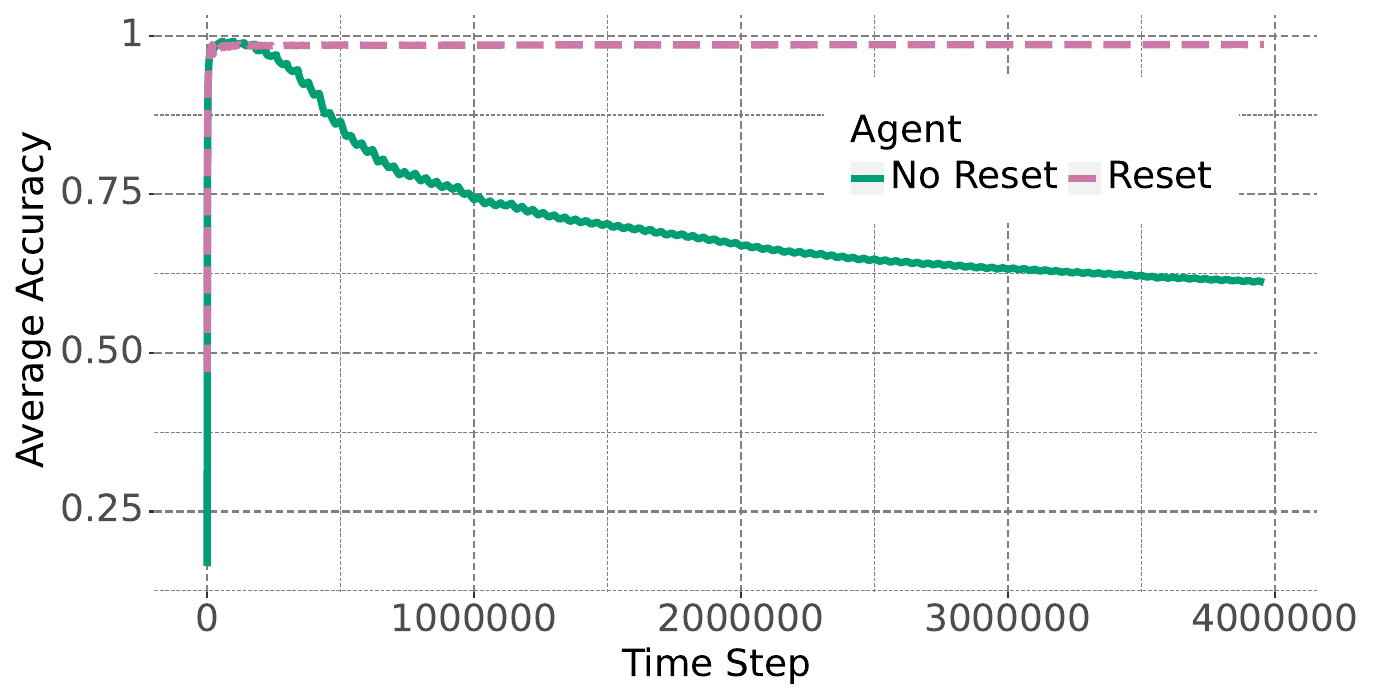}
\caption{Average Accuracy when Hidden Layer Width = 10}
\end{subfigure} 
\end{tabular}
\caption{(a) Average accuracy at the end of evaluation when the neural network has different hidden layer widths. As the agent becomes increasingly capacity constrained (smaller hidden layer width), forgetting becomes more beneficial.  (b) Average accuracy over time when the hidden layer width is $10$. When the agent is severely capacity constrained, resetting prevents loss of plasticity.} 
\label{fig:varying-capacity}
\end{figure}

\subsection{Continual Exploration}\label{sec:continual_exploration}

\henrik{We now turn to environments in which actions may influence future observations. In such environments, active exploration may be required to attain strong performance, as different actions can expose the agent to different information.} As an example, when a recommendation system recommends items to its users, the user behavior that the system observes will vary depending on the recommendations it makes. To improve the quality of recommendations in the long term, a recommendation system may benefit from suggesting a diverse range of items to new users, enabling it to learn more about user preferences.

While seeking out new information can be helpful in the long run, exploration often induces a cost in the short term. Specifically, an agent may sacrifice immediate reward. For instance, a recommendation system may have to recommend multiple different items to a user until it identifies the full range of item types that the user likes. While this is great for the user (and the recommendation system) in the long term, in the short term, the user ends up recommended items they may not like. This trade-off between seeking out information and optimizing immediate reward is commonly known as the exploration-exploitation trade-off. Agents must strike a balance between exploring to seek out information and exploiting existing information to optimize immediate performance.

This problem of balancing exploration and exploitation has primarily been studied in the context of vanishing-regret learning. In this section, we will instead study exploration in the context of continual learning. Specifically, we investigate the implications that nonstationarity has for intelligent exploration. We will argue through didactic examples and simulations that in order to perform well in a nonstationary environment, an agent should (1) continuously engage in exploration, (2) prioritize seeking out information that remains useful over an extended period, and (3) learn about the environment dynamics to guide exploration. \henrik{We also direct readers to the extensive literature on learning in non-stationary bandits~\citep{auer2019adaptively,besbes2014stochastic,bogunovic2016time,chen2019new,chen2023non,luo2018efficient,min2023information,slivkins2008adapting,trovo2020sliding}.\footnote{\henrik{See Chapter 31 "Non-Stationary Bandits" in~\citet{lattimore2020bandit} for a comprehensive review.}}}

\subsubsection*{Exploration in Stationary Environments}
\label{sec:exploration_stationary}
Before considering the implications that nonstationarity has for exploration, let us first consider exploration in a stationary setting.
In a typical stationary environment, the degree to which a performant agent explores typically 
decreases over time. In other words, an agent acquires progressively less information as time goes on. 
For instance, in some stationary environments, the total amount an agent needs to learn to attain optimal or near-optimal average reward is bounded. 
In such an environment, the amount of information acquired per timestep vanishes as time progresses.

To illustrate the trade-off between exploration and exploitation, as well as the typical decrease in exploration over time in a stationary environment, let us examine a special case of the coin tossing example previously discussed in Example~\ref{ex:coins-fixed-biases}. Suppose that the bias $p_1 = 0.8$ of coin $1$ is known, and the prior distribution over the bias of coin $2$ is uniform {\it dyadic} over the set $\{0, 1\}$. Consequently, at each timestep
$t$ before coin $2$ is tossed, the bias $p_2$ is distributed according to a uniform distribution over $\{0, 1\}$; once coin $2$ is tossed, the belief distribution of $p_2$ \henrik{is updated to be equal to the outcome of the second coin's toss.} %
These environment dynamics are characterized by a function $\rho$, defined as follows. For $a = 1$, $\rho(1|h,a) = 0.8$.  For $a = 2$, $\rho(1 | h, a) = p_2$ if coin $2$ was previously tossed according to $h$, and otherwise, $\rho(1 | h, a) = 0.5$.

In this example, an agent can benefit from learning about the bias $p_2$ associated with coin $2$. 
Once the coin is tossed and its bias is revealed, the agent can consistently select the better coin. 
However, the act of exploring and learning about $p_2$ comes at a cost of sacrificing immediate reward; the expected reward of tossing coin $2$ is only $0.5$, which is significantly lower than the expected reward of $0.8$ associated with coin $1$. 
It is worth mentioning that this example also illustrates the typical decrease in exploration over time in a stationary environment. Indeed, this example presents an extreme case where an agent is ``done" with exploring and learning about $p_2$ after the first toss of coin $2$.

\subsubsection*{Exploration in Nonstationary Environments}
In nonstationary environments, the nature of exploration differs from exploration in stationary environments. 
Here, we outline three key implications of nonstationarity on exploration: 
\begin{enumerate}
\item \textbf{Never stop exploring.} In a typical nonstationary environment, new information continually arrives.
Crucially, there is usually a non-diminishing supply of new and valuable information.
As a result, it is common for an optimal agent to engage in continuous exploration to learn about such information. This is in direct contrast to the stationary setting, where agents tend to reduce their exploration over time. 

This idea has been explicitly or implicitly discussed in prior nonstationary bandit learning or nonstationary reinforcement learning literature. For example, many nonstationary learning algorithms are designed to learn about a different latent variable at each timestep, e.g., a different mean reward or a different MDP. 
For this purpose, many nonstationary bandit learning algorithms estimate a mean reward and then adopt a stationary bandit learning algorithm as a subroutine \citep{besbes2019optimal, besson2019generalized, cheung2019learning, garivier2008upper, 9194367, gupta2011thompson, hartland2006multi, kocsis2006discounted, pmlr-v31-mellor13a, raj2017taming, trovo2020sliding, viappiani2013thompson}.

\item \textbf{Seek out durable information.}  
While new information continually arrives in a typical nonstationary environment, it is important to recognize that some information may be transient and loses its relevance over time. 
In order to succeed in a nonstationary environment, an agent must prioritize seeking out information that remains valuable and relevant for a longer duration. We refer to this characteristic as \emph{durability}, which represents the degree to which an agent's acquired information remains useful over time. 
An agent should deprioritize acquiring information that is less durable. 

The concept of information durability was introduced by~\citet{liu2023nonstationary}.  This work also emphasizes the importance of an agent intelligently considering the durability of information when selecting actions through didactic coin tossing games, theoretical results, and simulation experiments.

\item \textbf{Learn about environment dynamics to guide exploration.} 
In order to seek out information that is more durable, an agent needs to determine the durability of information. To achieve this, an agent can benefit from dedicating a portion of its computational budget to learning about aspects of the environment dynamics that determine the durability of information.
\end{enumerate}

\subsubsection*{Coin Replacement Games}
We examine three coin replacement examples to illustrate these three implications that nonstationarity has on exploration. 
These examples are variants of the coin tossing example described in Section 6.2.1. 
Recall that the bias $p_1 = 0.8$ of coin $1$ is known, and the prior distribution over the bias of coin $2$ is uniform dyadic over the set $\{0, 1\}$. 
The key difference is that now in the coin replacement examples, the second coin is replaced at each timestep with probability $q_2$. The coin replacement probability $q_2$ varies across the three examples. 
Note that these games also serve as specific instances of Example~\ref{ex:coins-evolving-biases} in Section~\ref{sec:objective}, where the prior distributions over the bias of each coin are provided in Section 6.2.1.

\textbf{Small replacement probability.} Let us first consider a game where the coin replacement probability $q_2$ for coin $2$ is known to the agent and is small, for instance $q_2 = 0.001$.  In this game, before coin $2$ is tossed for the first time, the expected reward from selecting coin $2$ is $0.5$. 
If the agent has selected coin $2$, and the latest outcome is tails, then the expected reward from selecting coin 2 is $0.0005$. 
In both of these cases, the expected reward from selecting coin 2 is much smaller than the expected reward of $0.8$ associated with coin $1$. 
In addition, the bias of coin $1$ is known, so selecting coin $2$ in this context exemplifies exploration. 

Despite that coin $2$ is associated with a lower expected reward, an optimal agent should eventually select coin $2$, because it is very likely that the coin has eventually been  replaced, possibly with a new coin of bias $1$. 
If the new coin does indeed have a bias of $1$, selecting coin $2$ allows the agent to learn about this bias and then continue selecting the same coin, resulting in a reward of $1$ for a long time---an average of $1000$ consecutive timesteps. 
Since selecting coin $2$ exemplifies exploration, this game serves as an illustration that an optimal agent may need to continuously explore, unlike in stationary environments.

\textbf{Large replacement probability.}  Next, suppose that the coin replacement probability $q_2$ for coin $2$ is instead large, say, $q_2 = 0.999$.
Because coin $2$ is likely to be replaced at each timestep, the information associated with it quickly becomes obsolete. In other words, the information is not very durable.  
Therefore, unlike in the previous game, an agent does not benefit from learning about the bias of coin $2$ anymore in this game. 
Indeed, an optimal agent will only ever select coin $1$ throughout the entire game. 
This variation highlights the importance of only seeking out information to the extent that the information is durable.

\textbf{Unknown Replacement Probability}
Now suppose that the coin replacement probability $q_2$ for coin $2$ is unknown. 
This presents a typical scenario where an agent does not know the environment dynamics \emph{a priori}. 
Recalling the previous two variations of the coin replacement game, we observe that the optimal behaviors differ significantly based on this coin replacement probability. This indicates that understanding and learning about the coin replacement probability is crucial for determining the durability of information and selecting actions accordingly.
This variation of the game emphasizes the importance of learning about the dynamics of the environment in order to guide exploration. By acquiring knowledge about the coin replacement probability, an agent can make informed decisions on how to explore and seek out durable information in a nonstationary environment.

\subsubsection*{Experiments in AR(1) Bandits}
While the coin tossing games serve as a model of environments with abrupt changes and bounded rewards, 
our insights on continual exploration extend beyond such settings. To demonstrate this, we replicate a variation of experiments from~\citet{liu2023nonstationary}. In particular, we conduct experiments on a class of Gaussian bandits that capture continuous or smooth changes in environments with unbounded rewards. These bandits are known as AR(1) Gaussian bandits, and Example~\ref{ex:nonstationary-Gaussian-TS} in Appendix~\ref{app:cl_agents} serves as one specific instance of an AR(1) Gaussian bandit. The AR(1) bandits or similarly constructed nonstationary Gaussian bandits have been studied by \cite{gupta2011thompson, gaussianar1, kuhn2015wireless, liu2023nonstationary, slivkins2008adapting}.

\subsubsection*{Environment}

We consider a two-armed Gaussian bandit described in Example~\ref{ex:nonstationary-Gaussian-TS}. Recall that the environment is characterized by latent random variables 
$\theta_{0,a}$ independently and identically distributed $\mathcal{N}(\mu_{0,a}, \Sigma_{0,a})$,
and 
updated according to 
\begin{align*}
\theta_{t+1,a} = \eta \theta_{t,a} + Z_{t+1,a},
\end{align*}
with each $Z_{t+1,a}$ independently sampled from $\mathcal{N}(0,\zeta^2)$. 
Each reward $R_{t+1} = O_{t+1} = \theta_{t, A_t} + W_{t+1, A_t}$, where each $W_{t+1,a}$ is sampled independently from $\mathcal{N}(0,\sigma^2)$. 
The parameters of the environment include the prior mean $\mu_{0, a} \in \mathbb{R}$,  the prior variance $\Sigma_{0,a} \in \mathbb{R}_+$, the AR(1) parameter $\eta \in \mathbb{R}$, the $Z_{t,a}$ variance $\zeta^2 \in \mathbb{R}_+$, and the observation noise variance $\sigma^2 \in \mathbb{R}_+$.

In our experiments, we let 
$\mu_{0,a} = 0$, 
$\Sigma_{0,a} = 1$, 
$\sigma = 1$, and $\zeta$ be such that each sequence $(\theta_{t,a}: t \in \mathbb{N})$ is a stationary stochastic process.
The AR(1) parameter $\eta$ is the remaining parameter that we vary across experiments. 
It determines the degree to which information about $\theta_{t,a}$ is durable; if $\eta = 1$, then the information about $\theta_{t,a}$ remains useful forever, and if $\eta = 0$, then the information immediately lose its relevance at the next timestep.

\subsubsection*{Agents}
We consider two agents: Thompson sampling \citep{thompson1933likelihood}, which does not take into account the durability of information when selecting actions, and predictive sampling \citep{liu2023nonstationary}, which does. Both agents have privileged access to the environment parameters. 

\textbf{Thompson sampling.} First, we consider a Thompson sampling agent. This agent is representative of those which do not account for the durability of information. 
The agent maintains a posterior distribution of $\theta_{t,a}$ for each action $a$, parameterized by $\mu_{t,a}$ and $\Sigma_{t,a}$. \henrik{Further details on these parameters are in Example~\ref{ex:nonstationary-Gaussian-TS} in Appendix~\ref{app:cl_agents}.}

Thompson sampling updates these parameters according to the following equations: 
\begin{align*}
\mu_{t+1,a} = \left\{\begin{array}{ll}
\eta \mu_{t,a} + \alpha_{t+1} (O_{t+1} - \eta \mu_{t,a}) & \mathrm{if\ } a = A_t \\
\eta \mu_{t,a} & \mathrm{otherwise,}
\end{array}\right.
\qquad
\Sigma_{t+1,a} = \left\{\begin{array}{ll}
\frac{1}{\frac{1}{\eta^2 \Sigma_{t,a} + \zeta^2} + \frac{1}{\sigma^2}} &  \mathrm{if\ } a = A_t \\
\eta^2 \Sigma_{t,a} + \zeta^2 & \mathrm{otherwise,}
\end{array}\right.
\end{align*}
where 
$\alpha_{t+1} = \Sigma_{t+1, A_t}/\sigma^2$. 
The agent then estimates mean rewards and selects the action corresponding to the largest estimate by 
sampling each $\hat{\theta}_{t,a}$ independently from $\mathcal{N}(\mu_{t,a}, \Sigma_{t,a})$, and selecting an action uniformly at random from the set $\arg \max_{a \in \mathcal{A}}\hat{\theta}_{t,a}$.

\textbf{Predictive sampling.} The other agent that we consider is a predictive sampling agent. Predictive sampling can be viewed as a modified version of Thompson sampling that de-prioritizes transient information. 
In an AR(1) bandit, specifically, similar to the Thompson sampling agent, this agent maintains the same set of hyperparameters, and also updates parameters $\mu_{t+1, a}$ and $\Sigma_{t+1,a}$ according the equations above. 
The parameters determine the posterior belief of $\theta_{t,a}$. 
Different from Thompson sampling, the agent then samples each  $\hat{\theta}_{t,a}$ independently from a different distribution, i.e., $\mathcal{N}(\tilde{\mu}_{t,a}, \tilde{\Sigma}_{t,a})$, and selects an action uniformly from the set $\arg \max_{a \in \mathcal{A}}\hat{\theta}_{t,a}$. 

Specifically, predictive sampling updates the parameters  $\tilde{\mu}_{t, a}$ and $\tilde{\Sigma}_{t,a}$ 
of its  
sampling distribution as follows: 
\begin{align*}
&\ \tilde{\mu}_{t,a} = \mu_{t, a} \text{ and } 
\tilde{\Sigma}_{t,a} = \frac{\eta_a^2 \Sigma_{t, a}^2}{\eta_a^2 \Sigma_{t, a} + x_a^*},
\end{align*}
where 
$x_a^* = \frac{1}{2} \left(\zeta_a^2 + \sigma^2 - \eta_a^2 \sigma^2 + \sqrt{(\zeta_a^2 + \sigma^2 - \eta_a^2 \sigma^2)^2 + 4 \eta_a^2 \zeta_a^2 \sigma^2}\right)$. 

The update expression for $\tilde{\Sigma}_{t,a}$ is quite involved, so we provide some intuition for how predictive sampling behaves relative to Thompson sampling. 
We first consider the case where $\eta = 1$. This corresponds to a stationary environment. 
In this setting, predictive sampling executes the same policy as Thompson sampling because $\tilde{\Sigma}_{t,a} = \Sigma_{t,a}$. 

On the other extreme, if $\eta = 0$, $\theta_{t,a}$ is completely determined by the process noise. 
This corresponds to an environment where information about $\theta_{t,a}$ is completely not durable. 
In this setting, predictive sampling's sampling variance $\tilde{\Sigma}_{t,a}$ for each arm $a$ is $0$, and this agent therefore executes a greedy policy. 

More generally, the predictive sampling agent's sampling variance $\tilde{\Sigma}_{t,a}$ lies between $0$ and $\Sigma_{t,a}$, i.e., 
$0 \leq \tilde{\Sigma}_{t,a} \leq \Sigma_{t,a}$. 
In addition, the ratio $\tilde{\Sigma}_{t,a} / \Sigma_{t,a}$ is monotonically increasing in $\eta$.  
Therefore, predictive sampling samples actions from distributions with smaller variances as compared to that of Thompson sampling.
This corresponds to more exploitation versus exploration. 
In summary, in an AR(1) bandit, predictive sampling can be viewed as a variant of Thompson sampling that adjusts how it balances exploration and exploitation according to $\eta$, which determines the durability of information. 
The behavior of PS is motivated by the fact that when information about $\theta_{t,a}$ is less durable, meaning that $\theta_{t,a}$ is less informative in predicting $(\theta_{k,a}: k \geq t+1)$, an agent should deprioritize acquiring information about $\theta_{t,a}$.

\subsubsection*{Results}
We conduct two experiments to investigate the impact of information durability on agent performance. 
In the first experiment, we compare the performance of predictive sampling and Thompson sampling over time. The results demonstrate that predictive sampling 
consistently outperforms Thompson sampling. Furthermore, we observe that predictive sampling tends to take greedy actions more frequently; here, greedy actions refer to actions with higher mean reward estimates. This finding supports our hypothesis that de-prioritizing transient information
leads to improved performance.

In the second experiment, we examine the performance gap between the two agents in environments with varying levels of information durability. The results reveal that the performance gap between the two agents is more significant in environments where information is less durable. This suggests that the choice of exploration strategy becomes even more critical when information is less durable. 

Recall that the environments we examine are parameterized by $\eta$, which determines the degree to which information about $\theta_{t,a}$ is durable. In such environments, the information durability associated with different actions are comparable, and deprioritizing transient information translates to deprioritizing exploration relative to exploitation as information durability decreases.  

\textbf{De-prioritizing transient information is beneficial.} Figure~\ref{fig:ar1_1}(a) plots the action selection frequencies against time, and the band corresponds to a 95\% confidence band. This figure shows that predictive sampling consistently selects greedy actions more than Thompson sampling.
Figure~\ref{fig:ar1_1}(b) plots the average reward collected by each agent against time ($t = 1, ... 200$). This figure shows that predictive sampling consistently outperforms Thompson sampling across time. The two plots suggest that de-prioritizing transient information, 
which in this environment corresponds to exploiting more by taking more greedy actions, leads to better performance in nonstationary environments.

\textbf{The benefit is larger in environments with less durable information.} 
Figure~\ref{fig:ar1_2}(a) plots the action selection frequency of each agent over $200$ timesteps against the AR(1) parameter. The figure shows that the gap between the action selection frequencies also increases as the AR(1) parameter decreases, i.e., as information becomes less durable.
Figure~\ref{fig:ar1_2}(b) plots the average reward collected by each agent over $200$ timesteps against the AR(1) parameter. The figure shows that the performance gap between the agents increases as the AR(1) parameter decreases, i.e., as information becomes less durable. 
The two plots suggest that the performance gain from deprioritizing less durable information is larger in environments where information is less durable.

These experiments provide valuable insights into the relationship between information durability and agent performance, further reinforcing the importance of intelligent exploration in nonstationary environments.

\begin{figure}[!ht]
    \centering
    \begin{tabular}{cc}
        \begin{subfigure}{0.42\linewidth}
        \includegraphics[width=\linewidth]{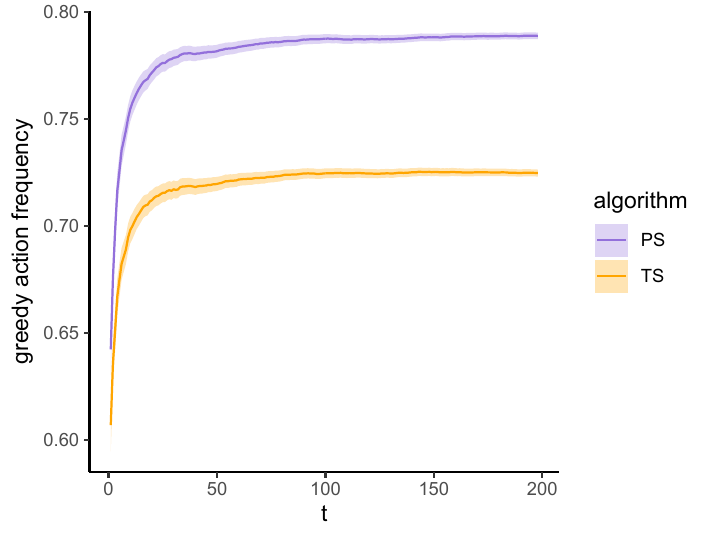}
        \caption{Greedy Action Selection Frequencies}
    \end{subfigure}
    \hfill %
  \begin{subfigure}{0.42\linewidth}
    \includegraphics[width=\linewidth]{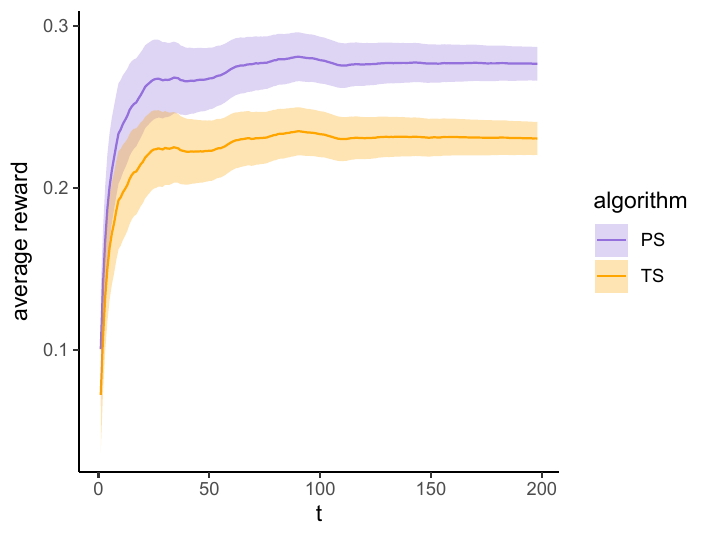}
        \caption{Average Rewards}
    \end{subfigure} 
    \end{tabular}
    \caption{Experiment where $\eta = 0.9$, and $t \in \{1, 2, ... , 200\}$: (a) The frequencies at which Thompson sampling (TS) and predictive sampling (PS) agents select greedy actions. PS consistently selects greedy actions more than TS throughout time. (b) Average reward collected by TS and PS agents. PS consistently attains higher average reward than TS throughout time.} %
    \label{fig:ar1_1}
\end{figure}

\begin{figure}[!ht]
    \centering
    \begin{tabular}{cc}
    \begin{subfigure}{0.4\linewidth}
    \includegraphics[width=\linewidth]{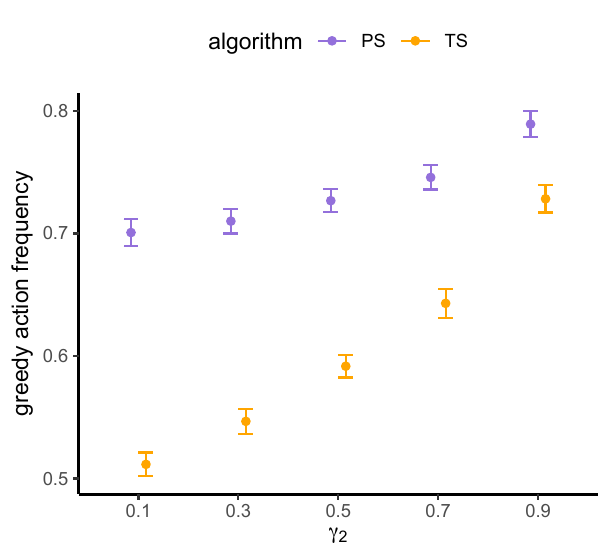}
    \caption{Greedy Action Selection Frequencies}
    \end{subfigure}
    \hfill
    & \begin{subfigure}{0.4\linewidth}
    \includegraphics[width=\linewidth]{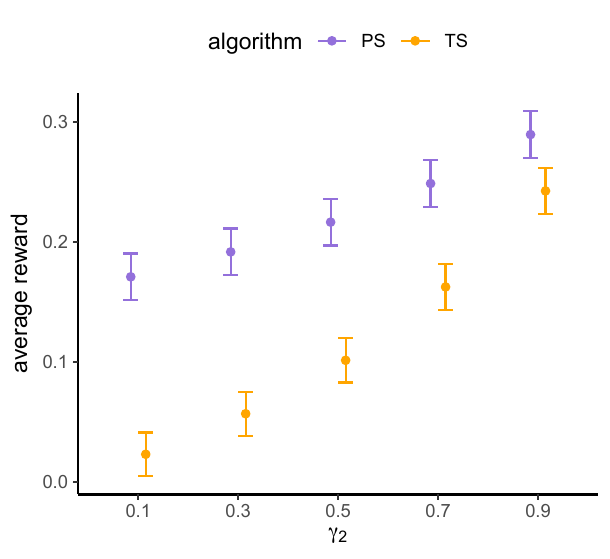}
    \caption{Average Rewards}
    \end{subfigure} 
    \end{tabular}
    \caption{Experiment where $\eta \in \{0.1, 0.3, 0.5, 0.7, 0.9\}$, and $t = 200$: (a) The frequencies at which Thompson sampling (TS) and predictive sampling (PS) agents select greedy actions. PS consistently selects greedy actions more than TS for varying $\eta$, and the gap between the greedy action selection frequencies increases as $\eta$ decreases. (b) Average reward collected by TS and PS agents. PS consistently attains average reward higher than TS for varying $\eta$, and the gap between the average rewards increases as $\eta$ decreases. 
    }
    \label{fig:ar1_2}
\end{figure}

\subsection{Continual Learning from Delayed Consequences} \label{sec:delayed-consequences}

So far, we have considered supervised learning and bandit learning, both of which involve no delayed consequences. 
In contrast, in this section we will consider the case in which actions have delayed consequences. Common examples of this include robotics, game playing (Go, chess, Hanabi, etc), and of course, regular life. In these cases, the agent's actions affect which situations it ends up in, and the reward is often delayed. 
While our framing of ``reinforcement learning" is very broad, we should note that it is precisely this problem formulation with delayed consequences that typically goes under the reinforcement learning rubric. 

Most commonly, this problem is framed as learning in a Markov Decision Process (MDP). Typically, the MDP is stationary and not changing over time. To study continual learning, we will instead consider an MDP that is changing over time. In summary, we consider learning from delayed consequences in a nonstationary MDP as a special case of our computationally constrained RL framework for continual learning. This type of environment has been extensively studied in prior work, surveyed by \citet{padakandla2021survey}.

\subsubsection*{Environment}

Consider an environment $\mathcal{E} = (\aspace, \mathcal{O}, \rho)$ in which observations are generated by an MDP with time-varying transition probabilities.  Each observation $O_t$ is of the current state $S_t$ of the MDP.  Hence, the state space of the MDP is $\mathcal{S} = \mathcal{O}$. We consider a small MDP where $|\mathcal{S}| = 10$ and $|\mathcal{A}| = 3$.

For each state-action pair $(s,a)$, the vector $(P_{t,s,a,s'}:s' \in \sspace)$ represents the transition probabilities at timestep $t$. The initial transition probability vectors are independently sampled from a $\text{Dirichlet}(1/\mathcal{S}, \ldots, 1/\mathcal{S})$ distribution. At every timestep $t$, the MDP is updated with some small probability. Specifically, for each state-action pair $(s,a)$, with some probability $\eta \in (0,1)$, the vector $(P_{t,s,a,s'}:s' \in \sspace)$ is replaced by a new, independent $\text{Dirichlet}(1/\mathcal{S}, \ldots, 1/\mathcal{S})$ sample.

After the action $A_t$ is executed, the environment produces as an observation the next state $S_{t+1} = O_{t+1}$.  This state is generated as though the environment samples from the state distribution $P_{t, S_t, A_t, \cdot}$.  Note that this probability vector is random, and the observation probability function $\rho$ that characterizes environment dynamics is defined by 
$$\rho(o | H_t, A_t) = \E[P_{t,S_t,A_t,o} | H_t, A_t]$$ 
for all $o \in \mathcal{O}$.

There is a known distinguished state $s_* \in \mathcal{S}$, which we will refer to as {\it the goal state}. A nonzero reward is received only upon arrival at the goal state:
$$R_{t+1} = r(H_t, A_t, O_{t+1}) 
= \left\{\begin{array}{ll}
r_{t+1} > 0 \qquad & \text{if } S_{t+1} = s_* \\
0 \qquad & \text{otherwise.}
\end{array}\right.
$$

Because the MDP is changing over time, we can think of the agent as encountering different stationary MDPs (determined by $P_t$) throughout time. We scale the reward $r_{t+1}$ such that the optimal long-term average reward in each MDP is the same ($.5$). This ensures low variance across time and seeds. See details in Appendix \ref{apdx:continual-delayed}. %

\subsubsection*{Agent}

We consider a version of Q-learning called optimistic Q-learning. The agent is initialized with $Q_0(s,a) = 0$ for all $(s,a)$ and updates action values according to
$$Q_{t+1}(s,a) = \left\{\begin{array}{ll}
Q_t(S_t,A_t) + \alpha \left(R_{t+1} + \gamma \max_{a \in \aspace} Q_t(S_{t+1}, a) - Q_t(S_t, A_t)\right) + \zeta \qquad & \text{if } (s,a) = (S_t,A_t) \\
Q_t(s,a) + \zeta & \text{otherwise.}
\end{array}\right.
$$

There are three hyperparameters: the stepsize $\alpha$, the discount factor $\gamma$, and the optimistic boost $\zeta$. The formula is identical to the regular Q-learning formula, with one addition: all action values are incremented with an optimistic boost $\zeta$ at each timestep. This ensures that all actions will eventually be taken upon a revisit to a state. A higher $\zeta$ can be interpreted as leading to more exploration since it leads to all actions being revisited more often.

At each timestep, the agent acts greedily with respect to $Q_t$ and selects the action with the highest action value. If multiple actions have the same action value the agent samples uniformly from those. Thus, $Q_t$ induces a mapping from situational state to action. \henrik{For additional details on optimistic Q-learning, see Example~\ref{ex:optimistic-q-learning} in Appendix~\ref{app:delayed_consequences}.}

In the stationary setting with a constant MDP, there is an optimal action-value function $Q_t^*$ that does not change over time. Therefore, by annealing the stepsize and optimistic boost appropriately, the agent can converge on the optimal action-value function. The difference $|Q_t - Q_t^*|$ vanishes. By acting greedily with respect to the optimal action-value function, the agent can perform optimally. In contrast, in our environment, the MDP is changing at a constant rate, and so is the associated optimal action-value function $Q_t^*$. In this case, annealing the stepsize and optimistic boost will result in premature convergence and hurt performance. By annealing the stepsize and optimistic boost, it will eventually update $Q_t$ at a much slower rate than the optimal action-value function $Q_t^*$ is changing. Indeed, the agent will update at a slower and slower rate, such that it effectively stops learning information that is useful: by the time the agent has made even a single update to $Q_t$, the optimal action-value function has changed many times. Because the objective is to maximize the infinite-horizon average reward, any performance up until a finite time will have no influence on the overall long-term average reward. Thus, since an annealing agent will effectively stop learning after a certain point in time, it will underperform under the long-term average reward objective. This is generally true in the continual learning setting where the optimal mapping from situational state to action keeps changing indefinitely.

Consequently, in our experiments, we only consider constant $\alpha$ and $\zeta$ values. In reinforcement learning, the discount factor $\gamma$ is typically presented as a component of the MDP. In our framing of continual learning, $\gamma$ is instead a hyperparameter of the agent that controls the effective planning horizon. For all experiments, we set $\gamma = 0.9$.

\subsubsection*{Results}

\begin{figure}[htbp]
\centering
\begin{tabular}{cc}
\begin{subfigure}{0.47\linewidth}
\includegraphics[width=\linewidth]{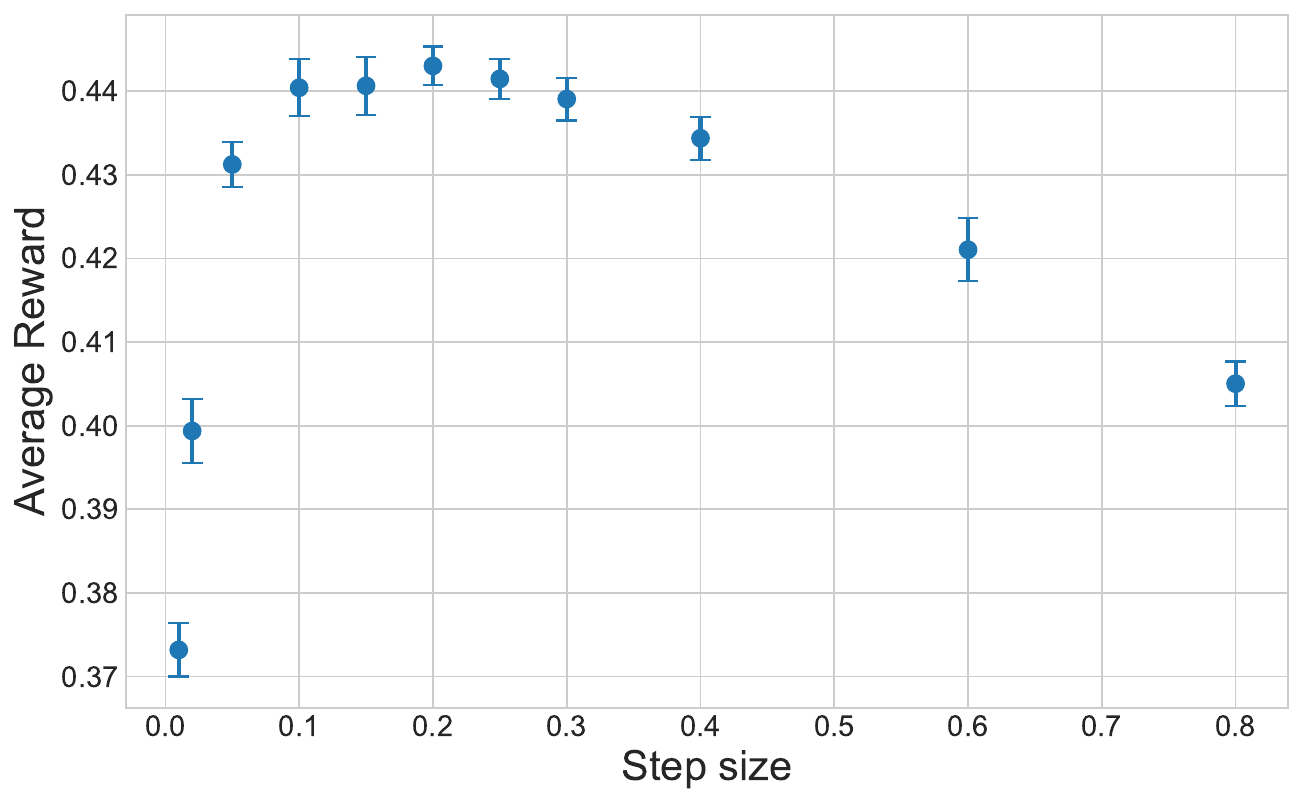}
\caption{$\eta=1e-4$}
\end{subfigure} 
& \begin{subfigure}{0.47\linewidth}
\includegraphics[width=\linewidth]{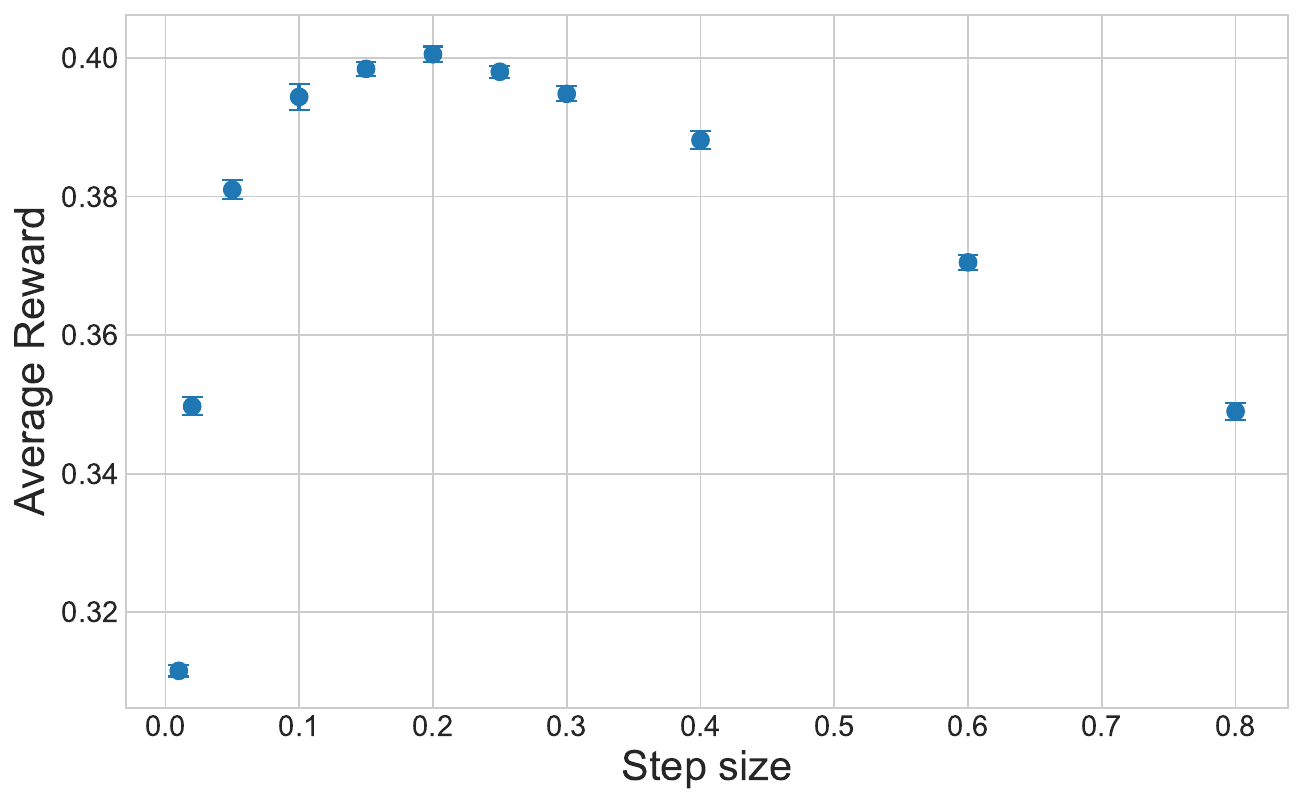}
\caption{$\eta=1e-3$}
\end{subfigure}
\end{tabular}
\caption{Average reward versus the stepsize $\alpha$. Interestingly, optimal stepsize is the same ($.2$) in both environments.
}
\label{fig:delayed-reward-vs-stepsize}
\end{figure}

\begin{figure}[htbp]
\centering
\begin{tabular}{cc}
\begin{subfigure}{0.48\linewidth}
\includegraphics[width=\linewidth]{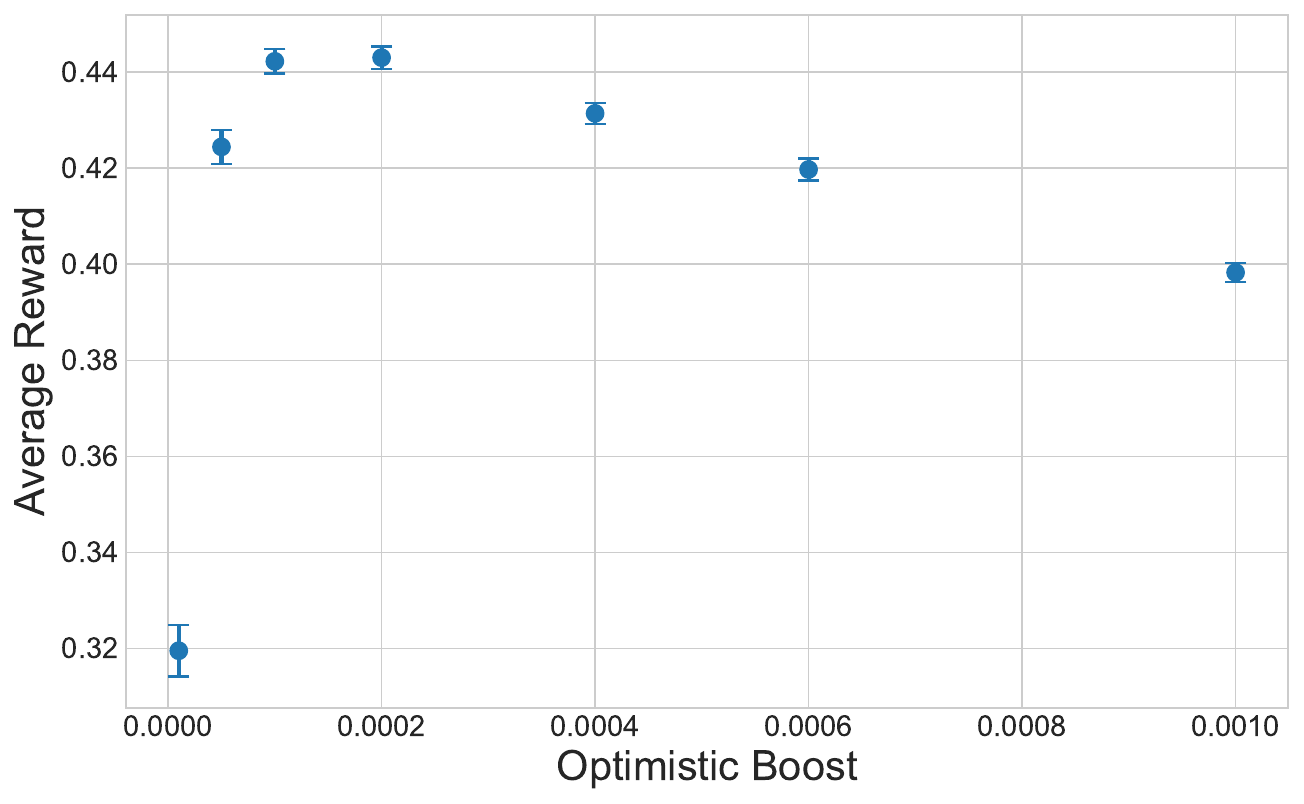}
\caption{$\eta=1e-4$}
\end{subfigure} 
& \begin{subfigure}{0.48\linewidth}
\includegraphics[width=\linewidth]{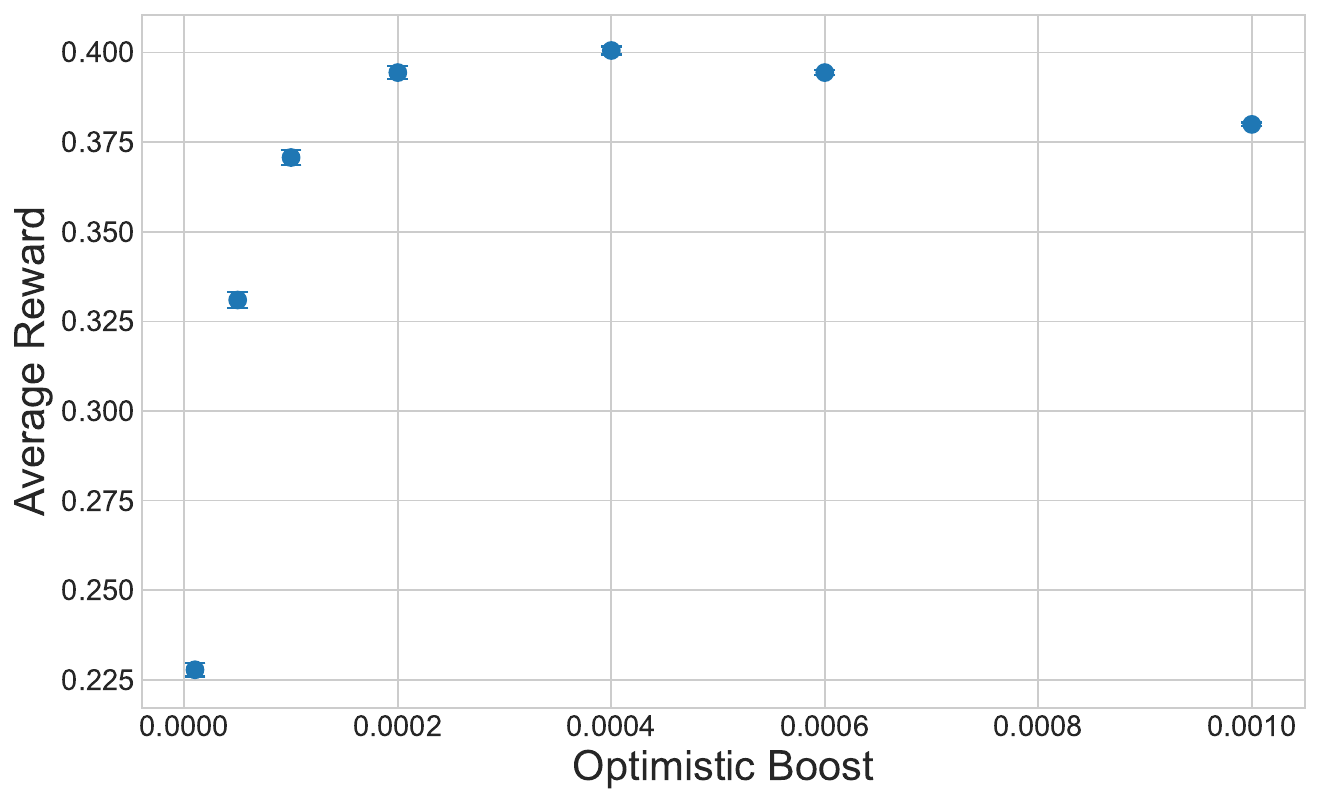}
\caption{$\eta=1e-3$}
\end{subfigure} 

\end{tabular}
\caption{Average reward versus the optimistic boost $\zeta$. As the nonstationarity parameter $\eta$ increases, the optimal optimistic boost $\zeta^*$ increases.
}
\label{fig:delayed-reward-vs-boost}
\end{figure}

On each environment, we perform a sweep over stepsize $\alpha$  and optimistic boost $\zeta$  to select the optimal values. On each environment, we plot how the average reward depends on the stepsize $\alpha$ and optimistic boost $\zeta$, respectively. For each value of $\alpha$, in Figure~\ref{fig:delayed-reward-vs-stepsize}, we plot the largest average reward achievable when sweeping over $\zeta$. Similarly, in Figure~\ref{fig:delayed-reward-vs-boost}, we plot the largest average reward achievable for each $\zeta$. We find that the optimal optimistic boost increases with the degree of nonstationarity in the environment. These results are intuitive. When the environment is changing, old knowledge becomes obsolete, and it is imperative to seek out new knowledge. Higher values of optimistic boost lead to more exploratory behavior needed to accomplish this. In contrast, however, we note that it is a little surprising that the optimal stepsize was the same in both environments.

In more complex environments, it is cumbersome or infeasible to perform a hyperparameter search to find optimal values. We conjecture that a more sophisticated agent, without being initialized with the optimal hyperparameters, could automatically learn them over time, fully online. This is an example of meta-learning which has been explored extensively in previous literature (see e.g \citet{duan2016rl,flennerhag2021bootstrapped,thrun1998learning}). Because our objective considers an infinite horizon, the sophisticated agent would therefore be able to reach the same asymptotic performance as an agent that is initialized with the optimal hyperparameters.

\FloatBarrier

\subsection{Continual Auxiliary Learning}\label{sec:continual_auxiliary_learning}
While an agent's goal is to maximize average reward, a complex environment can offer enormous amounts of feedback beyond reward.  This feedback can be used to accelerate the agent's learning and thus increase its average reward.  For example, sensors of a self-driving car ingest visual and auditory feedback far beyond what is required to determine reward, which, for example, might simply indicate safe arrival to a destination.  By predicting future trajectories of its own vehicle, learning from realizations to improve these predictions, and using these predictions to plan, an agent can learn to maximize reward much more quickly than if it learned only from reward feedback.

Auxiliary tasks are those that are distinct from, though possibly helpful to, the primary task of maximizing average reward.  Prediction of future outcomes other than reward, as we considered in our self-driving car example, serves as an example of an auxiliary task.  Learning to perform auxiliary tasks can accelerate learning for the primary task.

In a complex environment, it is often unclear which auxiliary tasks are helpful and how they relate to the primary task.  However, a long-lived agent can learn this over time.  We refer to this as {\it continual auxiliary learning}.  In the remainder of this section,  we illustrate benefits of {\it continual auxiliary learning} through the following didactic example:
\begin{example}
\label{ex:continual-aux-learning}
{\bf (continual auxiliary learning)}
Consider a modified version of continual SL where the input set is the singleton $\mathcal{X}=\{\emptyset\}$, and the label set consists of $k$-dimensional binary vectors $\mathcal{Y}=\{0,1\}^k$.
The labels are generated via
$$Y_{t+1} \sim \sigma(\phi_t)  ,$$
where $\sigma$ denotes the sigmoid function $\exp(x)/(1+\exp(x))$ applied element-wise.
We assume that $\phi_t = A \theta_t$, where
$A$ is a sparse vector with only $K$ non-zero components (including the first),
and $\theta_t\in \mathbb{R}$ evolves according to the AR(1) process
$$\theta_{t} = \eta \theta_{t-1} + W_t,$$
where $W_t \overset{iid}{\sim} \mathcal{N}(0, 1-\eta^2)$ and $W_t \perp \theta_{t-1}$.
The actions of the agent are predictions about the first component $Y_{t+1,1}$ of the label.  Hence, $\mathcal{A}=\Delta_{\{0, 1\}}$.  We take reward to be the negative log-loss:
$$r(H_t, A_t, O_{t+1}) = \ln P_t(Y_{t+1,1}).$$ 
\end{example}
In this example, prediction of $Y_{t+1,1}$ constitutes the primary task.  Prediction of the remaining label components $Y_{t+1,2:k}$ are possible auxiliary tasks.  These components may offer information that allows an agent to more quickly learn about $\theta_t$, which in turns improves its ability to perform its primary task.  If the number of components $k$ is large and most components are not relevant to the agent's primary task, a long-lived agent can improve its performance by learning which components {\it are} relevant.  In particular, an agent can learn over time about the vector $A$.  Learning about $A$ enhances the agent's ability to quickly learn about $\theta_t$.  Learning about $A$ can be thought of as {\it meta-learning}, as it aims to learn about something that can guide the agent's learning about its primary task.
In the following sections, we will investigate performance gains resulting from this approach.

\subsubsection*{Methods}\label{sec:continual-aux-learning-methods}
\paragraph{No auxiliary learning.}

As a baseline, we consider an agent that ignores all auxiliary components $Y_{t+1,2:k}$.
The agent maintains a vector $\hat{\phi}_t \in \mathbb{R}$. 
At each time, it first scales $\hat{\phi}_t$ by $\mu$ and then updates the vector via gradient descent to maximize reward:
\begin{align*}
    \hat{\phi}_{t}' &= \mu \hat{\phi}_t \\
     L_t &= -Y_{t+1,1} \log (
    \sigma (\hat{\phi}_t'))
    - (1-Y_{t+1,1}) \log(1 - \sigma (\hat{\phi}_t')
    )
    \\
    g_t & = \frac{\partial}{\partial {\hat{\phi}_t'}} L_t\\
    &= \frac{
    -Y_{t+1,1} + (1 - Y_{t+1,1}) \exp(\hat{\phi}_t')
    }{1+\exp(\hat{\phi}_t')}
    \\
    \hat{\phi}_{t+1} &= \hat{\phi}_t' - \alpha g_t,
\end{align*}
where the best learning rate $\alpha$ is found by grid search.
\paragraph{Auxiliary learning with $A$ known.}
We also consider the agent that perfectly learns from the auxiliary information. 
In particular, consider the agent that is given the value of $A$, and maintains $\hat{\theta}_t\in \mathbb{R}$. 
At each timestep, it first decays $\hat{\theta}_t$ by $\mu$, and then updates it by gradient descent to minimize the loss with respect to all components of $Y$:
\begin{align*}
    \hat{\theta}_t' &= \mu \hat{\theta}_t \\
     L_t &= -Y_{t+1}^T \log (
    \sigma (A \hat{\theta}_t'))
    - (1-Y_{t+1})^T \log(1 - \sigma (A \hat{\theta}_t')
    )
    \\
    g_t & = \frac{\partial}{\partial {\hat{\theta}_t'}} L_t\\
    &= A^T \left[\left(
    -Y_{t+1} + (1 - Y_{t+1}) \odot \exp(A \hat{\theta}_t')
    \right)
    \oslash \left(1+\exp(A \hat{\theta}_t')\right)\right]
    \\
    \hat{\theta}_{t+1} &= \hat{\theta}_t' - \alpha g_t,
\end{align*}
where we use $\odot$ and $\oslash$ to denote element-wise multiplication and division, respectively, and the best learning rate $\alpha$ is found by grid search.
\paragraph{Auxiliary learning with $A$ learned.}
This agent internally maintains both $\hat{A}_t$ and $\hat{\theta}_t$. 
It uses vanilla gradient descent on the loss with respect to all components to update $\hat{\theta}_t$, and learns $\hat{A}_t$ via meta gradient descent.
The update rules for $\theta_t$ are
\begin{align*}
    \hat{\theta}_t' &= \mu \hat{\theta}_t \\
     L_t &= -Y_{t+1}^T 
     \log (
    \sigma (\hat{A}_t \hat{\theta}_t'))
    - (1-Y_{t+1})^T \log(1 - \sigma (\hat{A}_t \hat{\theta}_t')
    )
    \\
    g_t & = \frac{\partial}{\partial {\hat{\theta}_t'}} L_t\\
    &= \hat{A}_t^T 
    \left[
    \left(-Y_{t+1} + (1 - Y_{t+1}) \odot \exp(\hat{A}_t \hat{\theta}_t')
    \right)
    \oslash \left(1+\exp(\hat{A}_t \hat{\theta}_t')\right)
    \right]
    \\
    \hat{\theta}_{t+1} &= \hat{\theta}_t' - \alpha g_t,
\end{align*}
and the update rules for $\hat{A}_t$ are
\begin{align*}
    h_t & = \mu h_{t-1} - \alpha \mu \left(
    -Y_{t} + (1 - Y_{t}) \odot \exp(\hat{A}_{t-1} \hat{\theta}_{t-1}')
    \right)
    \oslash \left(1+\exp(\hat{A}_{t-1} \hat{\theta}_{t-1}')\right) \\
    &\qquad\qquad - \alpha \mu \hat{A}_{t-1}^T \left[
    \exp(\hat{A}_{t-1} \hat{\theta}_{t-1}')
    \odot \hat{A}_{t-1} \oslash \left(1+\exp(\hat{A}_{t-1} \hat{\theta}_{t-1}')\right)^{\circ 2} \right] h_{t-1} 
    \\
    \hat{A}_{t+1} & = \hat{A}_{t} - \beta g_t h_t,
\end{align*}
where $A^{\circ 2}$ denotes element-wise square of $A$, and $\beta$ is the meta-learning rate.
Readers are referred to Appendix~\ref{apdx:continual-auxiliary-learning-meta-grads} for derivation of these update rules.

\subsubsection*{Results}

Figure~\ref{fig:continual-aux-results} plots reward versus time.
The results are generated with autoregressive model coefficient $\mu=0.99$, label dimension $k=100$, and $K=10$ useful components.
For the agent with no auxiliary learning and the agent performing auxiliary learning with $A$ known, the best learning rate $\alpha$ found by grid search is used ($0.202$ and $0.092$, respectively).
The agent performing auxiliary learning with $A$ learned uses a learning rate $\alpha=0.01$, which is far from optimal, and a meta-learning rate $\beta=0.05$.

\begin{figure}[htbp]
    \centering
    \includegraphics[width=0.7\linewidth]{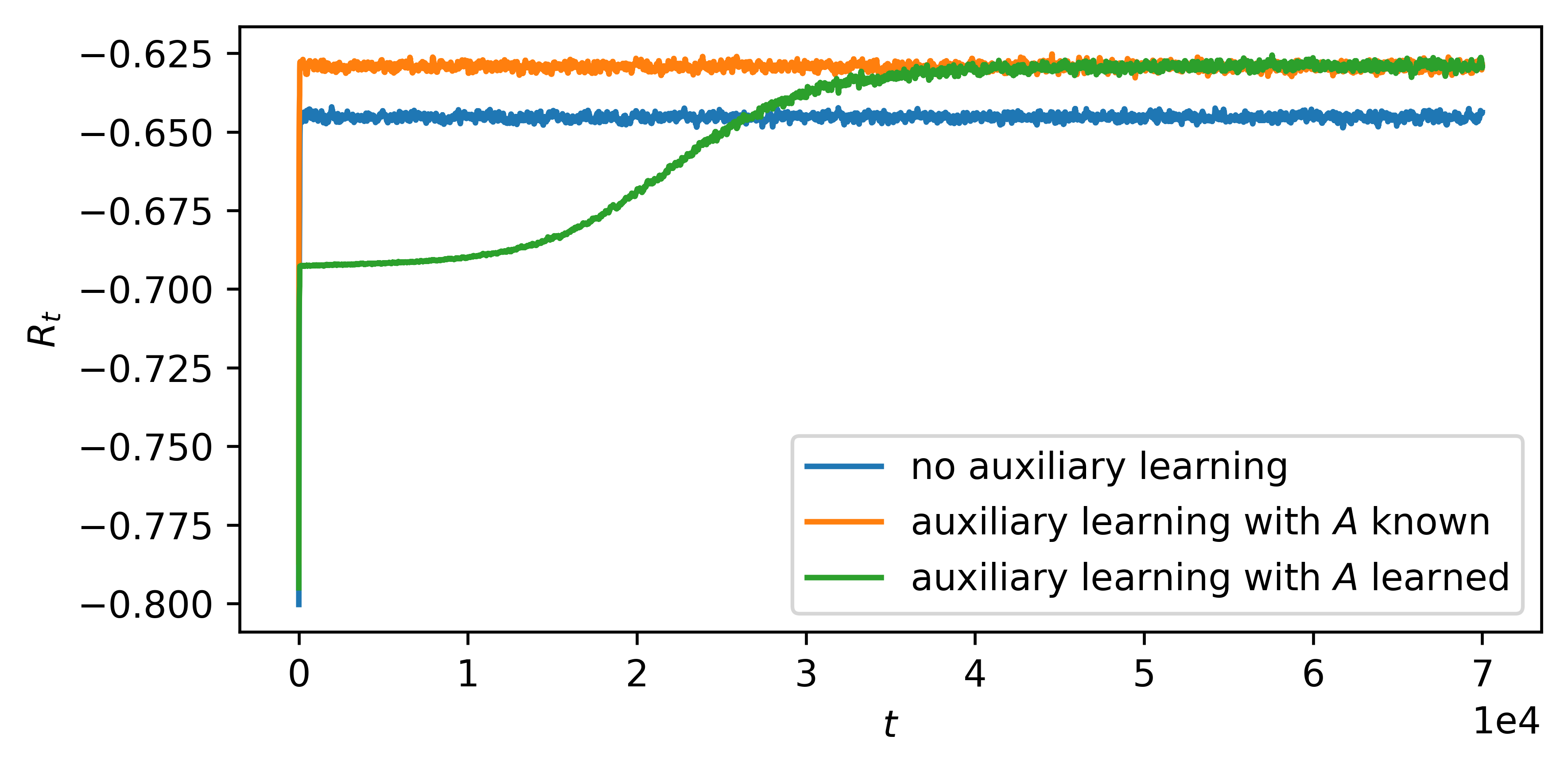}
    \caption{
    Learning $A$ eventually yields performance at the level attained if $A$ were known.  This exceeds the performance of an agent that does not engage in auxiliary learning.
    }
    \label{fig:continual-aux-results}
\end{figure}

From Figure~\ref{fig:continual-aux-results} we can see that learning $A$ eventually yields performance at the level achieved with $A$ known.  Either of these outperform an agent that does not engage in auxiliary learning.
This demonstrates how continual auxiliary learning can help an agent perform well in the long run, even if the agent does not initially know which auxiliary tasks are useful.

However, the performance improvement comes at the cost of extra computation.
In this particular example, the agent that does not engage in auxiliary learning requires $O(1)$ FLOPs per timestep, while the others require $O(k)$ FLOPs per timestep.
Hence, the performance improvement relies on roughly $k$ times more compute.
In the general case, implementing continual auxiliary learning in a scalable manner with only a modest increase in compute remains an interesting topic for future research.

\begin{summary}
\subsubsection*{Continual Supervised Learning}
\begin{itemize}
\item We consider continual supervised learning with an objective of \textbf{maximizing average online accuracy}: \begin{align*}
\begin{array}{ll}
\max_\pi & \liminf_{T\rightarrow\infty} \E_\pi \left[\frac{1}{T} \sum_{t=0}^{T-1} \mathbbm{1}(Y_{t+1} = \hat{Y}_{t+1}) \right] \\
\text{s.t.} & \text{computational constraint}
\end{array}
\end{align*}
\item \textbf{Evaluation Protocol}: we first tune hyperparameters on a development sequence and then train the agent with the best hyperparameters on multiple evaluation sequences. The agent's performance is averaged across all evaluation sequences.
\item Experiments with a variant of Permuted MNIST indicate that:
    \begin{itemize}
        \item A performant agent can \textbf{forget non-recurring information}.
        \item A performant agent can \textbf{forget recurring information if it can relearn that information quickly} relative to the duration of the information's utility.
        \item When computation is constraining, \textbf{forgetting can increase average reward}.
    \end{itemize}
\end{itemize}

\subsubsection*{Continual Exploration}
\begin{itemize}
\item Via experiments with a coin replacement game, we identify three properties of effective exploration in the face of nonstationarity:
\begin{itemize}
    \item An optimal agent \textbf{never stops exploring.}
    \item An agent should \textbf{prioritize acquisition of more durable information -- that is, information that will remain valuable and relevant over a longer duration.}
    \item To assess durability of information, an agent can benefit from \textbf{learning about environment dynamics}. 
\end{itemize}
\item Via experiments with a two-armed Gaussian bandit, we demonstrate that \textbf{prioritizing acquisition of more durable information is beneficial} for maximizing average reward and that the benefit of this is greater in environments with less durable information.
\end{itemize}

\subsubsection*{Continual Learning with Delayed Consequences}
\begin{itemize}
\item Optimistic Q-learning can learn from delayed consequences in a nonstationary Markov decision processes.
\item As the degree of nonstationarity increases, more intense exploration increases average reward.
\item A more sophisticated agent may be able to adapt stepsize, discount, and optimism hyperparameters online to increase average reward.
\end{itemize}

\subsubsection*{Continual Auxiliary Learning}
\begin{itemize}
\item \textbf{Auxiliary tasks} are tasks that are distinct from, though possibly helpful to, the primary task of maximizing average reward.
\item \textbf{Auxiliary learning} refers to identifying useful auxiliary tasks and their relationship to the primary task.
\item A long-lived agent has time to learn not only how to perform auxiliary tasks but also what auxiliary tasks are useful to learn about and how they relate to the primary task.
\end{itemize}

\end{summary}
\clearpage

\section{Conclusion}

\textbf{Summary.} In this monograph, we framed continual learning as computationally constrained reinforcement learning. Under this perspective, we formally introduced an objective for continual learning. This objective is to maximize the infinite-horizon average reward under computational constraints. We formalized the concepts of agent state, information capacity, and learning targets, which play key roles for thinking about continual learning and distinguishing it from traditional vanishing-regret learning. Leveraging information theory, we decomposed the prediction error into forgetting and implasticity components. We concluded with case studies that studied implications of our objective on behaviors of performant agents. 

\textbf{Future Research.} In our case studies, we discussed how different agent capabilities contribute to performance under our objective. These capabilities include balancing forgetting with fast relearning, seeking out durable information when exploring, modeling environment dynamics, and meta-learning hyperparameters. We hope that this work inspires researchers to design agents that exhibit such capabilities and to study the trade-offs between them under computational constraints. In particular, we hope our holistic objective helps researchers reason about these trade-offs and reveal additional desired capabilities.

\section*{Acknowledgements}

Financial support from the Stanford Knight Hennessy Fellowship, the Stanford MS\&E fellowship, and the Army Research Office (ARO) Grant W911NF2010055 is gratefully acknowledged.

We thank participants of the 2023 Barbados Lifelong Reinforcement Learning Workshop, Stanford University students of the 2022 offering of Reinforcement Learning: Frontiers, Dave Abel, Andre Baretto, Dimitri Bertsekas, Shibhansh Dohare, Clare Lyle, Sanjoy Mitter, Razvan Pascanu, Doina Precup, Marc'Aurelio Ranzato, Mark Ring, Satinder Singh, Rich Sutton, John Tsitsiklis, Hado van Hasselt, Tsachy Weissman, and Zheng Wen for stimulating discussions and feedback that greatly benefited this monograph.

\bibliography{references}
\clearpage

\appendix

\section*{Appendix}
\addcontentsline{toc}{section}{Appendix}

\section{Implications on Finite Time Behavior}\label{app:objective}
\henrik{
In this section, we elaborate on our argument in Section \ref{sec:objective} on why a secondary criterion to expected average reward may not be necessary. We will not provide a formal proof but rather provide the basic intuition behind this argument. 

We will begin by making two assumptions.
\begin{enumerate}
    \item {\bf (limited capacity)} The agent is capacity constrained in the sense that it can only take on a finite number of ``agent states'' (introduced in Section \ref{sec:information}). This means it can only store a finite number of bits. 
    \item {\bf (capacity gates performance)} The environment is very complex and more capacity always improves expected average reward. We say that the agent's performance is gated by capacity.
\end{enumerate}

To make things more concrete, let us consider the coin tossing example in Example \ref{ex:coins-evolving-biases}. We reproduce the example below for convenience. 

\begin{example}
{\bf (coin replacement)} \textit{Recall the environment of Example \ref{ex:coins-fixed-biases}, but suppose that, at each timestep $t$, before action $A_t$ is executed, each coin $a$ is replaced with a new coin with some fixed probability $q_a$.  Coin replacement events are independent, and each new coin's bias is independently sampled from its prior distribution. With this change, biases of the two available coins can vary over time.  Hence, we introduce time indices and denote biases by $(p_{t,1}, p_{t,2})$.}
\end{example}

Consider two agents with the same capacity constraint and same expected average reward:
\begin{enumerate}
    \item {\bf (optimal agent)} This agent maximizes expected average reward given capacity constraints.
    \item {\bf (late-bloomer agent)} This agent is identical to the optimal agent, except that for the first $10,000$ timesteps it takes the first action, regardless of what it observes. Therefore, it performs poorly for this period. From timestep $10,001$ onwards it behaves exactly like the first agent. We will call this the 'late-bloomer' agent.
\end{enumerate}

We will now argue that if both agents are computationally constrained, the late-bloomer agent can be improved so as to outperform the first 'optimal' agent on average reward without increasing its capacity. That is, the so-called optimal agent is actually not optimal. This would imply that considering average reward is sufficient.

In Section \ref{sec:information}, we thought of the agent as probabilistically choosing actions, and probabilistically transitioning into the next agent state. Here, in contrast, to make the argument clearer, we will think of the agent as acting deterministically. This just amounts to defining agent state to also include the state of the agent's random number generator. 

Because the late-bloomer agent always takes the same action for the first $10,000$ timesteps, it must explicitly, or implicitly, keep track of which timestep it is in. We can conclude two things about how the late-bloomer agent must allocate agent states:
\begin{enumerate}
    \item {\bf ($\mathbf{10,000}$ 'lazy' agent states are needed)} The agent has to allocate at least $10,000$ agent states corresponding to the first $10,000$ timesteps such that action $1$ is taken regardless of what it observes. As a shorthand, we call these agent states 'lazy'. Why are at least $10,000$ needed? Well, if there were less than $10,000$ agent states allocated to the first $10,000$ timesteps, one of the agent states would have to recur in that time period. This would induce a loop, such that the agent would keep coming back to the same state over and over, and thus taking action $1$ forever.  This would contradict the assumption that it would behave identically to first agent after $10,000$ timesteps. 
    \item {\bf ('lazy' agent states cannot recur after the first $\mathbf{10,000}$ timesteps)} The agent states corresponding to the first $10,000$ timesteps can never recur after the initial time period of $10,000$ timesteps. Why is that? Imagine if the agent state corresponding to the tenth timestep would recur at some later timestep. In that case, the agent would take action $1$ for the next $9,990$ timesteps. That would contradict the assumption that it behaved identically to the first 'optimal' agent, after timestep $10,000$.
\end{enumerate}

Therefore, after timestep $10,000$ the late-bloomer agent's effective capacity is less than the first 'optimal' agent. The late-bloomer is effectively "wasting" agent states. Therefore, the late-bloomer's effective capacity could be increased without increasing actual capacity. Since we assumed that capacity gates performance, this means, we can improve the late-bloomer such that it outperforms the first 'optimal' agent on expected average reward. Thus, it would be sufficient to consider expected average reward to pick out the better agent.

}

\section{Continual Learning Agents}\label{app:cl_agents}

In our abstract formulation, a continual learning agent is characterized by an agent policy $\pi$, which specified the conditional distribution of each action $A_t \sim \pi(\cdot|H_t)$. To crystalize this notion, in this section we describe a few specific instances as concrete examples of agents. In each instance we present the interface $(\aspace, \ospace)$ for which the agent is designed, the algorithm the agent implements to compute each action $A_t$, and an environment in which the agent ought to be effective.

\subsection{Tracking}\label{sec:tracking}
Suppose observations $O_1, O_2, O_3, \ldots$ are noisy measurements of a latent stochastic process \henrik{$\theta_1, \theta_2, \theta_3, \ldots$}.  For example, each \henrik{observation} could be a thermostat reading, which is not exactly equal to the current temperature \henrik{$\theta_t$}. \henrik{The temperature $\theta_t$ at each timestep is unobserved, and is thus a latent variable.} A tracking agent generates estimates $A_t$ of the latent process, each of which can be viewed as predictions of $O_{t+1}$.  The least mean squares (LMS) algorithm implements a simple tracking agent.  While the algorithm more broadly applies to vector-valued observations, to start simple, we consider only the scalar case.
\begin{example}
\label{ex:scalar-lms}
{\bf (scalar LMS)}
This agent interacts with an environment through real-valued actions and observations: $\aspace = \ospace = \mathbb{R}$. Each action $A_t$ represents a prediction of the next observation $O_{t+1}$, and to penalize errors, the reward function is taken to be negative squared error: $r(H_t, A_t, O_{t+1}) = - (O_{t+1} - A_t)^2$.  Initialized with $\mu_0 \in \mathbb{R}$, the agent updates this parameter according to
$$\mu_{t+1} = \eta \mu_t + \alpha (O_{t+1} - \eta \mu_t),$$
where $\eta \in [0,1]$ and $\alpha \in [0,1]$ are shrinkage and stepsize hyperparameters.  The agent executes actions $A_t = \eta \mu_t$.
\end{example}
The LMS algorithm is designed to track a latent process that generates observations.  Let us offer an example of an environment driven by such a process, for which the algorithm is ideally suited. Consider a random sequence $(\theta_t: t \in \mathbb{Z}_+)$, with the initial latent variable $\theta_0$ distributed according to a prior $\mathcal{N}(\mu_0, \Sigma_0)$ and updated according to $\theta_{t+1} = \eta \theta_t + Z_{t+1}$, with each process perturbation $Z_{t+1}$ independently sampled from $\mathcal{N}(0,\zeta^2)$. Further, suppose that $O_{t+1} = \theta_{t+1} + W_{t+1}$, with each observation noise sample $W_{t+1}$ drawn independently from $\mathcal{N}(0,\sigma^2)$. What we have described provides a way of generating each observation $O_{t+1}$. We have thus fully characterized an environment $\env = (\aspace, \ospace, \rho)$. \henrik{Specifically, $\aspace = \ospace = \mathbb{R}$, and $\rho$ is implied by the preceding equations.}  

Note that we are using the parameter $\eta$ for two purposes.  First, it is a hyperparameter of the agent described in Example \ref{ex:scalar-lms}. Second, it serves in the specification of this hypothetical environment that we frame as one for which the agent is ideally suited. For this environment, with an optimal choice of stepsize, the LMS algorithm attains minimal average reward over all agent designs.  Figure \ref{fig:tracking}(a) plots average reward as a function of stepsize. The tracking behavior in Figure \ref{fig:tracking}(b) is attained by the optimal stepsize.

\begin{figure}
\centering
\begin{subfigure}{0.4\textwidth}
    \includegraphics[width=\textwidth]{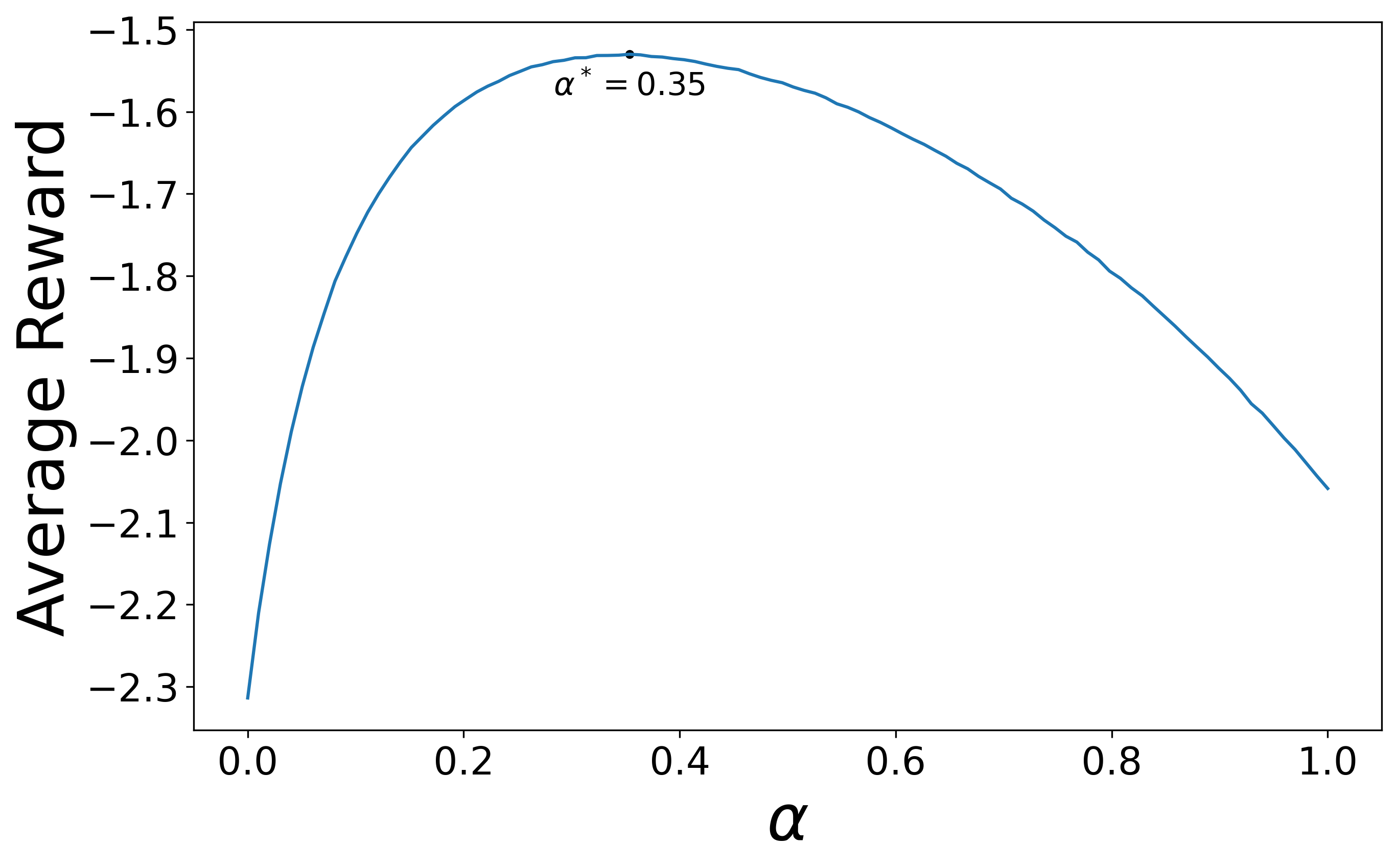}
    \caption{Average reward as a function of stepsize $\alpha$, averaged over $10^4$ time steps and $10^3$ trials.}
\end{subfigure}
\begin{subfigure}{0.4\textwidth}
    \includegraphics[width=\textwidth]{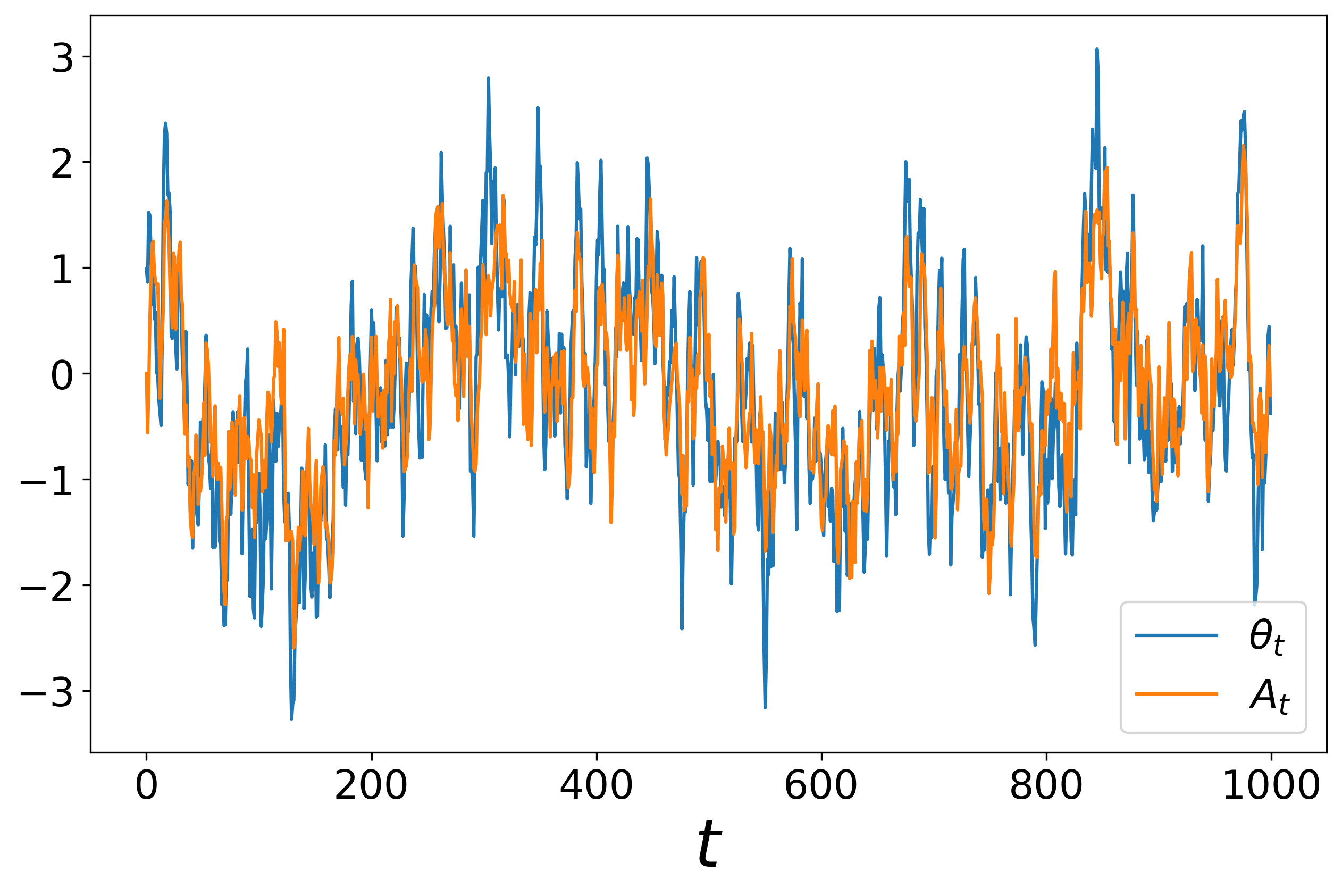}
    \caption{The LMS agent with optimal stepsize ($\alpha=0.35$) tracks the latent scalar process well.}
\end{subfigure}
\caption{Tracking a latent scalar process using the LMS algorithm. The plots are generated with $\mu_0=0$, $\Sigma_0=1$, $\eta=0.9$, $\zeta=0.5$, and $\sigma=1$.}
\label{fig:tracking}
\end{figure}

\subsection{Exploration}\label{sec:exploration}

In Example \ref{ex:scalar-lms}, actions do not impact observations. The following agent updates parameters in a similar incremental manner but then uses them to select actions that determine what information is revealed through observations.
\begin{example}
\label{ex:nonstationary-Gaussian-TS}
{\bf (nonstationary Gaussian Thompson sampling)}
This agent interfaces with a finite action set $\aspace$ and real-valued observations $\ospace = \mathbb{R}$.  The reward is taken to be the observation, so that $R_{t+1} = r(H_t, A_t, O_{t+1}) = O_{t+1}$.  Initialized with $\mu_0 \in \mathbb{R}^\aspace$, the agent updates parameters according to
\begin{align*}
\mu_{t+1,a} = \left\{\begin{array}{ll}
\eta \mu_{t,a} + \alpha_{t+1} (O_{t+1} - \eta \mu_{t,a}) & \qquad \mathrm{if\ } a = A_t \\
\eta \mu_{t,a} & \qquad \mathrm{otherwise.}
\end{array}\right.
\end{align*}
Note that each observation only impacts the component of $\mu_t$ indexed  by the executed action. The stepsize varies with time according to $\alpha_{t+1} = \Sigma_{t+1, A_t} / \sigma^2$, for a sequence initialized with a vector $\Sigma_0 \in \mathbb{R}_+^\aspace$ and updated according to
\begin{align*}
\Sigma_{t+1,a} = \left\{\begin{array}{ll}
\frac{1}{\frac{1}{\eta^2 \Sigma_{t,a} + \zeta^2} + \frac{1}{\sigma^2}} & \qquad \mathrm{if\ } a = A_t \\
\eta^2 \Sigma_{t,a} + \zeta^2 & \qquad \mathrm{otherwise,}
\end{array}\right.
\end{align*}
where $\eta \in \mathbb{R}$, $\zeta \in \mathbb{R}_+$, $\sigma \in \mathbb{R}_+$ are scalar hyperparameters.  
Each $\hat{\theta}_{t,a}$ is drawn independently from $\mathcal{N}(\mu_{t,a}, \Sigma_{t,a})$.  Then, $A_t$ is sampled uniformly from the set $\argmax_{a \in \aspace} \hat{\theta}_{t,a}$.
\end{example}

\henrik{We will now describe an environment this agent can be applied to. The environment is characterized by latent random variables updated according to $\theta_{t+1,a} = \eta \theta_{t,a} + Z_{t+1,a}$, with each $Z_{t+1,a}$ independently sampled from $\mathcal{N}(0,\zeta^2)$.  The agent maintains parameters of the posterior distribution of $\theta_{t,a}$, which is $\mathcal{N}(\mu_{t,a}, \Sigma_{t,a})$.  For each action $a$, the agent samples $\hat{\theta}_{t,a}$ from the posterior distribution.  The agent then executes an action that maximizes among these samples. We call this environment nonstationary since the agent can be interpreted as tracking changing latent variables $\theta_{a,t}$.}

\henrik{Note that we are using the parameters $\eta$ and $\zeta$ for two purposes. First, they are hyperparameters of the agent described in the Example \ref{ex:nonstationary-Gaussian-TS}.  Second, they specify this hypothetical environment that we frame as one for which the agent is well-suited.}

The environment we have described constitutes a {\it bandit}; \henrik{this is a reference to a slot machine.}   With multiple arms, a player chooses at each time an arm $A_t$ and receives a payout $R_{t+1}$.  Because payout distributions are not listed, the player can learn them only by experimenting. As the player learns about each arm's payouts, they face a dilemma: in the immediate future they expect to earn more by exploiting arms that yielded high payouts in the past, but by continuing to explore alternative arms they may learn how to earn higher payouts in the future.

\henrik{If $\eta=1$ and $\zeta=0$, the environment becomes stationary. The distribution for arm $a$ is Gaussian with unknown mean $\theta_a$, which can be learned by observing payouts. The agent is then selecting actions via Thompson sampling \citep{russo2018tutorial,thompson1933likelihood}.  In particular, at each timestep $t$, the posterior distribution of $\theta_a$ is $\mathcal{N}(\mu_{t,a}, \Sigma_{t,a})$, and the agent samples $\hat{\theta}_{t,a}$ from this distribution, then executes an action that maximizes among them.}

This nonstationary Gaussian Thompson sampling agent may be applied in other environments as well, whether or not observations are driven by Gaussian processes.  For example, with a suitable choice of $\zeta$, it may exhibit reasonable behavior in the coin replacement environment of Example \ref{ex:coins-evolving-biases}.  However, the agent does not adequately address environments in which actions induce delayed consequences.

\subsection{Delayed Consequences}\label{app:delayed_consequences}

\henrik{The agent of Example \ref{ex:nonstationary-Gaussian-TS} maintains parameters $\mu_t$ that serve the prediction $\eta \mu_{t,A_t}$ of the expected immediate reward $R_{t+1}$.  These predictions can guide the agent to select actions that earn high immediate reward.  To learn from and manage delayed rewards, the Q-learning algorithm \citep{watkins1989learning} instead maintains predictions of the expected discounted return $\sum_{k=0}^\infty \gamma^k R_{t+k+1}$.  To make such predictions, the agent will condition on both the action $A_t$, and what's known as a \textit{situational state} $S_t$, which provides context for the agent's decision.  The following example elaborates.
}

\begin{example}
\label{ex:optimistic-q-learning}
{\bf (optimistic Q-learning)}
This agent is designed to interface with any observation set $\ospace$ and any finite action set $\aspace$.  It maintains an action value function $Q_t$, which maps the situational state $S_t$, which takes values in a finite sets $\mathcal{S}$, and action $A_t$ to a real value $Q_t(S_t,A_t)$.  The situational state $S_t$ represents features of the history $H_t$ that are predictive of future value.  It is determined by an update function $f$, which forms part of the agent design, according to
$$S_{t+1} = f(S_t,A_t,O_{t+1}),$$
and substitutes for history in the computation of rewards, which take the form $R_{t+1} = r(S_t, A_t, S_{t+1})$.
We consider an optimistic version of Q-learning, which updates the action value function according to
$$Q_{t+1}(s,a) = \left\{\begin{array}{ll}
Q_t(S_t,A_t) + \alpha \left(R_{t+1} + \gamma \max_{a \in \aspace} Q_t(S_{t+1}, a) - Q_t(S_t, A_t)\right) + \zeta \qquad & \text{if } (s,a) = (S_t,A_t) \\
Q_t(s,a) + \zeta & \text{otherwise.}
\end{array}\right.
$$
Hyperparameters include a discount factor $\gamma \in (0,1)$, a stepsize $\alpha \in (0,1)$, and an optimistic boost $\zeta \in \mathbb{R}_+$.  Each action $A_t$ is sampled uniformly from $\argmax_{a \in \aspace} Q_t(S_t,a)$.
\end{example}
To interpret this algorithm, consider situational state dynamics generated by \henrik{an unknown Markov Decision Process (MDP)}, identified by a tuple $(\mathcal{S}, \mathcal{A}, r, P, S_0)$.  Here, $P$ is a random variable that specifies transition probabilities $P_{a,s,s'}$ for actions $a \in \aspace$ and situational states $s,s' \in \mathcal{S}$.  \henrik{That is, the probability of sampling next state $s'$, given the current state $s$ and action $a$ is $P_{a,s,s'}$.}
If the hyperparameters $\gamma$, $\alpha$, and $\zeta$ anneal over time at suitable rates to $1$, $0$, and $0$, along the lines of a more sophisticated version of Q-learning analyzed by \cite{dong2022simple}, the sequence $(Q_t: t \in \mathbb{Z}_+)$ should converge to the optimal action value function of the MDP.  By the same token, the agent would attain optimal average reward.

Versions of Q-learning permeate the RL literature.  %
\henrik{The one we have described presents some features especially relevant for continual learning. First, action values are boosted by exploration bonuses, which incentivizes visits to state-action pairs that have not been visited much recently. To see this, note that the if the agent repeatedly takes the same action in some state, the action value estimate will keep growing for all the other actions. Therefore, the agent will eventually pick one of the other actions. The exploration bonus in this algorithm is similar to that considered by \cite{sutton1990integrated}, except that it incentivizes visits to state-action pairs even if they have already been visited many times, as long as they haven't been visited much recently.}  %
\henrik{Second, this algorithm may perform well even if the Markov property assumption does not hold.}
It suffices for $S_t$ to enable useful predictions of value that guide effective decisions rather than approximate the state of the environment, which could be far more complex.  This suits the spirit of continual learning, since the subject is oriented toward addressing very complex environments.  Another noteworthy feature is that the hyperparameters $\gamma$, $\alpha$, and $\zeta$ are fixed.  If the situational state dynamics were stationary, annealing hyperparameters over time is beneficial.  However, as we have discussed in Section \ref{sec:delayed-consequences}, fixed parameters fare better in the face of nonstationarity.

\section{MDP Exchangeability}\label{app:mdp_exchangeability}
This section presents conditions under which it is natural to characterize dynamics using an unknown MDP, in other words, characterizing the environment as a mixture of MDPs. Suppose each observation is interpreted as a situational state $S_t = O_t$ and 
rewards depend on history only through the situational state, taking the form $R_{t+1} = r(S_t, A_t, S_{t+1})$. 
The environment is a mixture of MDPs if there exists a random variable $P$, where each $P_a$ is a $|\mathcal{S}| \times |\mathcal{S}|$ stochastic transition matrix with $P_{sas^{\prime}}$ denoting its $(s, s^{\prime})$ entry, 
such that for all policies $\pi$, time 
$T \in \mathbb{Z}_{+}$, and histories $h_T 
= (a_0, s_1, a_1, s_2, ... , a_{T-1}, s_T)
\in \mathcal{H}_T$, we have 
\begin{align}
\mathbb{P}_{\pi}(H_T = h_T | P) = 
\prod_{t = 0}^{T-1} P_{s_t a_t s_{t+1}} \prod_{t = 0}^{T-1} \pi(a_t | h_t),
\label{eq:mix_mc}
\end{align}
where $h_t = (a_0, s_1, a_1, s_2, ... , a_{t-1}, s_t)$. This $P$ represents a natural learning target in this case.

We assume that a recurrent-like property holds. Specifically, let $\overline{\pi}$ be defined such that $\overline{\pi}(a | h) = 1 / |\actions|$ for all $a \in \actions$ and $h \in \mathcal{H}$, we assume that 
\begin{align}
\mathbb{P}_{\overline{\pi}}\left(S_t = S_1 \text{ for infinitely many } t\right) = 1. 
\label{eq:recurrence}
\end{align}

We introduce a certain kind of symmetry condition on the set of all histories. Specifically, we introduce a binary relation $\sim$. 
Let $h, h^{\prime} \in \mathcal{H}$ be two histories of finite horizon. We let $h \sim h^{\prime}$ if and only if they exhibit the same number of state-action-state transition counts, for every such state-action-state tuple. 
For instance, let $\mathcal{S} = \mathbb{N}$, and $\actions = \{a, b, c, d\}$, $(2, b, 3) \nsim (3, b, 2)$ but 
$(2, b, 1, a, 3, c, 2) \sim (3, c, 2, b, 1, a, 3)$.

We assume that $S_0$ is deterministic and $S_t \neq S_0$ a.s. for all $t \in \mathbb{N}_{++}$. 

We establish a theorem, which extends a result on partial exchangeability established by \cite{diaconis1980finetti}. The theorem provides a necessary and sufficient condition for the environment to be represented by a mixture of MDPs. 
\begin{theorem}
Suppose Equation \ref{eq:recurrence} holds. Then the environment is a mixture of MDPs in the sense of Equation \ref{eq:mix_mc} if and only if for all policies $\pi$ and $\pi^{\prime}$, all $T \in \mathbb{Z}_{+}$, and all histories $h_T, h^{\prime}_T \in \mathcal{H}_T$, 
\begin{align*}
h_T \sim h_T^{\prime} \quad \mathrm{implies} \quad \frac{\mathbb{P}_{\pi}(H_T = h_T)}{\prod_{t = 0}^{T-1} \pi(a_t | h_t)} = \frac{\mathbb{P}_{\pi^{\prime}}(H_T = h_T^{\prime})}{\prod_{t = 0}^{T-1} \pi^{\prime}{(a_t^{\prime} | h_t^{\prime})}}. 
\end{align*}
\label{th:mdp}
\end{theorem}

\begin{proof}
The ``$\Rightarrow$" direction is obvious. We prove the other direction. 
Let $\overline{\pi}$ be defined such that $\overline{\pi}(a | h) = \frac{1}{|\actions|}$ for $a \in \actions$ and $h \in \mathcal{H}$. For any two histories $h_T, h^{\prime}_T \in \mathcal{H}_T$ for all $T \in \mathbb{Z}_{+}$, if the state-action-state-action transition counts in $h_T$ and $h^{\prime}_T$ are the same for all state-action-state-action tuples (we assume that the last action is the same), then $h_T \sim h^{\prime}_T$. This implies that 
\begin{align*}
\frac{\mathbb{P}_{\overline{\pi}}(H_T = h_T)}{\prod_{t = 0}^{T-1} \overline{\pi}(a_t | h_t)} = \frac{\mathbb{P}_{\overline{\pi}}(H_T = h_T^{\prime})}{\prod_{t = 0}^{T-1}\overline{\pi}(a_t^{\prime} | h_t^{\prime})}.
\end{align*}
Hence, $\mathbb{P}_{\overline{\pi}}(H_T = h_T) = \mathbb{P}_{\overline{\pi}}(H_T = h_T^{\prime})$. 

By Theorem 7 of \citep{diaconis1980finetti}, $X_t = (S_t, A_t)$, defined by $\overline{\pi}$, is a mixture of Markov chains. In other words, there exists a 
$|\mathcal{S}| |\actions| \times |\mathcal{S}| |\actions|$ 
stochastic transition matrix-valued random variable $\overline{P}$, where each element in the $(sa, s^{\prime}a^{\prime})$-th position is denoted by $\overline{P}_{s a, s^{\prime} a^{\prime}}$, such that for all $T \in \mathbb{Z}_+$ and history $h \in \mathcal{H}_T$, 
\begin{align*}
\mathbb{P}_{\overline{\pi}}\left(H_T = h_T | \overline{P}\right) = 
\prod_{t = 0}^{T-1} \overline{P}_{s_t a_t, s_{t+1} a_{t+1}}. 
\end{align*}
We define $P$ as follows: for all $s, s^{\prime} \in \mathcal{S}$ and $a \in \actions$, we have 
\begin{align*}
P_{s a s^{\prime}} = \sum_{a^{\prime} \in \actions} \overline{P}_{sa, s^{\prime}a^{\prime}}, 
\end{align*}
(or, equivalently, $P_{s a s^{\prime}} = |\actions| \overline{P}_{s a, s^{\prime} a^{\prime}}$ for any $a^{\prime} \in \actions$). Then, for all policies $\pi$, all $T \in \mathbb{Z}_+$ and all $h_T \in \mathcal{H}_T$, 
\begin{align*}
\mathbb{P}_{\pi}(H_T = h_T | P) = \frac{P_{\overline{\pi}}(H_T = h_T | P)}{\prod_{t = 0}^{T-1} \overline{\pi}(a_t | h_t)}
\prod_{t = 0}^{T-1} \pi(a_t | h_t) 
= \prod_{t=0}^{T-1} P_{s_t a_t s_{t+1}} \prod_{t = 0}^{T-1} \pi(a_t | h_t). 
\end{align*}
\end{proof}
When the environment is a mixure of MDP determined by $P$, the latent transition probabilities serve as a natural learning target $\chi = P$. The ``efficient reinforcement learning'' literature presents many agents and analyses that treat $P$ as a learning target.

\henrik{
\section{Useful results from information theory}\label{app:information_theory}
\begin{lemma}[Chain rule of mutual information]\label{lem:chain-rule-mutual}
    \begin{align*}
        \I(X;Y,Z) = \I(X;Z) + \I(X; Y | Z).
    \end{align*}
\end{lemma}

\begin{lemma}\label{lem:mutual-information-unchanged-1}
If $D\perp B |(A,C)$, then
    \begin{align*}
        \I(A, D; B | C) = \I (A; B | C)
    \end{align*}
\end{lemma}
\begin{proof}
By the chain rule of mutual information (Lemma~\ref{lem:chain-rule-mutual})
    \begin{align*}
    \I(A, D; B | C) - \I (A; B | C)
    = \I(D; B | C, A),
    \end{align*}
    which is $0$ by assumption.
\end{proof}

}

\section{Capacity-Constrained LMS with an AR(1) Process}\label{apdx:lms}
In this section we provide a detailed analysis of the didactic example discussed in Section \ref{se:lms}.
In \ref{apdx:lms-reparam} we introduce a reparameterization of the agent update rule in Equation~\ref{eq:lms-agent-update}, and derive some useful formulas, including the posterior predictions of the agent (Theorem \ref{thm:lms-posterior-pred}).
We then apply these formulas to calculate the implasticity error and forgetting error (\ref{apdx:lms-plasticity}).

\subsection{Reparameterizing the Agent}
\label{apdx:lms-reparam}
In this section we offer a reparameterization of the update rule in Equation~\ref{eq:lms-agent-update}, and derive some useful formulas.
We let ${\alpha'}=1-\alpha$.
Then, Equation~\ref{eq:lms-agent-update} can be rewritten as
\begin{equation}
\label{eq:lms-agent-update-2}
U_{t+1} = {\alpha'} U_t + \alpha Y_{t+1} + Q_{t+1}
.
\end{equation}
In the results to follow, we derive closed form expressions for $U_t$ and $Y_{t+1}$ in terms of independent Gaussian random variables. Doing so will allow us derive exact formulas for relevant information-theoretic quantities.

We begin by deriving a closed form formula for $U_t$ in terms of independent Gaussian random variables:

\begin{lemma}
    For all $t \in \mathbb{Z}_{++}$,
    \[
    U_t = 
    \sum_{i=0}^{t-1} \alpha {\alpha'}^i W_{t-i}
    + \sum_{i=0}^{t-1} \alpha \frac{\eta^{i} - {\alpha'}^{i}}{\eta-{\alpha'}} V_{t-i}
    + \alpha \frac{\eta^{t} - {\alpha'}^{t}}{\eta-{\alpha'}} \theta_0
    + {\alpha'}^t U_0
    + \sum_{i=0}^{t-1} {\alpha'}^i Q_{t-i}
    ,
    \]
    where we set $\frac{\eta^{i} - {\alpha'}^{i}}{\eta-{\alpha'}}=i\eta^{i-1}$ when $\eta={\alpha'}$.
\end{lemma}
\begin{proof}
    The base case $(t=0)$ can trivially be verified.
    The cases where $t\ge 1$ can be proven by induction, by directly plugging in the formula.
\end{proof}

We now provide analogous formulas for $\theta_t$ and $Y_{t+1}$.
\begin{lemma}
    For all $t\in \mathbb{Z}_{++}$,
    \[
    \theta_t = \sum_{i=0}^{t-1} \eta^i V_{t-i} + \eta^t \theta_0.
    \]
\end{lemma}

\begin{corollary}
    For all $t\in \mathbb{Z}_{++}$
    $$Y_{t+1}  = \sum_{i=0}^{t-1} \eta^i V_{t-i} + \eta^t \theta_0 + W_{t+1}.$$
\end{corollary}

We omit the proofs as they all follow from simple applications of induction.

As we are interested in the steady state behavior of error, the following lemma provides closed-form expressions for quantities which comprise the covariance matrices relevant for forgetting and implasticity error computations.
\begin{lemma}
\label{lem:lms-cov}
As $t\to\infty$,
\begin{align*}
    \E[Y_t^2] &= 1 + \sigma^2 ,\\
    \E[Y_t Y_{t+k}] &= \eta^{k} \qquad \forall k\ge1 ,\\
    \E[U_t^2] &= 
    \frac{\alpha^2 \sigma^2}{1-{\alpha'}^2} 
    + \frac{\alpha^2(1-\eta^2)}{(\eta-{\alpha'})^2}\left(
    \frac{1}{1-\eta^2}
    +\frac{1}{1-{\alpha'}^2}
    -\frac{2}{1-\eta{\alpha'}}
    \right)
    + \frac{\delta^2}{1-{\alpha'}^2} \\
    &=
    \alpha \frac{
    \sigma^2 (1 - {\alpha'} \eta) + 1 + {\alpha'} \eta
    }{
    (1 - {\alpha'} \eta) ( 1 + {\alpha'})
    } + \frac{\delta^2}{1-{\alpha'}^2}
    ,\\
    \E[U_tU_{t+1}] &=
    \frac{\alpha^2 \sigma^2 {\alpha'}}{1-{\alpha'}^2} 
    + \frac{\alpha^2(1-\eta^2)}{(\eta-{\alpha'})^2}\left(
    \frac{\eta}{1-\eta^2}
    +\frac{{\alpha'}}{1-{\alpha'}^2}
    -\frac{\eta + {\alpha'}}{1-\eta{\alpha'}}
    \right)
    + \frac{{\alpha'}\delta^2}{1-{\alpha'}^2}\\
    &=
    \alpha \frac{
    {\alpha'}\sigma^2 (1 - {\alpha'} \eta) + {\alpha'} + \eta
    }{
    (1 - {\alpha'} \eta) ( 1 + {\alpha'})
    } + \frac{{\alpha'}\delta^2}{1-{\alpha'}^2}
    ,\\
    \E[U_{t+k}Y_t] &=
    \alpha \sigma^2{\alpha'}^k + 
    \frac{\alpha (1-\eta^2)}{\eta - {\alpha'}}\left(
    \frac{\eta^{k+1}}{1-\eta^2} - \frac{{\alpha'}^{k+1}}{1-\eta{\alpha'}}
    \right) \qquad \forall k\ge0
    \\
    &= \alpha \sigma^2{\alpha'}^k + \frac{\alpha}{1 - \eta {\alpha'}}
    \left(
    \frac{\eta^{k+1}-{\alpha'}^{k+1}}{\eta-{\alpha'}}
    - \eta^2{\alpha'} \frac{\eta^{k}-{\alpha'}^{k}}{\eta-{\alpha'}}
    \right)
    ,\\
    \E[U_tY_{t+k}] &=
    \frac{\alpha\eta^{k-1}(1-\eta^2)}{\eta-{\alpha'}}\left(
    \frac{1}{1-\eta^2} - \frac{1}{1-\eta{\alpha'}}
    \right) \qquad \forall k\ge 1
    \\
    &= \frac{\alpha \eta^k }{1 - {\alpha'} \eta}
    .
\end{align*}

\end{lemma}

For the remainder of the analysis, we will assume that $U_t$ and $Y_t$ are weak-sense jointly stationary processes with covariance given by the previous expressions.
Since $U_t$ and $Y_t$ are jointly Gaussian, this implies that they are strong-sense jointly stationary as well.

With the expressions derived in Lemma \ref{lem:lms-cov}, we now provide the steady state posterior distribution of $Y_{t+1}$ conditioned on $U_t$. Recall that this posterior distribution constitutes the \emph{optimal} prediction of $Y_{t+1}$ conditioned on $U_t$.

\begin{theorem}\label{thm:lms-posterior-pred}
    \textbf{(posterior predictions)}
    $$\lim_{t\to\infty} \Prob(Y_{t+1}\in\cdot|U_t) = \mathcal{N}\left(\mu,\ \Delta^2\right),$$
    where $$\mu = \frac{
    \alpha \eta (1 - {\alpha'}^2)
    }{
    \alpha^2 \sigma^2(1-{\alpha'}\eta)
    +\alpha^2(1+{\alpha'}\eta)
    +\delta^2(1-{\alpha'}\eta)
    } U_t$$ and $$\Delta^2 = 1+\sigma^2-\frac{
    \alpha^2 \eta^2 (1 - {\alpha'}^2)
    }{
    \alpha^2 \sigma^2(1-{\alpha'}\eta)^2
    +\alpha^2(1-{\alpha'}^2\eta^2)
    +\delta^2(1-{\alpha'}\eta)^2
    }.$$
\end{theorem}
\begin{proof}
    Since all the variables are $0$-mean multivariate Gaussians, $Y_{t+1}$ is a Gaussian conditioned on $U_t$.
Its mean is given by
\begin{align}\label{eq:lms-post-mean}
    \E[Y_{t+1} | U_t] 
    &= \frac{\E[U_tY_{t+1}]}{\E[U_t^2]} U_t \nonumber\\
    &= \frac{
    \frac{\alpha \eta }{1 - {\alpha'} \eta}
    }{
    \alpha^2 \frac{
    \sigma^2 (1 - {\alpha'} \eta) + 1 + {\alpha'} \eta
    }{
    (1 - {\alpha'} \eta) ( 1 - {\alpha'}^2)
    } + \frac{\delta^2}{1-{\alpha'}^2}
    } U_t
    \nonumber\\
    &= \frac{
    \alpha \eta (1 - {\alpha'}^2)
    }{
    \alpha^2 \sigma^2(1-{\alpha'}\eta)
    +\alpha^2(1+{\alpha'}\eta)
    +\delta^2(1-{\alpha'}\eta)
    } U_t
    ,
\end{align}
and its variance is given by
\begin{align}\label{eq:lms-post-var}
    \mathbb{V}[Y_{t+1} | U_t] 
    &= \E[Y_{t+1}^2] - \frac{\E[U_tY_{t+1}]^2}{\E[U_t^2]} \nonumber\\
    &= 1 + \sigma^2 - \frac{
    \left(\frac{\alpha \eta }{1 - {\alpha'} \eta}\right)^2
    }{
    \alpha^2 \frac{
    \sigma^2 (1 - {\alpha'} \eta) + 1 + {\alpha'} \eta
    }{
    (1 - {\alpha'} \eta) ( 1 - {\alpha'}^2)
    } + \frac{\delta^2}{1-{\alpha'}^2}
    }
    \nonumber\\
    &= 1+\sigma^2-\frac{
    \alpha^2 \eta^2 (1 - {\alpha'}^2)
    }{
    \alpha^2 \sigma^2(1-{\alpha'}\eta)^2
    +\alpha^2(1-{\alpha'}^2\eta^2)
    +\delta^2(1-{\alpha'}\eta)^2
    }
    ,
\end{align}
where we use the covariances calculated in Lemma~\ref{lem:lms-cov}.
\end{proof}

We now derive a closed-form solution to the value of $\delta^2$ (variance of quantization noise) which will ensure that agent state $U_t$ will contain at most $C$ nats of information about $H_t$ as $t\to\infty$.

\begin{theorem}\label{thm:lms-capacity-noise}
    \textbf{(information capacity)}
    For all $\alpha \in [0, 1]$, the minimal quantization noise variance $\delta^2$ needed to achieve information capacity $C$, i.e., 
    \[\limsup_{t\to\infty}\I(U_t; H_{t})\le C,\]
    is given by
    \[
    \delta^2= 
    \alpha^2 \frac{
    \sigma^2 (1 - {\alpha'} \eta) + 1 + {\alpha'} \eta
    }{
    1 - {\alpha'} \eta
    }
    \frac{\exp(-2C)}{1-\exp(-2C)}.
    \]
\end{theorem}
\begin{proof}
    We let $\Sigma$ denote the covariance matrix of $H_t=Y_{1:t}$
    and the vector $\phi$ denote the covariance between $U_t$ and $Y_{1:t}$:
    \[
    \phi = 
    \begin{bmatrix}
        \E[U_tY_1] & \E[U_tY_2] & ... & \E[U_tY_t]
    \end{bmatrix}^T.
    \]
    Since the variables are multivariate $0$-mean Gaussians, the mutual information is given by
    \begin{align*}
        \I(U_t; H_{t})
        &=
        \frac{1}{2}  \left(
        \log \E[U_t^2] 
        + \log \det \Sigma
        -\log \det
        \begin{bmatrix}
            \E[U_t^2] & \phi^T \\
            \phi & \Sigma
        \end{bmatrix}
        \right)
        \\
        &= \frac{1}{2}  \left(
        \log \E[U_t^2] 
        + \log \det \Sigma
        -\log \det
        \begin{bmatrix}
            \E[U_t^2] - \phi^T \Sigma^{-1} \phi & 0 \\
            \phi & \Sigma
        \end{bmatrix}
        \right)
        \\
        &=  \frac{1}{2}  \left(
        \log \E[U_t^2] 
        -\log \left(
            \E[U_t^2] - \phi^T \Sigma^{-1} \phi 
            \right)
        \right)
        \\
        &= - \frac{1}{2} \log \left(
        1 - \frac{\phi^T \Sigma^{-1} \phi }{\E[U_t^2]}
        \right)
        .
    \end{align*}
    From Lemma~\ref{lem:lms-cov} we know that neither $\phi$ or $\Sigma$ depends on $\delta^2$.
    Hence, the above expression only depends on $\delta^2$ through
    \[
    \E[U_t^2] =
    \alpha \frac{
    \sigma^2 (1 - {\alpha'} \eta) + 1 + {\alpha'} \eta
    }{
    (1 - {\alpha'} \eta) ( 1 + {\alpha'})
    } + \frac{\delta^2}{1-{\alpha'}^2}.
    \]

    For every positive $\delta$, $\I(U_t; H_t)=h(U_t) - h(U_t | H_t)$ is finite.
    As $\delta\to 0^+$, $h(U_t)$ converges to a finite value, while $h(U_t | H_t)$ diverges to negative infinity.
    This implies that when $\delta=0$,
    \[
    \phi^T \Sigma^{-1} \phi = \E_{\delta=0}[U_t^2] = \alpha \frac{
    \sigma^2 (1 - {\alpha'} \eta) + 1 + {\alpha'} \eta
    }{
    (1 - {\alpha'} \eta) ( 1 + {\alpha'})
    }.
    \]
    Note that $\I(U_t; H_{t})\le C$ is equivalent to
    \[
    \E[U_t^2] \ge 
    \frac{\phi^T\Sigma^{-1}\phi}{1-\exp(-2C)}
    =
    \frac{\E_{\delta=0}[U_t^2]}{1-\exp(-2C)}.
    \]
    Plugging in the formula for $\E[U_t^2]$, we get
    \begin{align*}
    \delta^2 &\ge (1-{\alpha'}^2) 
    \alpha \frac{
    \sigma^2 (1 - {\alpha'} \eta) + 1 + {\alpha'} \eta
    }{
    (1 - {\alpha'} \eta) ( 1 + {\alpha'})
    }
    \frac{\exp(-2C)}{1-\exp(-2C)}\\
    &=\alpha^2 \frac{
    \sigma^2 (1 - {\alpha'} \eta) + 1 + {\alpha'} \eta
    }{
    1 - {\alpha'} \eta
    }
    \frac{\exp(-2C)}{1-\exp(-2C)}
    .
    \end{align*}
\end{proof}

\subsection{Plasticity and Forgetting Error}
\label{apdx:lms-plasticity}
Assuming $\{U_t\}$ and $\{Y_t\}$ are weak-sense jointly stationary processes, the total forgetting error can be simplified via Theorem~\ref{thm:errors-simplification} to be 
\[
\I(H_{t+1:\infty}; U_{t-1}|U_{t}, O_t) = \I(Y_{t+1:\infty}; U_{t-1}|U_{t}, Y_t).
\]
Similarly, the total implasticity error simplifies to
\[
 \I(H_{t+1:\infty};O_{t}|U_{t}) =  \I(Y_{t+1:\infty};Y_{t}|U_{t}).
\]
All the variables in the above expressions are linear combinations of $0$ mean multivariate normal variables, and so they are also $0$ mean multivariate normal variables.
Hence we can calculate this mutual information using the well known formula for conditional mutual information of multivariate Gaussians:
\begin{lemma}
If $X_1, X_2, X_3$ are jointly Gaussian with covariance matrix $\Sigma_{123}$.
Then,
    \[
    \I(X_1; X_2 | X_3) = \frac{1}{2} \ln \frac{\det \Sigma_{13}\det \Sigma_{23}}{\det\Sigma_{123}\det\Sigma_3},
    \]
    where we use $\Sigma_{ij}$ and $\Sigma_i$ to denote submatrices corresponding to the respective variables $X_i$.
\end{lemma}
We can use the covariances in Lemma~\ref{lem:lms-cov} to calculate these determinants.

\subsection{The Optimal Learning Rate is Independent of Information Capacity}\label{apdx:opt_learning_rate}
In this section we show that for all AR(1) processes, the optimal learning rate is independent of information capacity, and is equal to the optimal learning rate for the noiseless agent.

We begin by proving a lemma with establishes that for the noiseless agent, the optimal learning rate attains the optimal prediction based on the entire history.

\begin{lemma}
    For all AR(1) processes without quantization noise in agent state updates, there exists a learning rate $\alpha^*$ such that under this learning rate, $Y_{t+1}\rightarrow U_t \rightarrow H_t$ is Markov.
    In particular, this means that $\alpha^*$ is an optimal learning rate.
\end{lemma}
\begin{proof}
Suppose that we have a AR(1) process parameterized by $\eta<1$ and $\sigma$.
Consider the equation for ${\alpha'}$
\begin{equation}\label{eq:lms-beta-optimal}
    {\alpha'} + \frac{1}{{\alpha'}} =
    \eta + \frac{1}{\eta} + \frac{1}{\sigma^2\eta} - \frac{\eta}{\sigma^2}.
\end{equation}
When $\eta<1$, the right hand side is strictly greater than $2$.
Hence, there exists an unique solution of ${\alpha'}$ in $(0, 1)$, which we call ${\alpha'}^*$.
We claim that $\alpha^*=1-{\alpha'}^*$ makes $Y_{t+1}\rightarrow U_t \rightarrow H_t$ Markov.

Since these are jointly Gaussian with mean $0$, it suffices to show that $U_t$ is proportional to the projection of $Y_{t+1}$ onto the linear subspace spanned by $H_t=Y_{1:t}$.
Equivalently, we just need to verify that $\E[U_t Y_{t-k}]$ is proportional to $\E[Y_{t+1} Y_{t-k}]$ for $k=0,...,t-1$.
From Lemma~\ref{lem:lms-cov}, we know that $\E[Y_{t+1} Y_{t-k}]=\eta^{k+1}$.
On the other hand,
\begin{align*}
\eta\E[U_t Y_{t-k}] -\E[U_t Y_{t-k-1}]
    &= \eta \left( 
    \alpha^* \sigma^2{{\alpha'}^*}^k + 
    \frac{\alpha^* (1-\eta^2)}{\eta - {{\alpha'}^*}}\left(
    \frac{\eta^{k+1}}{1-\eta^2} - \frac{{{\alpha'}^*}^{k+1}}{1-\eta{{\alpha'}^*}}
    \right) \right)
    \\
    &\qquad - \left(\alpha^* \sigma^2{{\alpha'}^*}^{k+1} + 
    \frac{\alpha^* (1-\eta^2)}{\eta - {{\alpha'}^*}}\left(
    \frac{\eta^{k+2}}{1-\eta^2} - \frac{{{\alpha'}^*}^{k+2}}{1-\eta{{\alpha'}^*}}
    \right)
    \right)
    \\
    &=
    \alpha^*{{\alpha'}^*}^k \left(
    \eta\sigma^2 - {\alpha'}^*\sigma^2
    - \frac{1-\eta^2}{\eta-{\alpha'}^*} \frac{\eta{\alpha'}^*-{{\alpha'}^*}^2}{1-\eta{\alpha'}^*}
    \right)
    \\
    &= \frac{\alpha^*{{\alpha'}^*}^k}{1-\eta{\alpha'}^*} \left(
    \left(\eta\sigma^2 - {\alpha'}^*\sigma^2\right) (1-\eta{\alpha'}^*)
    - (1-\eta^2){{\alpha'}^*}
    \right)
    \\
    &= \frac{\eta\sigma^2\alpha^*{{\alpha'}^*}^{k+1}}{1-\eta{\alpha'}^*} \left(
    {\alpha'}^* + \frac{1}{{\alpha'}^*}
    - \frac{1}{\eta} - \eta  - \frac{1}{\eta\sigma^2} + \frac{\eta}{\sigma^2}
    \right)
    \\
    &=0 & \text{(by Equation~\ref{eq:lms-beta-optimal})}
    .
\end{align*}
Thus, $\E[U_tY_{t-k}]=\eta^k \E[U_t Y_t]$ is proportional to $\E[Y_{t+1} Y_{t-k}]=\eta^{k+1}$.
Therefore, we have shown that $Y_{t+1}\rightarrow U_t \rightarrow H_t$ is Markov.
This means that $U_t$ is a sufficient statistic of $H_t$ with respect to predicting $Y_{t+1}$, which implies that $\alpha^*$ is an optimal learning rate.
\end{proof}

\lmsLearningRateIndep*

\begin{proof}
    For any AR(1) process parameterized by $\eta$ and $\sigma$, assume that in the noiseless case ($\delta=0$), there exists a learning rate $\alpha^*$ such that under this learning rate, $Y_{t+1}\rightarrow U_t(\alpha^*, 0) \rightarrow H_t$ is Markov, where we use $U_t(\alpha, \delta)$ to denote the agent state $U_t$ under the learning rate $\alpha$ and quantization noise variance $\delta^2$.

    Now consider a bound on capacity of the form $\I(H_t, U_t)\le C$. This is enforced by choosing a quantization noise variance as a function of $\alpha$, $\delta:=\delta(\alpha)$, that attains $\I(H_t, U_t)\le C$.
    Since the updates are linear, 
    \[
    U_t(\alpha, \delta) = U_t(\alpha, 0) + C_\alpha \delta Z,
    \]
    where $Z$ is a standard normal independent of $H_t$ and $C_\alpha$ is a constant that depends on $\alpha$.

    The learning rate that achieves minimal total error is
    \begin{align*}
        \argmin_\alpha \I(Y_{t+1} ; H_t | U_t(\alpha, \delta))
        &= \argmin_\alpha \left(
        \I(Y_{t+1} ; H_t , U_t(\alpha, \delta))
        - \I(Y_{t+1} ;  U_t(\alpha, \delta))
        \right)
        & \text{(chain rule.)}
        \\
        &= \argmin_\alpha \left(
        \I(Y_{t+1} ; H_t)
        - \I(Y_{t+1} ;  U_t(\alpha, \delta))
        \right)
        & \text{(since $U_t(\alpha, \delta)\perp Y_{t+1}|H_t$.)}
        \\
        &= \argmax_\alpha \I(Y_{t+1} ;  U_t(\alpha, \delta))
    \end{align*}
    Since $U_t(\alpha,\delta)\rightarrow H_t \rightarrow U_t(\alpha^*, 0)$ and 
    $H_t \rightarrow U_t(\alpha^*, 0) \rightarrow Y_{t+1}$ are Markov,
    $U_t(\alpha,\delta) \rightarrow U_t(\alpha^*, 0)\rightarrow Y_{t+1}$ is also Markov.
    So we can apply Lemma~\ref{lem:lms-mutual-inf-markov} on $\I(Y_{t+1}; U_t(\alpha, \delta))$ with $Y=U_t(\alpha^*, 0)$.
    By the lemma, $\I(Y_{t+1}; U_t(\alpha, \delta))$ is a strictly increasing function of $\I(Y_{t+1}; U_t(\alpha^*, 0))$ and $\I(U_t(\alpha, \delta); U_t(\alpha^*, 0))$.
    
    The first term $\I(Y_{t+1}; U_t(\alpha^*, 0))$ is a constant independent of $\alpha$ and $\delta$. The second term is bounded above by the data processing inequality via
    \begin{align*}
    \I(U_t(\alpha, \delta); U_t(\alpha^*, 0)) 
    \le \I(U_t(\alpha, \delta); H_t) 
    \le C.
\end{align*}
In particular, equality is attained when $\alpha=\alpha^*$, since 
\begin{align*}
    C &= \I(U_t(\alpha^*, \delta(\alpha^*); H_t)
    & \text{(definition of $\delta(\alpha^*)$)}
    \\
    &= h(U_t(\alpha^*, \delta(\alpha^*))
    - h(U_t(\alpha^*, \delta(\alpha^*) | H_t)
    \\
    &= h(U_t(\alpha^*, \delta(\alpha^*))
    - h(U_t(\alpha^*, \delta(\alpha^*) | H_t, U_t(\alpha^*, 0))
    & \text{(since $U_t(\alpha^*, 0)$ is a function of $H_t$)}
    \\
    &= h(U_t(\alpha^*, \delta(\alpha^*))
    - h(C_{\alpha^*} \delta(\alpha^*) Z | H_t, U_t(\alpha^*, 0))
    \\
    &= h(U_t(\alpha^*, \delta(\alpha^*))
    - h(C_{\alpha^*} \delta(\alpha^*) Z | U_t(\alpha^*, 0))
    & \text{(since $H_t \perp Z \ |\ U_t(\alpha^*, 0)$.)}
    \\
    &= h(U_t(\alpha^*, \delta(\alpha^*))
    - h(U_t(\alpha^*, \delta(\alpha^*) | U_t(\alpha^*, 0))
    \\
    &= \I(U_t(\alpha^*, \delta(\alpha^*) ; U_t(\alpha^*, 0))
    .
\end{align*}
Thus, we have shown that for all information capacity $C=\I(U_t; H_t)$, $\alpha^*$ is the learning rate that minimizes overall error.
\end{proof}

\begin{lemma}
    \label{lem:lms-mutual-inf-markov}
    Suppose that $(X, Y, Z)$ is a $3$-dimensional joint Gaussian satisfying $X\rightarrow Y\rightarrow Z$ is Markov, $\I(X; Y)>0$, and $\I(X;Z)>0$, then
    $\I(X;Z)$ can be expressed in terms of $\I(X; Y)$ and $\I(Z; Y)$ through the following formula
    \[
    \I(X;Z) = - \frac{1}{2} \log \left(
    1 - \left(1 - e^{-2\I(X;Y)}\right)
    \left(1 - e^{-2\I(Z;Y)}\right)
    \right).
    \]
    In particular, this is a strictly increasing function in $\I(X;Y)$ and $\I(Z;Y)$.
\end{lemma}
\begin{proof}
    Since invertible linear transformations do not change mutual information, we can assume without loss of generality that $X,Y,Z$ are all $0$ mean, and $\E[XY]=\E[YZ]=\E[Y^2]=1$.
    Since $X\rightarrow Y\rightarrow Z$ is Markov, we can find $X'$ and $Z'$ such that 
    $X=Y+X'$, $Z=Y+Z'$, and $X',Y,Z'$ are independent zero-mean Gaussians.
    Suppose that $\mathbb{V}[X']=x^2$, and $\mathbb{V}[Z']=z^2$.
    Then, the mutual information $\I(X;Z)$ can be written as
    \begin{align*}
    \I(X;Z) &=
    \frac{1}{2} \log \frac{
    \E[X^2]\E[Z^2]
    }{\E[X^2]\E[Z^2] - \E[XZ]^2}
    \\
    &= \frac{1}{2} \log \frac{
    (1+x^2)(1+z^2)
    }{(1+x^2)(1+z^2) - 1}
    \\
    &= -\frac{1}{2} \log \left(
    1 - \frac{1}{1+x^2}\frac{1}{1+z^2}
    \right).
  \end{align*}
  Similarly, the mutual information $\I(X;Y)$ simplifies to
  \begin{align*}
    \I(X;Y) &=
    \frac{1}{2} \log \frac{
    \E[X^2]\E[Y^2]
    }{\E[X^2]\E[Y^2] - \E[XY]^2}
    \\
    &= \frac{1}{2} \log \frac{
    (1+x^2)
    }{(1+x^2) - 1}
    \\
    &= \frac{1}{2} \log \frac{
    (1+x^2)
    }{x^2}.
  \end{align*}
  This implies that
  \[
  \frac{1}{1+x^2} = 1 - \exp(-2\I(X;Y)).
\]
By symmetry,
\[
  \frac{1}{1+z^2} = 1 - \exp(-2\I(Z;Y)).
\]
Plugging these back into the formula for $\I(X;Z)$, we get the desired result.
\end{proof}

\subsection{Modifying IDBD to Account for Quantization Noise}
\label{se:capacity-constrained-IDBD-derivation}

As illustrated in Figure~\ref{fig:IDBD-convergence}, the standard version of IDBD does not find the optimal step size $\alpha^*$ subject to a capacity constraint $\I(U_t; H_t) \leq C$.
This is because of the quantization error, with variance that depends on the learning rate $\alpha$ in the agent update rule.
Hence, the update rule for the log learning rate $\beta$ needs to be modified accordingly.

In our context, the standard IDBD update rule of \citet{Sutton1992IDBD} takes the form
\[
\beta_{t+1} = \beta_t - \frac{1}{2} \zeta \frac{\partial (U_t - Y_{t+1})^2}{\partial \beta}
.
\]
Since our updates to $U_t$ are random, we propose to take the expectation over $(U_t - Y_{t+1})^2$:
\[
\beta_{t+1} = \beta_t - \frac{1}{2} \zeta \frac{\partial \E\left[(U_t - Y_{t+1})^2\right]}{\partial \beta}
,
\]
where the expectation is taken over the random quantization noise $Q_t$.
Recall that $U_t$ is calculated from the previous state via
\[
U_{t} = U_{t-1} + \alpha_{t} (Y_{t} - U_{t-1}) + Q_{t}.
\]
Since $Q_t$ is uncorrelated with the other terms, $\E\left[(U_t - Y_{t+1})^2\right]$ simplifies to 
$$\E\left[(U_t - Q_t - Y_{t+1})^2\right] + \E[Q_t^2] = \E\left[(U_t - Q_t - Y_{t+1})^2\right] + \delta_*(\alpha_t)^2.$$
The derivative of the second term with respect to $\beta$ is just
\begin{align*}
    \frac{\partial \left(\delta_*(\alpha_t)^2\right)}{\partial\beta} = \frac{\partial \left(\delta_*(\alpha_t)^2\right)}{\partial\alpha_t} 
    \frac{\partial \alpha_t}{\partial \beta} = \alpha_t 
    \frac{d}{d \alpha} \delta_*^2(\alpha_t).
\end{align*}

In the following derivations, we will use $A\approx B$ to mean that $\E[B]=A$. This approximation allows us to obtain a stochastic gradient whose expectation is equal to the true gradient, $\frac{\partial \E\left[(U_t - Y_{t+1})^2\right]}{\partial \beta}$.
Letting $h_t:=\frac{\partial (U_t-Q_t)}{\partial \beta}$, 
the derivative of the first term with respect to $\beta$ can by simplified via the chain rule:
\begin{align*}
\E\left[ \frac{\partial(U_t -Q_t- Y_{t+1})^2}{\partial \beta} \right] 
    &= \E\left[ \frac{\partial(U_t -Q_t- Y_{t+1})^2}{\partial (U_t-Q_t)}
    \frac{\partial (U_t-Q_t)}{\partial \beta} \right] \\
    &= 
    \E\left[ (U_t -Q_t- Y_{t+1})
    \frac{\partial (U_t-Q_t)}{\partial \beta} \right]
    \\
    &= \E\left[ (U_t - Y_{t+1})
    \frac{\partial (U_t-Q_t)}{\partial \beta} \right]
    & \text{(since $Q_t$ is uncorrelated with $ \frac{\partial (U_t-Q_t)}{\partial \beta}$) }
    \\
    &\approx (U_t - Y_{t+1}) h_t.
\end{align*}
Hence, the update rule for $\beta$ is
$$\beta_{t+1} = \beta_t + \zeta (Y_{t+1} - U_t) h_t - \frac{1}{2} \zeta \alpha_t \frac{d}{d \alpha} \delta_*^2(\alpha_t).$$

As in \citet{Sutton1992IDBD}, $h_t$ can be recursively updated via
\begin{align*}
    h_{t+1} &= \frac{\partial (U_{t+1}-Q_{t+1})}{\partial \beta}
    \\
    &= \frac{\partial (U_t + \alpha_{t+1} (Y_{t+1} - U_t))}{\partial \beta}
    \\
    &= (1 - \alpha_{t+1}) \frac{\partial U_t}{\partial \beta}
    + \alpha_{t+1} \left(Y_{t+1} - U_t \right)
    \\
    &= (1 - \alpha_{t+1}) \E\left[\frac{\partial (U_t - Q_t)}{\partial \beta}\right]
    + \alpha_{t+1} \left(Y_{t+1} - U_t \right)
    \\
    &\approx (1 - \alpha_{t+1}) h_t + \alpha_{t+1} (Y_{t+1} - U_t),
\end{align*}
which is exactly the update rule for $h_t$ without clipping $(1 - \alpha_{t+1})$.

\section{Types of Nonstationarity in Continual Supervised Learning}\label{app:types_of_nonstationarity}
\henrik{
It is common to study different types of nonstationarity, since each comes with unique challenges. Below is a categorization commonly used in continual supervised learning~\citep{moreno2012unifying}:  
\begin{enumerate}
    \item Covariate shift: the distribution over input features changes over time~\citep{goodfellow2013empirical}. For example, an agent operating an autonomous vehicle may encounter changes in lighting conditions throughout the day.
    \item Label shift: the distribution over labels changes over time~\citep{baby2024online}. In medical diagnosis applications, which illnesses are common may change over time, with certain infections like the flu becoming more prevalent in colder weather.
    \item Concept shift: the function from inputs to labels changes over time. For example, in financial modeling, the relationship between market variables and stock prices may change.
\end{enumerate}

While real-world continual learning problems may feature a combination of these three kinds of nonstationarity, explicitly considering them in isolation can help better reveal the capabilities of continual learning agents.
}

\section{Case Studies}
\subsection{Continual Supervised Learning}\label{apx:sl-experiments}

\textbf{Environment}

On our modified Permuted MNIST environment, we split all data into two subsets, one used to generate the development sequence and the second used to generate evaluation sequences. Each subset contains $100$ permutations, and the two subsets do not share any permutations. Consequently, there is no overlap in permutations between the development sequence and the evaluation sequences. Each permutation has $400$ unique data pairs. 

The sequence of data pairs the agent sees is generated from the environment as follows. At each timestep, the agent sees a batch of $b_\text{env}$ data pairs from the environment. In our experiments, we set the environment batch size $b_\text{env}$ to be 16. These $16$ data pairs are randomly sampled from the $400$ unique data pairs for the current permutation. We specify a permutation duration for each experiment (either $2,000$ or $20,000$ timesteps). This permutation duration specifies how many timesteps data pairs from a permutation arrive before the next permutation begins. Note that if the permutation duration is $2,000$, then each data pair from every permutation is seen by the agent approximately $\frac{2,000 \text{ timesteps } * 16 \text{ number of data pairs per timestep  }}{400 \text{ number of unique data pairs per permutation }} = 80$ times.

\begin{table}[h]
\centering
\begin{tabular}{|l|l|}
\hline
Parameter & Value  \\ 
\hline
Number of permutations & 100 \\
Number of unique data pairs per permutation  & 400  \\ 
Number of repetitions of the first permutation & 100 \\
Environment batch size & 16 \\
Learning rate  & 0.01  \\ 
Replay batch size & 16  \\ 
\hline
\end{tabular}
\caption{Environment parameters.}
\label{tab:env_params}
\end{table}

\textbf{Agents}

For all agents, we performed a sweep over learning rates in the set $\{0.001, 0.01, 0.1\}$ on the development sequence. We found $0.01$ to result in the largest average reward on the development sequence for all agents. For all runs on the evaulation sequences, we therefore set the learning rate to be $0.01$.

\begin{table}[h]
\centering
\begin{tabular}{|l|l|}
\hline
Hyperparameter & Value  \\ 
\hline
Learning rate  & 0.01  \\ 
Replay batch size & 16  \\ 
\hline
\end{tabular}
\caption{Agent hyperparameters common to all experiments.}
\label{tab:agent_hyperparams}
\end{table}

\saurabh{
\subsubsection{A Brief Note on Information Transfer in Permuted MNIST}\label{app:forward_transfer}
In this section, we discuss how a computationally unrestricted agent on Permuted MNIST may be implemented. This discussion gives an example of how an agent's constraints may affect the degree to which it can leverage information learned on previous tasks to quickly learn new tasks.

Suppose the agent is computationally unconstrained, and further suppose that the agent has learned the first task. For instance, we can imagine that the agent stores all image-label pairs. We can now ask: how much new information is there for the agent to acquire in the next task? It is sufficient for the agent to learn which permutation of the image pixels and, in our modified Permuted MNIST problem, which permutation of the labels the task corresponds to. Therefore, we can split the information to be acquired into two parts:
\begin{enumerate}
    \item Which image permutation has been applied. Since there are $n!$ possible permutations, the number of nats is $\text{ln }n!$.
    \item Which label permutation has been applied. Since there are $10!$ possible label permutations, the number of nats is $\text{ln }10!$.
\end{enumerate}
A computationally unrestricted agent can in principle acquire these bits very quickly, since each new image-label pair will narrow down the number of possible permutations significantly. Once it has acquired these bits, it can make perfect predictions through the following procedure: match each new image with the corresponding stored image-label pair from the first task to identify the correct label. 

However, for agents that are constrained in memory or compute, this approach is infeasible. For instance, if the first task's dataset is very large, it becomes infeasible for the agent to store all the image-label pairs. Additionally, for sufficiently large datasets, the compute required to search over the stored dataset for the purpose of finding a matching image to make a prediction becomes too large. We note that the dataset may have to be quite large for the compute bottleneck to kick in. Furthermore, the nature of the compute budget matters: for instance, what the cost of floating point operations is relative to the cost of memory access affects to what extent the agent can make effective use of its memory.
}

\subsection{Continual Learning with Delayed Consequences}
\label{apdx:continual-delayed}

\textbf{Environment}

The reward for transitioning to the goal state is computed as follows. Each time the MDP is updated, the reward in the goal state is recomputed. For the MDP at timestep $t$, it is computed such that the average long-term reward of the optimal MDP policy is approximately $.5$. This is done in three steps: First, we compute the optimal action value function for the MDP at timestep $t$ assuming a discount factor of $.9$. This is done using value iteration. The greedy actions of that action value function induce a stationary distribution over the states. We compute this stationary distribution. Let the value of this stationary distribution for the goal state be $d_g$. We set the reward in the goal state to be $0.5 / d_g$. This ensures that the average reward is approximately $0.5$ for the MDP.

\textbf{Hyperparameter Sweep}

We set $\gamma=0.9$ and sweep over parameters as follows:
\begin{table}[h]
\centering
\begin{tabular}{|ll|l|}
\hline
Hyperparameter & Value  \\ 
\hline
Step size $\alpha$ & 0.025, 0.05, 0.1, 0.15, 0.2, 0.3, 0.4, 0.6, 0.8   \\ 
Optimistic boost $\zeta$ & 0.00001, 0.00005, 0.0001, 0.0002, 0.0004, 0.0006, 0.010 \\

Discount factor $\gamma$ & 0.90 \\
\hline
\end{tabular}
\label{tab:my_label}
\end{table}

We average results over 8 seeds, with each seed running for 250,000 timesteps.

\subsection{Meta-Gradient Derivation for Continual Auxiliary Learning}
\label{apdx:continual-auxiliary-learning-meta-grads}
To update $\hat{A}_t$, we take the derivative of the loss and perform gradient descent.
We discovered that convergence is the fastest when we pretend that $\hat{A}_t\hat{\theta}_t'$ only depends on $\hat{A}_t$ through $\hat{\theta}_t'$, or in other words, when we put a stop-gradient on $\hat{A}_t$ in $\hat{A}_t\hat{\theta}_t'$.
With this in mind, we derive the update rules for $\hat{A}_t$:
\begin{align*}
    \hat{A}_{t+1} & = \hat{A}_{t} - \beta \frac{d L_t}{d \hat{A}} \\
    &\overset{(a)}{=} \hat{A}_{t} - \beta \frac{\partial L_t}{\partial \hat{\theta}_t'}
    \frac{d \hat{\theta}_t'}{d \hat{A}}
    \\
    &= \hat{A}_{t} - \beta g_t
    \frac{d \hat{\theta}_t'}{d \hat{A}}
    \\
    &\overset{(b)}{=} \hat{A}_{t} - \beta g_t
    h_t
    ,
\end{align*}
where in (a) we discard the term $\frac{\partial L_t}{\partial \hat{A}}$,
and in (b) we use $h_t$ to denote $\frac{d \hat{\theta}_t'}{d \hat{A}}$.
So $h_0=0$, and for $t\ge 1$, $h_t$ is recursively updated via
\begin{align*}
    h_t &= \frac{d \hat{\theta}_t'}{d \hat{A}}
    \\
    &= \frac{d (\mu \hat{\theta}_{t-1}' - \alpha \mu g_{t-1})}{d \hat{A}}
    \\
    &= \mu h_{t-1} - \alpha \mu \frac{d g_{t-1}}{d \hat{A}}
    \\
    &= \mu h_{t-1} - \alpha \mu \frac{\partial g_{t-1}}{\partial \hat{A}} - \alpha \mu \frac{\partial g_{t-1}}{\partial \hat{\theta}_{t-1}'} \frac{d \hat{\theta}_{t-1}'}{d \hat{A}}
    \\
    &\overset{(a)}{=}
    \mu h_{t-1} - \alpha \mu \left(
    -Y_{t} + (1 - Y_{t}) \odot \exp(\hat{A}_{t-1} \hat{\theta}_{t-1}')
    \right)
    \oslash \left(1+\exp(\hat{A}_{t-1} \hat{\theta}_{t-1}')\right) \\
    &\qquad\qquad - \alpha \mu \hat{A}_{t-1}^T \left[
    \exp(\hat{A}_{t-1} \hat{\theta}_{t-1}')
    \odot \hat{A}_{t-1} \oslash \left(1+\exp(\hat{A}_{t-1} \hat{\theta}_{t-1}')\right)^{\circ 2} \right] h_{t-1} 
\end{align*}
where in (a) we discard the term 
$- \alpha \mu \left[\exp(\hat{A}_{t-1} \hat{\theta}_{t-1}') \odot \hat{A}_{t-1} \oslash  \left(1+\exp(\hat{A}_{t-1} \hat{\theta}_{t-1}')\right)^{\circ 2} \right] \hat{\theta}_{t-1}' $, since we put a stop gradient on $\hat{A}_{t-1}$ in $\hat{A}_{t-1}\hat{\theta}_{t-1}'$.
Thus we have derived the update rules presented in Section~\ref{sec:continual_auxiliary_learning}.

\end{document}